\documentclass[11pt,letterpaper]{article}
\usepackage[margin=1in]{geometry}
\usepackage[T1]{fontenc}
\usepackage[utf8]{inputenc}
\usepackage{lmodern}
\usepackage{microtype}
\usepackage{amsmath,amssymb,amsthm}
\usepackage{algorithm}
\usepackage{algpseudocode}
\usepackage[authoryear,round]{natbib}
\setcitestyle{citesep={;},aysep={,},yysep={;}}
\usepackage{xurl}
\usepackage{graphicx}
\graphicspath{{figures/}{imgs/}}
\usepackage{tikz}
\usetikzlibrary{positioning,fit,arrows.meta}
\usepackage[hidelinks]{hyperref}
\hypersetup{
  pdftitle={Neural LoFi with Backward Coupling: \\
  A Spectral Theory of Cross-Layer Feature Learning},
  pdfauthor={Yatin Dandi, Matteo Vilucchio, Hugo Tabanelli, Florent Krzakala},
  pdfkeywords={feature learning, neural networks, spectral methods, backward coupling}
}

\title{Neural LoFi with Backward Coupling:\\ A Spectral Theory of Cross-Layer Feature Learning}
\author{%
  Yatin Dandi\thanks{Joint first authors.} \quad Matteo Vilucchio\footnotemark[1] \quad Hugo Tabanelli \quad Florent Krzakala\\[0.4em]
  \normalsize Information, Learning, and Physics Laboratory (IDEPHICS) EPFL, Lausanne, Switzerland
}
\date{}

\newcommand{\Etrain}{\widehat{\mathbb E}_n}
\newcommand{\RMS}{\operatorname{RMS}}

\newcommand{\Fwd}{\mathsf{FwdLoFi}}
\newcommand{\R}{\mathbb{R}}

\newcommand{\tr}{\operatorname{tr}}
\newcommand{\sign}{\operatorname{sign}}

\newtheorem{theorem}{Theorem}
\newtheorem{formaltheorem}{Theorem}[section]
\newtheorem{proposition}{Proposition}
\newtheorem{lemma}{Lemma}
\theoremstyle{definition}
\newtheorem{assumption}{Assumption}

\begin{document}
\raggedbottom
\maketitle

\begin{abstract}
During joint training, backpropagation allows later layers to reshape earlier representations, a feedback mechanism absent from purely layer-wise learning. We extend Neural Low-Degree Filtering (Neural LoFi), a recently proposed layer-wise spectral surrogate for feature learning, to capture this interaction. Starting from the exact backpropagation gradients, we derive a local coupling between a neuron’s incoming weights and its corresponding next-layer column. In a controlled small-initialization regime with smooth activations and fixed surroundings, these dynamics amplify singular modes of a cross-layer correlation operator. Each growing mode pairs an earlier feature with an upper-layer response through shared neuron coefficients. This motivates a backward-coupling algorithm in which upper-layer sensitivities generate sample-dependent effective labels for earlier spectral steps, while corrected features refine upward transmission. A function-space variational principle characterizes the selected feature–response pairs by maximizing the label correlation of their product under kernel-norm constraints. For three-layer networks, we derive self-consistent continuous-time dynamics for the filtered hidden-layer features, the activation kernels, and the network output under the proposed backward coupling. In a hierarchical multi-index model, we establish a feature-recovery separation: backward coupling recovers lower-layer directions invisible to the original forward label-weighted statistic. The linear specialization recovers residual-dependent singular-mode dynamics, while paired modes provide a dynamical mechanism for linked low-rank covariance components consistent with the rainbow-network picture. Experiments with fully connected and convolutional networks on binary CIFAR-10 show improved accuracy over forward Neural LoFi at matched widths.
\end{abstract}

\section{Introduction}
\label{sec:introduction}

Deep neural networks learn hierarchies of task-adapted representations
through backpropagation \citep{rumelhart1986learning,lecun2015deep}.
Yet a hierarchy in the forward computation does not imply a layer-wise
order of learning. During joint training, features emerging in later
layers reshape the learning signals received below, while changes in
earlier representations alter what the upper layers can learn.
Understanding this feedback requires going beyond how features are
constructed successively. It raises a central question:
\begin{center}
\emph{How does learning in later layers change which features earlier layers can discover?}
\end{center}

This paper builds on Neural Low-Degree Filtering (Neural LoFi),
introduced by \citet{dandi_neural_lofi} as an explicit spectral construction
for purely layer-wise learning. Given the current representation, each layer
selects directions through low-degree correlations with the labels and
passes them through a nonlinear random lift. Depth changes the geometry
in which subsequent layers search for signal, making structure that is
complex in the input accessible through low-degree statistics of
intermediate representations. However, this construction proceeds only
forward: features learned above cannot refine the directions selected
below. A useful upper representation may therefore become available only
after the lower feature-selection step that could benefit from it has
already been completed.

Existing theories describe complementary aspects of feature
learning. Neural-tangent limits characterize training with essentially
fixed representations \citep{jacot2018ntk,lee2019wide}; mean-field and
dynamical-field frameworks describe evolving features in shallow and
deep nonlinear networks
\citep{mei2018meanfield,chizat2018global,nguyen2020rigorous,bordelon2022selfconsistent,hajjar2024training}.
Feature-space mean-field models and tensor programs further
describe feature learning at infinite width
\citep{fang2021features,yang2021feature}.
Deep linear networks admit explicit analyses of coupled singular-mode
dynamics \citep{saxe2014exact}, while spectral analyses of quadratic and
diagonal networks connect learned weight structure to generalization
\citep{defilippis2026scaling}. Backward feature correction studies how joint training with
higher layers can refine lower representations
\citep{allenzhu2023backward}.

In this paper, we seek a tractable and interpretable spectral mechanism
that identifies how an upper-layer response changes the "earlier" features. We extend Neural LoFi with a tractable and interpretable
spectral surrogate that captures backward feature correction and,
in addition, identifies a distinct effect of joint training:
learning how lower-level cues selectively activate higher-level
features. This couples the learning of lower representations to
the learning of their transmission upward.
Starting from the exact backpropagation gradients, we
couple a neuron's incoming row to the next-layer column that reads its
activation. For a jointly trained block at small initialization, with smooth
activations and fixed surroundings, we prove that the leading dynamics amplify
singular modes of a cross-layer correlation operator. Each mode pairs
an earlier feature with an upper response, with the same growing
coefficients across intermediate neurons in both weight factors.
The paired updates coordinate what a neuron detects with how
its activation is transmitted upward: the incoming row acquires the
lower feature, while the outgoing column couples it to the selected
upper response.
A similar mechanism has been seen to spontaneously emerge in the feed-forward blocks of modern LLMs \citep{geva2021transformer,geva2022transformer} and we note additionally that
with linear activations, the residual-driven
row--column equations are exact and recover the singular-mode framework
of \citet{saxe2014exact}.

This mechanism motivates BackLoFi, a spectral correction
of a forward LoFi network that visits layer boundaries from the top
down. Upper-layer sensitivities define sample-dependent effective
labels for earlier spectral steps; the corrected nonlinear activations
then determine the corresponding upward transmission. The algorithm
uses spectral selection and normalized corrections as a surrogate for
the local dynamics. Its selection rule has a function-space
interpretation: an earlier feature $\chi$ and an upper response $\psi$
maximize $\Etrain[y\psi\chi]$, subject to the kernel-norm constraints
of their respective spaces. The search for a lower feature thus depends
on the response learned above as well as on the input representation
and the labels. That backward feedback can
make a previously invisible lower feature detectable, providing a drastic advantage to cross-layer training over layer-wise one.

We illustrate
this in a hierarchical multi-index model \citep{
nichani2023provable,wang2023learning,fu2025learning,dandi2026computational,tabanelli2026deep} with a visible anchor
and additional hidden directions. In its Gaussian  analysis, the original label-weighted second moment has no signal in
the hidden block, while modulation by a sensitivity of the upper
conditional predictor creates one. This illustrates a staircase learning mechanism through the backward pass.

A second theoretical contribution is a closed
description of coupled nonlinear feature learning in a three-layer
network. For a continuous variant with fixed Gaussian lifts, continual
forward filtering, and readout gradient flow, we prove convergence to
self-consistent feature and kernel dynamics. The evolving kernels
determine subsequent feature selection, while an accumulated coupling
retains earlier column corrections as lower features change. This provide an unifying description of how multi-layer kernel adaptivity \cite{bach2021quest} arise through cross-layer feature learning. It generalizes and extend the seminal work of \citet{saxe2014exact} that was limited to linear networks. This also provides a dynamical perspective on the key--value
organization identified in transformer feed-forward blocks by
\citet{geva2021transformer,geva2022transformer}: our local analysis
shows how feature detectors and their associated output directions
can be jointly selected and amplified during training.

Further, this mechanism also allows to provide an explanation for the rainbow-network picture of
\citep{guth2024rainbow}:
Shared neuron coefficients  yield linked low-rank
covariance components in adjacent weight matrices, providing a local
dynamical mechanism predicting the rainbow structure.

Finally, mechanistic experiments with fully connected and
convolutional networks on binary CIFAR-10 show considerable improved accuracy over
forward Neural LoFi at matched widths. They demonstrates that using
upper-layer information to revise earlier representations is useful
beyond the controlled theoretical setting, and that the predicted mechanism translate into real data scenarios.

A public implementation, together with codes to reproduce the numerical illustrations, is available at
\url{https://github.com/IdePHICS/BackLoFi-Theory}.

\section{Neural Low-Degree Filtering with backward coupling}
\label{sec:setting}

\paragraph{Backpropagation, and the guiding approximation ---}
\label{sec:gradient}
We first discuss the exact updates of a neuron's incoming row and
outgoing column, then use their local coupling to motivate our spectral
algorithm. Consider data $\mathcal D_n=\{(\mathbf x_\mu,y_\mu)\}_{\mu=1}^n$,
with centered scalar labels, and write $\Etrain$ for the empirical
average. Let $\mathbf z_0:\mathcal X\to\mathbb R^{p_0}$ be a given
input feature map; raw inputs are the special case $\mathbf z_0(\mathbf x)=\mathbf x$.
A depth-$L$ network constructs
\begin{equation}
 \mathbf h_\ell=\mathbf W_\ell\mathbf z_{\ell-1},\qquad
 \mathbf z_\ell=\frac1{\sqrt{p_\ell}}\sigma(\mathbf h_\ell),\qquad
 \widehat f=\langle\mathbf a_L,\mathbf z_L\rangle+b.
 \label{eq:main-network}
\end{equation}
Here $\mathbf W_\ell\in\mathbb R^{p_\ell\times p_{\ell-1}}$ and the
activation acts coordinatewise. The sensitivity of the output to the
preactivations, $\mathbf g^\ell=\partial\widehat f/\partial\mathbf h_\ell$,
obeys the backward recursion
\begin{equation}
 \mathbf g^L=\frac{\sigma'(\mathbf h_L)}{\sqrt{p_L}}\odot\mathbf a_L,
 \qquad
 \mathbf g^\ell=\frac{\sigma'(\mathbf h_\ell)}{\sqrt{p_\ell}}
      \odot(\mathbf W_{\ell+1}^{\top}\mathbf g^{\ell+1}).
 \label{eq:main-backward-signals}
\end{equation}
For the square loss, write $\Delta=y-\widehat f$ for the residual;
gradient flow then satisfies
\begin{equation}
 \dot{\mathbf W}_\ell
       =\Etrain[\Delta\,\mathbf g^\ell\mathbf z_{\ell-1}^{\top}].
 \label{eq:main-gradient}
\end{equation}

Neuron $j$ of layer $\ell$ owns two weight vectors: its incoming row,
whose transpose we write $\mathbf w_{\ell,j}$ (row $j$ of $\mathbf W_\ell$),
and its outgoing column $\boldsymbol\omega_{\ell+1,j}$ (column $j$ of
$\mathbf W_{\ell+1}$), through which the next layer reads its activation
(Figure~\ref{fig:backward-coupling-mechanism}).
By \eqref{eq:main-backward-signals} and \eqref{eq:main-gradient}, their
updates by gradient descent are
\begin{align}
 \dot{\mathbf w}_{\ell,j}
 &=\Etrain\!\left[\Delta\,
 \frac{\sigma'(h_{\ell,j})}{\sqrt{p_\ell}}
 \langle\boldsymbol\omega_{\ell+1,j},\mathbf g^{\ell+1}\rangle
 \mathbf z_{\ell-1}\right], \quad \dot{\boldsymbol\omega}_{\ell+1,j}
 =\Etrain[\Delta\,z_{\ell,j}\mathbf g^{\ell+1}].
 \label{eq:main-pair-gradient}
\end{align}
Each update depends on the other vector: the outgoing column
sets the sample-dependent gain driving the row, while the row's activation
determines the column update.

\paragraph{Layer-wise vs coupled training ---}
In \citep{dandi_neural_lofi}, the authors considered
the layer-wise ``Neural LoFi regime'' of the first equation in
which the layers above are still untrained and vanishingly small. The factor
$\langle\boldsymbol\omega_{\ell+1,j},\mathbf g^{\ell+1}(\mathbf x)\rangle$, the
effective readout of neuron $j$, then reduces to a sample-independent random
constant $\bar a_j$, and the residual to the label. For an activation smooth
near zero, $\sigma'(h_{\ell,j})\simeq\sigma'(0)+\sigma''(0)h_{\ell,j}$ gives
\begin{equation}
 \dot{\mathbf w}_{\ell,j}\simeq\frac{\bar a_j}{\sqrt{p_\ell}}\Bigl(
 \sigma'(0)\,\Etrain[y\,\mathbf z_{\ell-1}]
 +\sigma''(0)\,\Etrain[y\,\mathbf z_{\ell-1}\mathbf z_{\ell-1}^{\top}]\,
 \mathbf w_{\ell,j}\Bigr).
 \label{eq:main-lofi-limit}
\end{equation}
The rows evolve independently under the same label-weighted
moments. Neural LoFi retains the first-moment direction and the leading
second-moment eigenvectors by absolute eigenvalue, then applies a random lift. This leads to Alg.\ref{alg:neural-lofi} and weights of the form \eqref{eq:main-forward-lofi}.

We consider now a {\it joint training} procedure. With a learned upper representation, the effective readout becomes
sample dependent. To isolate the effect of this coupling, we treat $\mathbf g^{\ell+1}$ and $\mathbf z_{\ell-1}$ as fixed.
For a smooth activation with $\sigma(0)=0$, we retain the constant
derivative $\sigma'(0)$ in the row update and the linear activation in
the column update. Assuming $\sigma'(0) > 0$ gives
\begin{equation}
 \dot{\mathbf w}_{\ell,j}\simeq
 \frac{\sigma'(0)}{\sqrt{p_\ell}}\widetilde{\mathbf A}_\ell^\top\boldsymbol\omega_{\ell+1,j},
 \qquad
 \dot{\boldsymbol\omega}_{\ell+1,j}\simeq
 \frac{\sigma'(0)}{\sqrt{p_\ell}}\widetilde{\mathbf A}_\ell\mathbf w_{\ell,j},
 \qquad
 \widetilde{\mathbf A}_\ell
 :=\Etrain[\Delta\,\mathbf g^{\ell+1}\mathbf z_{\ell-1}^\top].
 \label{eq:main-coupled-residual}
\end{equation}
Therefore, the same matrix acts in both directions and its
singular pairs determine the growing modes. Close to initialization, under the initialization assumptions of
App.~\ref{app:scaled-feedback-proof}, the prediction is small, so
$\Delta\simeq y$ and the frozen coupling matrix becomes
\begin{equation}
 \mathbf A_\ell
 :=\Etrain[y\,\mathbf g^{\ell+1}\mathbf z_{\ell-1}^{\top}].
 \label{eq:main-coupled-operator}
\end{equation}
Write $s_i$ and
$(\mathbf u_i,\mathbf v_i)$ for the singular values and pairs of
$\mathbf A_\ell$. We obtain the following approximation for the joint-evolution of $\dot{\mathbf w}_{\ell,j}$ and $\langle\boldsymbol\omega_{\ell+1,j}$:

\begin{theorem}[Informal: amplification across adjacent layers]
\label{thm:controlled-feedback}
Under vanishing initialization and fixed surrounding layers,
for small enough time, a singular mode $(\mathbf u_i,\mathbf v_i,s_i)$ of the
frozen $\mathbf A_\ell$, with non-zero initial component, has growing components:
\begin{equation}
 \begin{aligned}
 \langle\mathbf w_{\ell,j}(t),\mathbf v_i\rangle
       &\simeq c_{ji}^{+}e^{\sigma'(0)s_i t/\sqrt{p_\ell}}
              +c_{ji}^{-}e^{-\sigma'(0)s_i t/\sqrt{p_\ell}},\\
 \langle\boldsymbol\omega_{\ell+1,j}(t),\mathbf u_i\rangle
       &\simeq c_{ji}^{+}e^{\sigma'(0)s_i t/\sqrt{p_\ell}}
              -c_{ji}^{-}e^{-\sigma'(0)s_i t/\sqrt{p_\ell}},
 \end{aligned}
 \label{eq:informal-paired-amplification}
\end{equation}
with amplitudes
$c_{ji}^{\pm}=\tfrac12(\langle\mathbf w_{\ell,j}(0),\mathbf v_i\rangle
\pm\langle\boldsymbol\omega_{\ell+1,j}(0),\mathbf u_i\rangle)$
set by initialization. See
App.~\ref{app:scaled-feedback-proof}
for details and error bounds.
\end{theorem}
Collecting the growing amplitudes of the linear approximation
\eqref{eq:main-coupled-residual} over the neurons of layer $\ell$,
$\mathbf q_i=(c_{ji}^{+})_j$, the growing parts have the rank-one directions
$\mathbf q_i\mathbf v_i^{\top}$ in $\mathbf W_\ell$ and
$\mathbf u_i\mathbf q_i^{\top}$ in $\mathbf W_{\ell+1}$, which we call a \emph{column spike}.
The same coefficients
$\mathbf q_i$ specify which intermediate neurons learn the feature and
which columns transmit it.

\paragraph{Neural LoFi with backward coupling ---}
\label{sec:combined}
Our algorithm, BackLoFi (Algorithm~\ref{alg:combined}), is motivated by this result and uses these singular
modes to construct a spectral surrogate for cross-layer feature learning.
On a fitted network, using labels in place of residuals is a design
choice; the theorem justifies this replacement near initialization.

The starting point is to use a {\it first pass} with forward Neural LoFi \citep{dandi_neural_lofi}
(Algorithm~\ref{alg:neural-lofi}), which builds the network one layer at a time
from the input up. Given
$\mathbf z_{\ell-1}$, a filter
$\mathbf F_\ell\in\mathbb R^{p_{\ell-1}\times k_\ell}$ collects the
$k_\ell$ directions amplified in \eqref{eq:main-lofi-limit}: the leading label weighted moments.
A Gaussian lift $\mathbf R_\ell\in\mathbb R^{p_\ell\times k_\ell}$ then distributes them over the $p_\ell$ neurons,
\begin{equation}
 \mathbf W_\ell=\mathbf R_\ell\mathbf F_\ell^{\top}.
 \label{eq:main-forward-lofi}
\end{equation}

\begin{algorithm}[t]
\caption{Forward Neural Low-Degree Filtering \citep{dandi_neural_lofi}}
\label{alg:neural-lofi}
\begin{algorithmic}[1]
\setlength{\itemsep}{1pt}
\Require Data $\{(\mathbf x_\mu,y_\mu)\}_{\mu=1}^n$, depth $L$, ranks $k_\ell$,
 widths $p_\ell$; $\mathbf z_0=\mathbf x$, $p_0=d$.
\For{$\ell=1,\ldots,L$}
 \State $\widehat{\mathbf m}_\ell\gets\Etrain[y\,\mathbf z_{\ell-1}]$,\;
  $\widehat{\mathbf C}_\ell\gets\Etrain[y\,\mathbf z_{\ell-1}\mathbf z_{\ell-1}^{\top}]$
  \Comment{label-weighted moments}
 \State $\mathbf F_\ell\gets[\,\widehat{\mathbf m}_\ell/\|\widehat{\mathbf m}_\ell\|\mid$
  leading $k_\ell-1$ eigenvectors of $\widehat{\mathbf C}_\ell$ by $|\lambda|\,]$
  \Comment{filter}
 \State $\mathbf z_\ell\gets\sigma(\mathbf R_\ell\mathbf F_\ell^{\top}\mathbf z_{\ell-1})/\sqrt{p_\ell}$,\;
  $\mathbf R_\ell\in\R^{p_\ell\times k_\ell}$ Gaussian
  \Comment{random lift}
\EndFor
\State Fit the readout $(\mathbf a_L,b)$ on $\mathbf z_L$;\;
 \Return $\widehat f(\mathbf x)=\langle\mathbf a_L,\mathbf z_L(\mathbf x)\rangle+b$
\end{algorithmic}
\end{algorithm}

\begin{algorithm}[t]
\caption{Neural Low-Degree Filtering with Backward Coupling (BackLoFi)}
\label{alg:combined}
\begin{algorithmic}[1]
\setlength{\itemsep}{1pt}
\Require Forward Neural LoFi network (Algorithm~\ref{alg:neural-lofi})
 $\mathbf W_\ell=\mathbf R_\ell\mathbf F_\ell^{\top}$, $\ell=1,\dots,L$,
 with readout $(\mathbf a_L,b)$; spike ranks $r_\ell$, scale $\beta>0$,
 passes $P$.
\For{$t=1,\ldots,P$}\label{line:pass}
 \State\label{line:top} Compute $\mathbf h_\ell,\mathbf z_\ell$ for all $\ell$ by
  \eqref{eq:main-network};\;
  $\mathbf g^L\gets\frac{\sigma'(\mathbf h_L)}{\sqrt{p_L}}\odot\mathbf a_L$
  \Comment{exact top sensitivity}
 \For{$\ell=L-1,\ldots,1$} \Comment{boundaries, top-down}
  \State\label{line:svd} $(\mathbf U_\ell,\mathbf V_\ell)\gets$ leading $r_\ell$
   singular pairs of $\mathbf A_\ell\gets\Etrain[y\,\mathbf g^{\ell+1}\mathbf z_{\ell-1}^{\top}]$
  \State\label{line:scores} $\mathbf s_\ell\gets\mathbf U_\ell^{\top}\mathbf g^{\ell+1}$
   \Comment{scores}
  \State\label{line:q0} $\mathbf Q^{0}_\ell\gets\mathbf W_\ell\mathbf V_\ell$
   \Comment{spike columns at order zero}
  \State\label{line:rows} $\mathbf F_\ell\gets[\mathbf F_\ell\mid
   \mathbf V_\ell\boldsymbol\Gamma^{\rm row}_\ell]$,\;
   $\mathbf R_\ell\gets[\mathbf R_\ell\mid\mathbf R^{\leftarrow}_\ell]$
   \Comment{rows of layer $\ell$}
  \State\label{line:recompute} Recompute $\mathbf h_\ell,\mathbf z_\ell$
   \Comment{features of the corrected rows}
  \State\label{line:q} $\mathbf Q_\ell\gets\Etrain[\mathbf z_\ell(y\mathbf s_\ell)^{\top}]$
   \Comment{correlations with the modified labels}
  \State\label{line:cols} $\mathbf F_{\ell+1}\gets[\mathbf F_{\ell+1}\mid
   \mathbf Q_\ell\boldsymbol\Gamma^{\rm col}_\ell]$,\;
   $\mathbf R_{\ell+1}\gets[\mathbf R_{\ell+1}\mid\mathbf U_\ell]$
   \Comment{columns of layer $\ell+1$}
  \State\label{line:signal} $\mathbf g^{\ell}\gets\frac{\sigma'(\mathbf h_\ell)}{\sqrt{p_\ell}}\odot\bigl(
   \Etrain[\mathbf W_{\ell+1}^{\top}\mathbf g^{\ell+1}]
   +\mathbf Q^{0}_\ell\mathbf s_\ell\bigr)$
   \Comment{first-order signal \eqref{eq:main-first-order}}
 \EndFor
 \State\label{line:refit} Recompute $\mathbf z_L$;\; refit $(\mathbf a_L,b)$
\EndFor
\State \Return $\widehat f(\mathbf x)=\langle\mathbf a_L,\mathbf z_L(\mathbf x)\rangle+b$
\end{algorithmic}
\end{algorithm}

After the forward pass, a backward pass visits boundaries from the top down, then refits
the readout. Boundary $\ell$ couples the rows of $\mathbf W_\ell$ to
the columns of $\mathbf W_{\ell+1}$. It receives the sensitivity updated
above and reads the representation below, which has not yet changed
during the pass.

First, retain the $r_\ell$ leading singular pairs of $\mathbf A_\ell$:
$\mathbf V_\ell=[\mathbf v_i]$ selects lower features and
$\mathbf U_\ell=[\mathbf u_i]$ selects upper responses
(line~\ref{line:svd}). These modes and responses are frozen during the
boundary update and recomputed at the next pass.

Second, install the lower directions with fresh Gaussian lift columns
$\mathbf R_\ell^{\leftarrow}$. Theorem~\ref{thm:controlled-feedback}
motivates coupling their upward transmission to the same intermediate
neurons. The algorithm estimates this transmission from the corrected
activations. After the row update, recompute $\mathbf z_\ell$ and set
\begin{equation}
 \mathbf s_\ell=\mathbf U_\ell^\top\mathbf g^{\ell+1},
 \qquad \widetilde{\mathbf y}_\ell=y\,\mathbf s_\ell,
 \qquad
 \mathbf Q_\ell=\Etrain[\mathbf z_\ell\widetilde{\mathbf y}_\ell^\top].
 \label{eq:main-exact-scores}
\end{equation}
Each column of $\mathbf Q_\ell$ is a first moment under a modified label.
Equivalently, $\mathbf U_\ell\mathbf Q_\ell^\top$ is the label-weighted
column gradient projected onto the retained upper directions. Computing
$\mathbf Q_\ell$ after the row step includes features absent from the
forward filter and uses the corrected nonlinear activations without
linearizing them.
The two updates are
\begin{equation}
 \mathbf W_\ell\leftarrow\mathbf W_\ell
   +\mathbf R^{\leftarrow}_\ell\boldsymbol\Gamma^{\rm row}_\ell\mathbf V_\ell^{\top},
 \qquad
 \mathbf W_{\ell+1}\leftarrow\mathbf W_{\ell+1}
   +\mathbf U_\ell\boldsymbol\Gamma^{\rm col}_\ell\mathbf Q_\ell^{\top},
 \label{eq:main-coupled-steps}
\end{equation}
implemented sequentially at lines~\ref{line:rows}--\ref{line:cols}.
The new directions must be added \emph{perturbatively} so that the modes of $\mathbf A_\ell$, selected on the current network, remain informative after the update.
Since $\mathbf W_\ell=\mathbf R_\ell\mathbf F_\ell^\top$, this can be done either by adding the small low-rank terms of
\eqref{eq:main-coupled-steps} or, equivalently, by appending columns to the filter and the lift.
The positive diagonal matrices $\boldsymbol\Gamma^{\rm row}_\ell$ and
$\boldsymbol\Gamma^{\rm col}_\ell$ normalize correction amplitudes using
a common scale $\beta$ (Appendix~\ref{app:algorithm}).
Fresh row lifts and separate normalization modify the theorem's paired
gradient flow while preserving coupling through the activations.

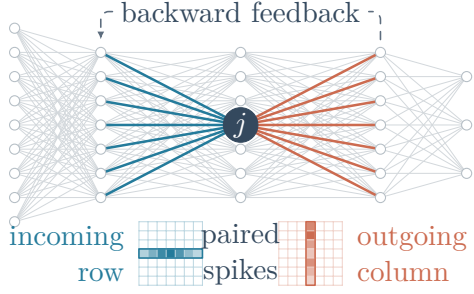
\begin{figure}[t]
 \centering
 \begin{minipage}[c]{.47\linewidth}
  \centering
  \begin{tikzpicture}[x=1cm,y=1cm,
  font=\normalfont\normalsize,
  >={Latex[length=3pt,width=2.7pt]},
  line cap=round,line join=round]
\definecolor{bcink}{HTML}{34495E}
\definecolor{bcgray}{HTML}{B5BEC5}
\definecolor{bcrow}{HTML}{247B9B}
\definecolor{bccol}{HTML}{CC6B50}
\tikzset{
  bcpoint/.style={circle,draw=bcgray,fill=white,line width=.4pt,
    minimum size=3.8pt,inner sep=0pt},
  bccontext/.style={draw=bcgray!60,line width=.24pt}
}

\foreach \k in {1,...,7} {
  \pgfmathsetmacro\yy{3.47-(\k-1)*.32}
  \coordinate (bcpre\k) at (1.3,\yy);
  \coordinate (bcmid\k) at (3.15,\yy);
  \coordinate (bcpost\k) at (5.0,\yy);
}
\foreach \k in {1,...,9} {
  \pgfmathsetmacro\yy{3.79-(\k-1)*.32}
  \coordinate (bcearly\k) at (.16,\yy);
}
\foreach \k/\yy in {1/3.15,2/2.51,3/1.87}
  \coordinate (bclate\k) at (6.14,\yy);

\foreach \i in {1,...,9}
  \foreach \j in {1,...,7}
    \draw[bccontext] (bcearly\i) -- (bcpre\j);
\foreach \i in {1,...,7} {
  \foreach \j in {1,...,7} {
    \draw[bccontext] (bcpre\i) -- (bcmid\j);
    \draw[bccontext] (bcmid\i) -- (bcpost\j);
  }
  \foreach \j in {1,...,3}
    \draw[bccontext] (bcpost\i) -- (bclate\j);
}
\foreach \k in {1,...,7} {
  \draw[bcrow,line width=.95pt] (bcpre\k) -- (bcmid4);
  \draw[bccol,line width=.95pt] (bcmid4) -- (bcpost\k);
}
\foreach \k in {1,...,9} \node[bcpoint] at (bcearly\k) {};
\foreach \k in {1,...,3} \node[bcpoint] at (bclate\k) {};
\foreach \k in {1,...,7} {
  \node[bcpoint] at (bcpre\k) {};
  \node[bcpoint] at (bcpost\k) {};
}
\foreach \k in {1,2,3,5,6,7} \node[bcpoint] at (bcmid\k) {};
\node[circle,draw=bcink,fill=bcink,line width=.4pt,minimum size=13pt,
  inner sep=0pt,text=white] at (bcmid4) {$j$};

\node[text=bcink,inner xsep=3pt,inner ysep=0pt] (bcfeedback)
  at (3.15,4.01) {backward feedback};
\draw[bcink!80,dashed,line width=.6pt,rounded corners=3pt]
  (5.0,3.64) -- (5.0,4.01) -- (bcfeedback.east);
\draw[bcink!80,dashed,line width=.6pt,->,rounded corners=3pt]
  (bcfeedback.west) -- (1.3,4.01) -- (1.3,3.64);

\begin{scope}[yshift=5.5mm]
\node[text=bcrow,anchor=east,align=right,inner sep=1pt] at (1.65,.26) {incoming\\row};
\node[text=bccol,anchor=west,align=left,inner sep=1pt] at (4.65,.26) {outgoing\\column};

\foreach \v [count=\k] in {.35,.6,.9,1,.8,.5,.3} {
  \pgfmathtruncatemacro\shade{100*\v}
  \fill[bcrow!\shade] ({1.805+(\k-1)*.12},.2) rectangle ++(.12,.12);
}
\foreach \u [count=\k] in {.8,1,.65,.4,.75,.5,.3} {
  \pgfmathtruncatemacro\shade{100*\u}
  \fill[bccol!\shade] (4.015,{-.16+(7-\k)*.12}) rectangle ++(.12,.12);
}
\draw[bcrow!30,line width=.25pt,step=.12,shift={(1.805,-.16)}] (0,0) grid (.84,.84);
\draw[bccol!30,line width=.25pt,step=.12,shift={(3.655,-.16)}] (0,0) grid (.84,.84);
\draw[bcrow,line width=.6pt] (1.805,.2) rectangle ++(.84,.12);
\draw[bccol,line width=.6pt] (4.015,-.16) rectangle ++(.12,.84);
\node[text=bcink,align=center,inner sep=0pt] at (3.15,.26) {paired\\spikes};
\end{scope}
\end{tikzpicture}
 \end{minipage}\hfill
 \begin{minipage}[c]{.50\linewidth}
  \caption{\textbf{Backward coupling and paired spikes.}
  The incoming row (blue) and outgoing column (red) meet at neuron $j$. Upper-layer sensitivity changes
  the feature learned below, whose activation changes upward transmission.
  Growing modes
  form paired rank-one spikes with shared neuron coefficients. The small
  matrices illustrate a mode concentrated on neuron $j$: an incoming row
  and its matching outgoing column.}
  \label{fig:backward-coupling-mechanism}
 \end{minipage}
\end{figure}
Finally, propagate the sensitivity to the boundary below.
At the top, $\mathbf g^L$ is exact.
Below, the exact
sensitivity $\mathbf g^{\ell}$ applies the chain rule
\eqref{eq:main-backward-signals} through the corrected $\mathbf W_{\ell+1}$. A local estimate is also
appealing. Since the column update \eqref{eq:main-coupled-steps} adds
$\mathbf U_\ell\boldsymbol\Gamma^{\rm col}_\ell\mathbf Q_\ell^\top$ to
$\mathbf W_{\ell+1}$, the new term to product $\mathbf W_{\ell+1}^\top\mathbf g^{\ell+1}$ is
$\mathbf Q_\ell\boldsymbol\Gamma^{\rm col}_\ell\mathbf s_\ell$: of the upper
sensitivity, the new columns transmit only the scores. The first-order
estimate keeps this term, with its loading replaced by a linearized estimate
$\mathbf Q^0_\ell$ from the pre-correction rows, and replaces the pre-existing
contribution by its training mean:
\begin{equation}
 \mathbf g^{\ell}=\frac{\sigma'(\mathbf h_\ell)}{\sqrt{p_\ell}}\odot\Bigl(
 \Etrain[\mathbf W_{\ell+1}^{\top}\mathbf g^{\ell+1}]
 +\mathbf Q^{0}_\ell\,\mathbf s_\ell\Bigr),
 \qquad
 \mathbf Q^{0}_\ell\propto\mathbf W_\ell\mathbf V_\ell .
 \label{eq:main-first-order}
\end{equation}
The matrices in this estimate are saved at the start of the pass;
the lower gate uses the corrected activations. Each column of
$\mathbf Q_\ell^0$ is rescaled to match the installed column's coordinate
RMS. Both the mean substitution and the pre-correction loading are
additional approximations. Appendix~\ref{app:algorithm} specifies the
normalization, the derivation, and the alternative signals. Subsequent passes repeat the sweep on the
corrected network.

\paragraph{Over the rainbow ---} A first consequence is that this provide a concrete mechanism for the rainbow network of
\citet{guth2024rainbow}.
First, a forward Neural Lofi network {\it is} a rainbow network albeit a simple one. Indeed each layer is a linear dimension reduction
$\mathbf F_\ell$ followed by white random features $\mathbf R_\ell$, its weight
covariance $\mathbf F_\ell\mathbf F_\ell^{\top}$ has low rank, and the
dependence between layers passes only through the filters, the random lifts
being independent.

In their scattering-network experiments,
\citet[Fig.~10]{guth2024rainbow} observe amplification of selected covariance
directions that largely preserves the initial neuron configuration.
Theorem~\ref{thm:controlled-feedback} gives a local mechanism for this
behavior: its growing modes amplify fixed, initialization-dependent neuron
patterns $\mathbf q_i$.
Backward coupling then goes beyond this picture: Shared coefficients link the
weight covariances of consecutive layers: if their covariance is
$\mathbf C_q$, the selected mode matrices $\mathbf V_\ell$ and $\mathbf U_\ell$
produce $\mathbf V_\ell\mathbf C_q\mathbf V_\ell^\top$ in incoming rows and
$\mathbf U_\ell\mathbf C_q\mathbf U_\ell^\top$ in outgoing columns. They also
couple the weights themselves, which the rainbow model takes to be independent
across layers up to a rotation: the column of $\mathbf W_{\ell+1}$ that reads
neuron $j$ carries the coefficient $q_{ji}$ tied to its row, and the lift of
layer $\ell+1$ acquires the deterministic columns $\mathbf u_i$ next to its
random ones.  In this picture it is joint training that {\it
paints the rainbow}: the layers above choose which covariance a layer grows,
the growth starts from the initialization and keeps it, and the covariances
of consecutive layers come out tied to each other through the neurons they
share.

\paragraph{Coupled modes in linear networks ---}
\label{sec:related-work}
The coupling has an exact counterpart in the one-hidden-layer
{\it linear} network $f(x)=\mathbf W_2\mathbf W_1x$, studied by
\citet[Sections~1.3 and~2]{saxe2014exact}. For unnormalized linear activations, whitened inputs
and squared-loss gradient flow, aligned positive modes of the
input--target correlation $\Etrain[yx^\top]=\sum_i s_i u_i v_i^\top$ obey
\begin{equation}
 \alpha_i:=\|\mathbf W_1v_i\|,\quad
 \beta_i:=\|\mathbf W_2^\top u_i\|,\qquad
 \dot\alpha_i=\beta_i(s_i-\alpha_i\beta_i),\quad
 \dot\beta_i=\alpha_i(s_i-\alpha_i\beta_i).
 \label{eq:main-linear-modes}
\end{equation}
The growing components scale as $e^{s_it}$ while
$\alpha_i\beta_i\ll s_i$; their product then saturates at $s_i$.
For deeper networks with aligned singular directions, each
layer's update contains the product of the mode amplitudes in all
other layers. With $D$ weight matrices and equal positive amplitudes
across layers, the end-to-end mode strength $m_i$ satisfies
\begin{equation}
 \dot m_i=D\,m_i^{\,2-2/D}(s_i-m_i).
 \label{eq:main-deep-linear-mode}
\end{equation}
For $D=2$, this gives exponential growth at small $m_i$, followed by
saturation. For $D>2$, the additional powers of $m_i$ change the
early amplification law. Appendix~\ref{sec:linear-dynamics} derives
both cases.
Theorem~\ref{thm:controlled-feedback} identifies the local counterpart
for nonlinear networks: a small adjacent row and column amplify modes
of a cross-layer correlation matrix determined by the surrounding features.
With scalar targets the linear correlation has rank at most one;
input-dependent nonlinear sensitivities can generate multiple cross-layer modes.
Backward coupling thus lets the response learned above alter
feature selection below, while the corrected activations determine
transmission upward. Section~\ref{sec:lessons} characterizes this selection
in function space and follows the resulting kernel evolution.

\section{Lessons from learning with backward coupling}
\label{sec:lessons}
Backward coupling has a function-space description parallel to, but richer than, Neural LoFi:
a selection principle at fixed representation, followed by an evolution of
the representation kernels.

\paragraph{A relevance--complexity trade-off for pairs ---}
Neural LoFi selects a feature whose square correlates maximally with the label,
$\max_{\|\varphi\|_{\mathcal H_{\ell-1}}=1}|\Etrain[y\varphi^2]|$
\citep{dandi_neural_lofi}. Backward coupling instead selects a pair: let
$K_{\ell-1}(\mathbf x,\mathbf x')=\langle\mathbf z_{\ell-1}(\mathbf x),
\mathbf z_{\ell-1}(\mathbf x')\rangle$ and
$G_{\ell+1}(\mathbf x,\mathbf x')=\langle\mathbf g^{\ell+1}(\mathbf x),
\mathbf g^{\ell+1}(\mathbf x')\rangle$, with RKHSs
$\mathcal H_{\ell-1}$ and $\mathcal H^g_{\ell+1}$.
The leading singular pair of $\mathbf A_\ell$ gives
$\psi=\langle\mathbf u,\mathbf g^{\ell+1}\rangle$ and
$\chi=\langle\mathbf v,\mathbf z_{\ell-1}\rangle$ solving
\begin{equation}
 s_1(\mathbf A_\ell) =
 \max_{
    \|\psi\|_{\mathcal H^g_{\ell+1}}\leq1,\,
        \|\chi\|_{\mathcal H_{\ell-1}}\leq1
 }
 \Etrain[y\,\psi\,\chi].
 \label{eq:variational-feature-pair}
\end{equation}
The RKHS constraints favor features that are simple to express in each
layer's geometry. Within these constraints,
\eqref{eq:variational-feature-pair} selects a lower feature $\chi$ and
an upper response $\psi$ whose product is maximally correlated with $y$;
neither feature needs to be predictive on its own. The paired row and
column updates learn the lower cue in the same neurons that transmit
the upper response, allowing the cue to modulate where that response
is expressed.
For example, a fur-texture cue could be coupled to an
upper response associated with an eye or ear detector, strengthening
that response where fur is present when their conjunction is informative
about the label. This selective activation refines where higher-level
features are expressed, and thereby the geometry of the representation.

\paragraph{Exact theory for Kernel dynamics with three layers ---}
\label{sec:variational-coupling}
\label{sec:kernel-evolution}
Joint training changes both the representations and the sensitivities
that guide feature selection. Theorem~\ref{thm:three-layer-closed-flow}
describes their coupled evolution in a three-layer network and its
sequential infinite-width limit for a continuous version of backward
coupling. In this version, the same selected lower perturbation is added
to every first-layer filtered coordinate, forward Neural LoFi updates the second-layer
features as the lower kernel changes, and gradient flow trains the
final readout. The row correction, column correction and readout
gradient update all use the same step size $dt$ before taking the
continuous-time limit. Appendix~\ref{app:three-forward-flow} gives the
algorithm, initialization and assumptions.

We assume that a forward Neural LoFi pass has already filtered the
features at layers $1$ and $2$ to fixed dimensions $k_1$ and $k_2$, respectively.
During the subsequent evolution, these filtered features are
$\mathbf h_{1,t}^\star(x)=\mathbf F_{1,t}^\top\mathbf z_{0,t}(x)\in\mathbb R^{k_1}$ and
$\mathbf h_{2,t}^\star(x)=\mathbf F_{2,t}^\top\mathbf z_{1,t}(x)\in\mathbb R^{k_2}$.
Let $\mathbf r\sim\mathcal N(0,I_{k_1})$ and
$\eta\sim\mathcal N(0,I_{k_2})$ be independent Gaussian vectors that
specify random projections of the filtered coordinates in the first
and second hidden layers, respectively. At finite widths $p_1,p_2$,
the first-layer lift consists of $p_1$ i.i.d.\ draws
$\mathbf r_1,\ldots,\mathbf r_{p_1}$, and the second-layer lift of
$p_2$ i.i.d.\ draws $\eta_1,\ldots,\eta_{p_2}$, sampled independently
of the first collection. Each draw indexes one neuron and remains fixed
while the filtered features evolve.

After taking the continuous-time limit, we let $p_1\to\infty$ and then
$p_2\to\infty$. Averages over the sampled neurons become Gaussian
expectations over $\mathbf r$ and $\eta$. We use these vectors to index
lower and upper neurons, respectively. To retain the column updates
added during training, we define the coupling function $C_t(\eta,\mathbf r)$ as their total
contribution from lower neuron $\mathbf r$ to upper neuron $\eta$,
with $C_0=0$. Its action on the current lower activations gives the
correction to the upper preactivation. The preactivations $h_{1,t},h_{2,t}$
and activations $z_{1,t},z_{2,t}$ are therefore
\begin{align}
 h_{1,t}(\mathbf r,x)&=\mathbf r^\top\mathbf h_{1,t}^\star(x),\qquad
 z_{j,t}=\sigma(h_{j,t}),\nonumber\\
 h_{2,t}(\eta,x)&=\eta^\top\mathbf h_{2,t}^\star(x)
       +\mathbb E_{\mathbf r}[C_t(\eta,\mathbf r)z_{1,t}(\mathbf r,x)].
 \label{eq:main-gaussian-features}
\end{align}

Let $w_{3,t}(\eta)$ be the readout coefficient of upper neuron $\eta$.
We further define a response function $S_t(\eta,s)$ that records this neuron's contribution
to the output sensitivity at a preactivation $s\in\mathbb R$.
This response and the scalar output $h_{3,t}$ are given by
\[
 S_t(\eta,s)=w_{3,t}(\eta)\sigma'(s),\qquad
 h_{3,t}(x)=\mathbb E_\eta[w_{3,t}(\eta)z_{2,t}(\eta,x)].
\]
Evaluating the response at the current preactivation gives
$g_t^2(\eta,x)=S_t(\eta,h_{2,t}(\eta,x))$.
The feature kernels $K_{1,t},K_{2,t}$ and sensitivity kernel $G_t$
average products at $x,x'$ over the same Gaussian lift row:
\begin{align*}
 K_{1,t}(x,x')&=\mathbb E_{\mathbf r}
       [z_{1,t}(\mathbf r,x)z_{1,t}(\mathbf r,x')],\\
 K_{2,t}(x,x')&=\mathbb E_\eta[z_{2,t}(\eta,x)z_{2,t}(\eta,x')],\qquad
 G_t(x,x')=\mathbb E_\eta[g_t^2(\eta,x)g_t^2(\eta,x')].
\end{align*}
The limiting evolution closes in the filtered features and scalar
output together with the functions $C_t(\eta,\mathbf r)$ and $S_t(\eta,s)$
defined above.
\begin{samepage}
\begin{theorem}[Informal: coupled dynamics in three layers]
\label{thm:three-layer-closed-flow}
Under the assumptions of Appendix~\ref{app:three-forward-flow}, the discrete
updates converge under the sequential limits $dt\to0$, $p_1\to\infty$ and
$p_2\to\infty$ to the following set of self-consistent equations, defined
over the training inputs. At each time $t$, let $(\psi_t,\chi_t)$ denote the maximizer of
\eqref{eq:variational-feature-pair} for the current sensitivity kernel $G_t$
and the input kernel $K_0$, with value $\rho_t>0$.
The filtered features, coupling and response satisfy
\begin{align}
 \dot h_{1,t,a}^\star
 &=\chi_t\quad(a=1,\ldots,k_1),\qquad
 \mathbf h_{2,t}^\star=\Fwd_2(K_{1,t},y),
 \label{eq:main-coupled-filter-flow}\\
 \partial_t C_t(\eta,\mathbf r)
 &=\frac{1}{\rho_t}
   \Etrain[y\chi_t(X)g_t^2(\eta,X)]
   \Etrain[y\psi_t(X)z_{1,t}(\mathbf r,X)],\nonumber\\
 \partial_t S_t(\eta,s)
 &=\sigma'(s)\Etrain[(y-h_{3,t}(X))z_{2,t}(\eta,X)],
 \label{eq:main-coupled-response-flow}
\end{align}
where $\Fwd_2(K_{1,t},y)$ gives the forward Neural LoFi features at
layer 2 from the current kernel and labels. The output satisfies
\begin{align}
 \dot h_{3,t}(x)
 &=\Etrain[K_{2,t}(x,X)(y-h_{3,t}(X))]
 +\mathbb E_\eta[g_t^2(\eta,x)\eta]^\top\dot{\mathbf h}_{2,t}^{\star}(x)
 \nonumber\\
 &\quad+\psi_t(x)\Etrain[K_{1,t}(x,X)y\psi_t(X)]
 \nonumber\\
 &\quad+\mathbb E_{\eta,\mathbf r}
       [g_t^2(\eta,x)C_t(\eta,\mathbf r)\dot z_{1,t}(\mathbf r,x)].
 \label{eq:main-three-informal-flow}
\end{align}
\end{theorem}
\end{samepage}

Theorem~\ref{thm:three-layer-closed-flow} extends the
layer-wise kernel constructions of forward Neural LoFi and AGOP-based
Deep Recursive Feature Machines
\citep{radhakrishnan2024mechanism,beaglehole2024deepcollapse,dandi_neural_lofi}
to a continuous evolution in which upper and lower representations
are {\it revised together}. The learned upper response changes which lower
feature is selected, learning that feature then changes the upper
kernel and the response itself. Thus the space in which features are
selected evolves with training, in contrast with selection of hierarchical features from a prescribed family \citep{bach2008hierarchical}.
Following this evolution requires only the current features, output,
accumulated coupling $C_t$ and response $S_t$. The equations are
autonomous in these quantities, with earlier column updates retained
in $C_t$, without the history integrals and two-time response functions
of the dynamical mean-field formulation of
\citet{bordelon2022selfconsistent}.

Because the corrections are selected by
\eqref{eq:variational-feature-pair}, this evolution can be studied
through the signal-to-noise ratio of correlations between layers.
An upper response can make a lower feature detectable through its
correlation with the modified label $y\psi_t$, even when the feature
had little correlation with $y$ alone. For the current kernels, an
informative leading population singular pair is reliably selected
when its gap dominates the empirical operator error
\citep{yu2015davis_kahan}. Controlling this error along the trajectory
would extend the emergence predictions of Neural LoFi
\citep{dandi_neural_lofi} from successive layers to joint learning,
using kernel concentration and spectral perturbation theory to
predict when an upper response enables a new lower feature to emerge.

\section{Mechanistic Numerical Experiments}
\label{sec:numerics}
We now illustrate the considerable advantages of backward feature correction with BackLofi over Lofi and layer-wise training in different situations.

\paragraph{Synthetic setting: The Deep Staircase ---}\label{sec:hierarchical-discovery}
We first consider a class of synthetic hierarchical models
\citep{nichani2024provable,wang2023learning,fu2025learning,
tabanelli2026deep,dandi_neural_lofi,wortsmanzurich2026scalinglaws},
adapted here to isolate the role of backward information.
The construction is also related to staircase models
\citep{NEURIPS2021_e2db7186,pmlr-v178-abbe22a,pmlr-v195-abbe23a,
JMLR:v25:23-1543,pmlr-v235-dandi24a,pmlr-v258-troiani25a,
tabanelli2025computational}, with one key difference: after the visible
features have been learned, the next useful directions become accessible
only through a backward signal. We refer to this mechanism as the
\emph{deep staircase}.

Specifically, we consider $n$ samples $(\mathbf x_\mu,y_\mu)$ with Gaussian
inputs $\mathbf x_\mu\sim\mathcal N(0,I_d)$ and scalar outputs
$y_\mu\in\R$, and define the sample exponent
$\alpha=\log n/\log d$.
Let $\Phi_2(\mathbf x)\in\R^{D_2(d)}$, with
$D_2(d)=d(d+1)/2$, denote the complete normalized quadratic Hermite
representation of $\mathbf x$. The first level contains
$d_1=p_U+p_V+p_W$ dense quadratic features,
\begin{equation}
    \mathbf h_U=A_U\Phi_2(\mathbf x),
    \qquad
    \mathbf h_V=A_V\Phi_2(\mathbf x),
    \qquad
    \mathbf h_W=A_W\Phi_2(\mathbf x),
    \label{eq:toy-main-blocks}
\end{equation}
where the three blocks have distinct roles. The block $\mathbf h_U$ is the
\emph{visible anchor}, while $\mathbf h_V$ and $\mathbf h_W$ are relevant
to the target but are constructed to be \emph{hidden} to the ordinary forward
population statistic. The second level combines two quadratic coordinates of the visible block,
\begin{equation}
    Q
    =
    \frac{
        \mathbf h_U^\top B_Q\mathbf h_U-\tr(B_Q)
    }{\sqrt2},
    \qquad
    Z
    =
    \frac{
        \mathbf h_U^\top B_Z\mathbf h_U-\tr(B_Z)
    }{\sqrt2},
    \qquad
    S
    =
    \sign(\mathbf h_V^\top C\mathbf h_W).
    \label{eq:toy-main-coordinates}
\end{equation}
We choose $B_Q$ and $B_Z$ symmetric, normalized and orthogonal, with a
sign-flip symmetry that leaves $Q$ invariant while sending $Z$ to $-Z$.
The raw output is
\begin{equation}
    y_{\rm raw}
    =
    \tanh\!\left(
        \frac{Q+ZS}{\sqrt2}+b_0
    \right),
    \label{eq:toy-main-target}
\end{equation}
We finally standardize $y_{\rm raw}$ of mean $\mu_y$ and variance $\sigma^2_y$ to $y = (y_{\rm raw}-\mu_y)/\sigma_y$ with zero mean and unit variance.. Schematically,
\begin{align}
    \mathbf x
    \;
    \longrightarrow
    \;
    \mathbf h^{(1)}
    =
    \big(
        \underbrace{\mathbf h_U}_{\text{visible}},
        \underbrace{\mathbf h_V}_{\text{hidden}},
        \underbrace{\mathbf h_W}_{\text{hidden}}
    \big)
    \;
    \longrightarrow
    \;
    h^{(2)}
    =
    Q(\mathbf h_U)
    +
    Z(\mathbf h_U)S(\mathbf h_V,\mathbf h_W)
    \;
    \longrightarrow
    \;
    y=g(h^{(2)}).
 \label{eq:toy-main-hierarchy}
\end{align}
The precise teacher distributions and normalizations are given in
Appendix~\ref{app:toy-setting}. The mechanism is particularly transparent after decomposing the target.
Set
$t_\pm=(Q\pm Z)/\sqrt2+b_0$, the target reads
\begin{equation}
    a(Q,Z)
    =
    \frac{1}{2\sigma_y}\big(
        \tanh(t_+)+\tanh(t_-)-2\mu_y
    \big),
    \qquad
    b(Q,Z)
    =
    \frac{1}{2\sigma_y}\big( \tanh(t_+)-\tanh(t_-) \big),
\end{equation}
\begin{equation}
    y=a(Q,Z)+b(Q,Z)S.
\end{equation}
By construction,
$a(Q,-Z)=a(Q,Z)$ and $b(Q,-Z)=-b(Q,Z)$
Thus $a$ carries the visible part of the hierarchy, while $bS$ contains the
hidden interaction.

\begin{figure}[t]
 \centering
 \includegraphics[width=\textwidth]{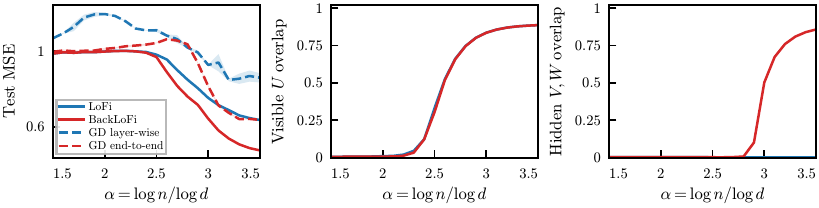}
 \caption{
    \textbf{Backward coupling unlocks the deep staircase.}
    Performance as a function of the sample exponent
    $\alpha=\log n/\log d$.
    \emph{Left:} Test MSE.
    \emph{Center:} Overlap with the visible first-level block $U$.
    \emph{Right:} Overlap with the hidden blocks $V,W$.
    Both spectral methods recover the anchor, the backward
    correction also recovers the hidden blocks. The prediction panel
    additionally compares layer-wise and end-to-end gradient descent.
    }
 \label{fig:toy-main}
\end{figure}

These parities also summarize what the ordinary forward hierarchy can
learn. At the first level, the label-weighted covariance has a nonzero
population component on $\mathbf h_U$, while its entire
$\mathbf h_V\oplus\mathbf h_W$ block cancels.

Forward LoFi therefore
recovers the visible first-level subspace. Once $\mathbf h_U$ is available,
the upper Forward step acts on its quadratic representation:
the first moment identifies $Q$, while the second-order statistic, after removing the $Q$ direction, identifies $Z$. Hence the forward staircase reaches
the visible hierarchy
\begin{equation}
    \mathbf x
    \longrightarrow
    \widehat{\mathbf h}^{(1)} = (\widehat{\mathbf h}_U, 0, 0)
    \longrightarrow
    \widehat{\mathbf h}^{(2)} = (\widehat Q, \widehat Z),
\end{equation}
but it still cannot access the hidden $\mathbf h_V,\mathbf h_W$ branch.

The backward pass provides the missing stepping stone. Specializing the
cross-layer estimator of Section~\ref{sec:setting} to the first boundary,
\begin{equation}
    \mathcal A
    =
    \mathbb E\!\left[
        y\,\mathbf g^{(2)}(\mathbf x)
        \mathbf z_0(\mathbf x)^\top
    \right],
    \label{eq:toy-main-cross-layer}
\end{equation}
where $\mathbf g^{(2)}$ is the sensitivity of the fitted upper predictor.
A retained left singular mode $u$ defines the sample-dependent score
$s=u^\top\mathbf g^{(2)}$. For the present teacher, the informative score
has a a non-zero linear
component along $Z$. The resulting Forward--Backward--Forward mechanism can
therefore be summarized as
\begin{equation}
    \mathbf x
    \xrightarrow{\rm Forward}
    \widehat{\mathbf h}_U
    \xrightarrow{\rm Forward}
    (\widehat Q,\widehat Z)
    \xrightarrow{\rm backward}
    s(\mathbf x)\ni \eta\,\widehat Z
    \xrightarrow{\,y\mapsto ys\,}
    \widehat{\mathbf h}^{(1)}_B
    =
    \bigl(
        \widehat{\mathbf h}_U,
        \widehat{\mathbf h}_{VW}
    \bigr)
    \xrightarrow{\rm Forward}
    \widehat{\mathbf h}^{(2)}_B .
    \label{eq:toy-main-backlofi-pipeline}
\end{equation}
Here the subscript $B$ denotes the representation obtained after the
backward correction. The backward score turns the previously invisible hidden interaction into a nonzero lower-level population signal:
\begin{equation}
    \mathbb E\!\left[
        y\,\mathbf h_V\mathbf h_W^\top
    \right]
    =0,
    \qquad
    \mathbb E\!\left[
        y\,s\,\mathbf h_V\mathbf h_W^\top
    \right]
    \neq0.
    \label{eq:toy-population-separation}
\end{equation}
Moreover, this cancellation persists under any scalar transformation of the
label: $g:\mathbb{R} \rightarrow \mathbb{R}$, analogous to the Statistical Query staircase settings in \cite{pmlr-v258-troiani25a, damian2025the}, namely, $\mathbb E\!\left[g(y)\,\mathbf h_V\mathbf h_W^\top\right]=0$.
The corrected lower pass can therefore recover the missing joint hidden
subspace, after which an ordinary Forward-LoFi pass is rerun on the
corrected representation. Figure~\ref{fig:toy-main} exhibits the same
qualitative separation for gradient descent and its spectral surrogate:
joint training and BackLoFi recover hidden directions missed by their
layer-wise counterparts.

The corresponding sample scales make the staircase quantitative. In the
regime
$p_U,p_V,p_W\asymp d^\epsilon$, $0<\epsilon<1$, the visible lower
subspace appears at
\begin{equation}
    n_U\asymp d^{2+\epsilon},
    \qquad
    n_{Q,Z}\asymp d^{2\epsilon},
    \qquad
    n_{\rm ForwardLoFi}\asymp d^{2+\epsilon},
\end{equation}
while, conditional on an informative backward score, recovery of the hidden
subspace again requires $d^{2+\epsilon}$ samples. The additional bottleneck
is learning the score itself. Appendix~\ref{app:toy-finite-sample}
derives a $d^{-1/2}$ contribution to the cross-layer matrix from the
quadratic feature contraction, leading to the sample scale
\begin{equation}
    n_{\rm score}\asymp d^3,
    \qquad
    n_{\rm BackLoFi}
    \asymp
    \max\{d^{2+\epsilon},d^3\}
    =
    d^3.
    \label{eq:toy-main-sample-scales}
\end{equation}
For the numerical choice $\epsilon=1/2$, this predicts the two characteristic
scales $\alpha_U=\frac52$ and $\alpha_{\rm Back}=3$. The $d^3$ prediction is conditional on the $d^{-1/2}$ informative
cross-layer gap. Appendix~\ref{app:toy-setting} gives the complete
derivation and the assumptions underlying these scalings.

\paragraph{Image classification  ---}\label{sec:experiments}
We now test whether the coupling isolated in
Section~\ref{sec:setting} accounts for what joint training learns on real
data.  We run Algorithm~\ref{alg:combined} on a fully connected and on a
convolutional multi-layer network on binary CIFAR-10, and compare it at matched width with
forward Neural LoFi, with the random-features network of the same
architecture, and with full-batch gradient descent
(Figure~\ref{fig:test-error-vs-n}; settings in
Appendix~\ref{app:hyperparameters}).

\begin{figure}[t]
 \centering
 \includegraphics[width=0.48\textwidth]{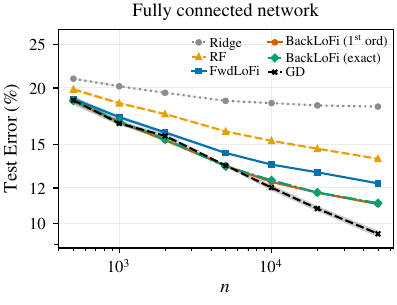}\hfill
 \includegraphics[width=0.48\textwidth]{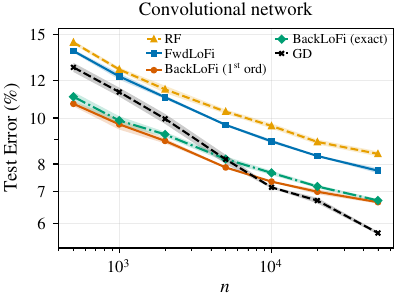}
 \caption{\textbf{Backward coupling improves on forward Neural LoFi and
 approaches gradient descent.} Test error vs number of training
 samples $n$ on binary CIFAR-10 (animal vs.\ vehicle), at matched width, for a
 fully connected network (\textit{left}) and a convnet
 (\textit{right}). Ridge is a linear regression on the pixels, RF the
 random-features network of the same architecture, FwdLoFi forward Neural
 LoFi, BackLoFi Algorithm~\ref{alg:combined} with the first-order or the exact
 backward signal, and GD full-batch gradient descent on the same architecture.
 Curves are means over seeds and bands one std error;
Spectral configurations and GD
 hyperparameters are selected on mean test error (see App.~\ref{app:hyperparameters}).}
 \label{fig:test-error-vs-n}
\end{figure}

One can read Fig.~\ref{fig:test-error-vs-n} in three steps, each holding for both architectures. First, BackLoFi improves on forward
 LoFi at every training-set size and at equal width. The information
that the layers above return to a boundary selects directions that the
forward filter, built from the same representation, did not find. The gain is
largest when data are scarce: on the convolutional network, BackLoFi trained on
$5{,}000$ images comes within $0.15$ points of forward Neural LoFi trained on
$50{,}000$. The replacement of features builds up over a few backward passes, until about half of the forward
columns have been replaced (see Appendix~\ref{app:additional-simulations}).

Second, the exact and the first-order backward signals give
nearly the same curves. Of the sensitivity of the layers above, the
first-order signal keeps only the scores
$\mathbf s_\ell$, one number per
selected mode. What the layers above pass down is therefore
captured by these $r_\ell$ numbers per sample.

Third, BackLoFi follows gradient descent on the same
architecture. It matches it up to $n=5{,}000$ on the fully connected network
and is ahead of it on the convolutional one, by two points at $n=500$.
Gradient descent pulls ahead only on the largest training sets, where
BackLoFi still recovers about half of the gap between forward Neural LoFi and
joint training. This is consistent with joint training using more than a few
coupled modes per boundary once data are abundant. The same mechanism appeared
in the deep staircase above, where directions invisible to the forward
statistic become visible through the response of the layers above.

Over most of the sample range, then, what joint training adds
to layer-wise learning is a low-rank coupling between the rows of a layer and
the columns that read them. Installing a few such modes per boundary, with no
gradient step and using only the scores of the layers above, reproduces it.

\section{Conclusion}
Backward Neural LoFi extends layer-wise spectral learning by
coupling lower features to upper responses: information learned
above changes which directions are selected below. The resulting
feature--response pairs admit a variational characterization,
while their local dynamics provide a mechanism for joint
amplification and structured weight covariances. For a continuous
three-layer variant, we prove convergence to a self-consistent evolution
of the nonlinear features, kernels, and output. The hierarchical
experiment illustrates hidden-feature discovery, and the image
experiments show improvements over forward learning at matched width.
Controlling repeated corrections beyond the local regime and establishing finite-sample guarantees for the implemented algorithm with learned upper scores remain open.
We believe these results help to open new directions into the understanding of the {\it unreasonable efficiency} of learning with deep networks \citep{lecun2015deep}.

\section*{Acknowledgments}
We would like to thank Surya Ganguli, Bruno Loureiro, Andrew Saxe, Matthieu Wyart and  Lenka Zdeborov\'{a} for discussions. We acknowledge funding from the Swiss National Science Foundation grants OperaGOST (grant number $200021\_200390$) and DSGIANGO (grant number $225837$). This work was supported by the Simons Collaboration on the Physics of Learning and Neural Computation via the Simons Foundation grant ($\#1257412$) and by the EPFL AI Center PhD Fellowship Program 2026.

\bibliography{references}
\bibliographystyle{plainnat}

\clearpage
\appendix

\section{Further related works}
\label{app:related-work}
We collect here the lines of work that the main text touches only briefly,
organized by their relation to backward coupling.

\paragraph{Feature correction by the layers above ---}
The closest antecedent is the \emph{backward feature correction} of
\citet{allenzhu2023backward}, where jointly trained upper layers refine the
representation learned below. Backward coupling isolates a narrower and more
explicit effect: an upper response, entering as a modified label, changes
which directions the lower spectral step can see at all. The same reversal of
the forward order appears in a different guise in the alternating-gradient-flow
analyses of \citet{kunin2025alternating}, where features of a two-layer network
activate one after the other through plateaus and sharp drops, and in the
sequential learning of compositional tasks of \citet{marchetti2026sequential}.

\paragraph{Multi-pass learning and staircases ---}
A single low-degree statistic of the input misses targets with information or
leap exponent above two \citep{arous2021online,benarous2024stochastic,
damian2024computational,pmlr-v258-troiani25a}. Repeated passes over the data
change this: reusing batches transforms the effective label and lets gradient
descent escape these constraints
\citep{pmlr-v235-dandi24a}, and staircase-type targets are learned one
correlation at a time, each step being unlocked by the previous one
\citep{NEURIPS2021_e2db7186,pmlr-v195-abbe23a,bardone2024sliding,
tabanelli2025computational}. The generative-exponent picture
\citep{NEURIPS2025_290056cc} quantifies what label transformations can achieve.
BackLoFi sits in this family with a specific choice of transformation, the
sample-wise response of the fitted layers above, and a specific place where it
acts, the spectral step of an earlier layer.

\paragraph{Hierarchical learning and depth ---}
\citet{dandi2026computational} study Gaussian hierarchical single- and
multi-index targets, showing how successive feature learning reduces the
effective dimension and yields a sample-complexity advantage over shallow
networks. In a related compositional setting, \citet{tabanelli2026deep}
recover latent features with layer-wise spectral estimators and establish
sample-complexity gains over shallow kernel methods.
\citet{wortsmanzurich2026scalinglaws} derive recovery thresholds for
individual features with power-law strengths: as the sample size increases,
successively weaker features become detectable, producing power-law decay
of the prediction error. For a multiscale single-index model,
\citet{bruna2026multiscale} proves recovery by online SGD on a correlation
objective with all layers evolving on the same timescale.

For discrete compositional data, \citet{cagnetta2024hierarchy} introduce
the Random Hierarchy Model and relate the emergence of representations
invariant to equivalent feature combinations to the detectability of
correlations between input patches and labels.
\citet{ren2026hierarchy} prove efficient learning of these hierarchies
under nondegeneracy assumptions by training a deep convolutional network
layer by layer. For Boolean targets with hierarchical Fourier structure,
\citet{ren2026algorithmic} give a layer-wise coordinate-descent algorithm for
logarithmic-depth networks and prove an approximation separation from
constant-depth networks with polynomial width, regular activations, and
controlled spectral norms.

\paragraph{Hessian and spectral learning ---}
Label-weighted second-order operators drive spectral initializations and
recovery in single- and multi-index estimation
\citep{lu2020phase,mondelli2018fundamental,maillard2020phase,
vilucchio2025asymptotics}, and describe the early dynamics of shallow
networks through Hessian-like objects
\citep{bonnaire2025role,zhang2025neural,montanari2026phase,
defilippis2026optimal}; the selection of spectral outliers against a bulk
follows the BBP/EJ mechanism \citep{baik2005phase,edwards1976eigenvalue},
as in spiked matrix--tensor models
\citep{ros2019complex,sarao2019afraid,sarao2020marvels}. Neural LoFi
\citep{dandi_neural_lofi} applies one such operator per layer. The
cross-layer matrix $\mathbf A_\ell$ of this paper correlates two
representations through the labels; its leading singular pairs select
the coupled features.

\paragraph{Spectral structure of trained weights ---}
Trained networks show low-rank, non-random structure in their weights:
spikes over a bulk in quadratic and diagonal networks
\citep{defilippis2026scaling}, nuclear-norm-type capacity control in
overparameterized quadratic networks \citep{erbanuclear}, similar spectra in
attention layers \citep{boncoraglio2025single}, and structured weight
spectra across architectures \citep{martin2021implicit}. The rainbow model
\citep{guth2024rainbow} organizes these observations into layers of random
features with learned low-rank covariances, assumed independent across layers
up to a rotation. Section~\ref{sec:lessons} shows what backward coupling
adds: linked covariances and a neuron-by-neuron dependence between the rows
of one layer and the columns of the next.

\paragraph{Feature learning from gradients of a fitted predictor ---}
Recursive feature machines update the feature metric with the average
gradient outer product (AGOP) of a fitted predictor
\citep{radhakrishnan2024mechanism,radhakrishnan2025linear,
beaglehole2023mechanism,boixadsera2026fact,cha2026weightgram}. These methods also feed
information from the output back into the features, through the gradient of
a predictor that must first be fitted. In backward coupling, the upper
sensitivity $\mathbf g^{\ell+1}$ plays a similar role: together with $y$,
it weights a first moment of the fixed lower representation, producing
the cross-layer matrix $\mathbf A_\ell$.

\paragraph{Local and backpropagation-free training ---}
Layer-wise and local training rules avoid an exact backward pass: feedback
alignment and its direct variants replace the transposed weights by random
matrices \citep{lillicrap2016random,nokland2016direct,launay2020direct},
local losses train layers separately \citep{nokland2019training,
belilovsky2020decoupled}, target propagation passes layer-wise targets
\citep{lee2015difference}, and label-modulated Hebbian rules
approximate spectral operators \citep{oja1982simplified,sanger1989optimal,
gerstner2018eligibility,illing2019biologically,illing2021local}. The
backward signal of BackLoFi is itself low-rank and approximate, carried by a
few scores per sample, in the spirit of these rules.

\section{Coupled mode amplification: proof of Theorem~\ref{thm:controlled-feedback}}
\label{app:scaled-feedback-proof}

We now give the full formal statement and assumptions for
Theorem~\ref{thm:controlled-feedback}.
Theorem~\ref{thm:controlled-feedback-formal} specifies the initialization
scales and training interval, and bounds the error between gradient flow
and the coupled linear dynamics used in the main text. It also shows that
the incoming row and matching outgoing column share the same exponentially
growing term when projected onto paired singular vectors.

\subsection{Assumptions and formal statement}

We use the network notation of Section~\ref{sec:setting}.
Fix a neuron $j$ in layer $\ell<L$, with incoming row
$\mathbf w_{\ell,j}$ and matching outgoing column
$\boldsymbol\omega_{\ell+1,j}$. We train this pair jointly and keep
all other weights frozen. With $x=\mathbf z_{\ell-1}$, the next
preactivation is
\begin{equation}
 \mathbf h_{\ell+1}(x)=\mathbf h_{\ell+1,0}(x)
 +\frac{1}{\sqrt{p_\ell}}
   \boldsymbol\omega_{\ell+1,j}\,
   \sigma(\langle\mathbf w_{\ell,j},x\rangle).
 \label{eq:scaled-pair-model}
\end{equation}
Here $\mathbf h_{\ell+1,0}$ is the fixed contribution of all neurons in
layer $\ell$ except $j$.

\begin{assumption}[Joint training of one pair in a fixed network]
\label{ass:pair-context}
The dataset, inputs $x$, preactivation contributions
$\mathbf h_{\ell+1,0}(x)$, depth and widths are fixed. The selected pair
$(\mathbf w_{\ell,j},\boldsymbol\omega_{\ell+1,j})$ follows equal-rate
Euclidean gradient flow for
$\tfrac12\Etrain[(y-f)^2]$; all other weights remain frozen
during the entire training interval. The subsequent layers have no biases and the
output intercept is zero.
\end{assumption}

\begin{assumption}[Activation]
\label{ass:pair-activation}
The activation is $C^3$, with bounded derivatives on the relevant
bounded preactivation sets, $\sigma(0)=0$, and $\sigma'(0)>0$.
\end{assumption}

\begin{assumption}[Layer-wise vanishing initialization]
\label{ass:pair-hierarchy}
Initialize
$\mathbf w_{\ell,j}(0)=\alpha_\ell\overline{\mathbf w}_{\ell,j}$ and
$\boldsymbol\omega_{\ell+1,j}(0)=
\alpha_{\ell+1}\overline{\boldsymbol\omega}_{\ell+1,j}$.
The frozen subsequent weights are
$\mathbf W_k=\alpha_k\overline{\mathbf W}_k$ for
$k=\ell+2,\ldots,L$, and $\mathbf a_L=\alpha_{L+1}a_0$.
The normalized initial vectors $\overline{\mathbf w}_{\ell,j}$,
$\overline{\boldsymbol\omega}_{\ell+1,j}$, matrices
$\overline{\mathbf W}_k$ and readout $a_0$ are fixed. The scales are
positive and satisfy
\begin{equation}
 \alpha_\ell\longrightarrow0,\qquad
 \alpha_{k+1}=o(\alpha_k)\quad(k=\ell,\ldots,L).
 \label{eq:scaled-strict-hierarchy}
\end{equation}
\end{assumption}
This is the scale hierarchy of
\citet[Assumption~1]{dandi_neural_lofi}, applied to the selected pair
and the frozen weights above it.

To obtain the linear approximation, evaluate the upper sensitivity at
$\mathbf h_{\ell+1,0}(x)$, with the selected neuron's contribution removed.
This gives the fixed sensitivity and correlation matrix
\begin{equation}
 \mathbf g_0(x)=
 \left.\frac{\partial f}{\partial\mathbf h_{\ell+1}}
 \right|_{\mathbf h_{\ell+1}=\mathbf h_{\ell+1,0}(x)},\qquad
 \mathbf A_\ell=\Etrain[y\mathbf g_0x^\top].
 \label{eq:scaled-theorem-operator}
\end{equation}
The activation derivative in the row and column updates contributes the
factor $\sigma'(0)/\sqrt{p_\ell}$. The largest growth rate of their linear
approximation is therefore
\[
 \lambda=\frac{\sigma'(0)}{\sqrt{p_\ell}}\|\mathbf A_\ell\|_2.
\]
Both $\mathbf A_\ell$ and $\lambda$ depend on the initialization scales;
they remain fixed during training.

The sensitivity contains one weight factor from each subsequent layer,
including the readout, so its size is bounded by a constant times
$\prod_{k=\ell+2}^{L+1}\alpha_k$. We assume that averaging against the
labels and inputs leaves a correlation of this same order:
\begin{equation}
 \lambda\ge c\prod_{k=\ell+2}^{L+1}\alpha_k
 \quad\text{for some fixed $c>0$ and all sufficiently small scales.}
 \label{eq:pair-correlation-scale}
\end{equation}
This condition excludes cancellation of the leading correlation that
drives the coupled growth.

\begin{samepage}
\begin{formaltheorem}[Formal version of Theorem~\ref{thm:controlled-feedback}]
\label{thm:controlled-feedback-formal}
Under Assumptions~\ref{ass:pair-context}--\ref{ass:pair-hierarchy}
and \eqref{eq:pair-correlation-scale}, let the linear approximation solve
\[
 \dot{\mathbf w}_{\ell,j}^{\rm lin}
   =\frac{\sigma'(0)}{\sqrt{p_\ell}}
     \mathbf A_\ell^\top\boldsymbol\omega_{\ell+1,j}^{\rm lin},
 \qquad
 \dot{\boldsymbol\omega}_{\ell+1,j}^{\rm lin}
   =\frac{\sigma'(0)}{\sqrt{p_\ell}}
     \mathbf A_\ell\mathbf w_{\ell,j}^{\rm lin}
\]
from the same initialization. For any fixed $M>0$ and sufficiently small
initialization scales, the following bounds hold uniformly for
$0\le t\le T$ with $\lambda T\le M$:
\begin{align}
 \|\mathbf w_{\ell,j}(t)\|
 &\le C_M\alpha_\ell, &
 \|\mathbf w_{\ell,j}(t)-\mathbf w_{\ell,j}^{\rm lin}(t)\|
 &\le C_M\alpha_\ell^2,\nonumber\\
 \|\boldsymbol\omega_{\ell+1,j}(t)\|
 &\le C_M\alpha_\ell, &
 \|\boldsymbol\omega_{\ell+1,j}(t)-\boldsymbol\omega_{\ell+1,j}^{\rm lin}(t)\|
 &\le C_M\alpha_\ell^2.
 \label{eq:scaled-theorem-bound}
\end{align}
Here $C_M$ is independent of the initialization scales and time.
Hence, the weights move within their initialization scale
$\alpha_\ell$, along the linearized trajectories with the error bounded as the smaller scale $\mathcal{O}(\alpha_\ell^2)$.
For each singular triple, $\mathbf A_\ell v_i=s_i u_i$ and
$\mathbf A_\ell^\top u_i=s_i v_i$, write
$c_{ji}^{\pm}=\tfrac12(\langle\mathbf w_{\ell,j}(0),v_i\rangle
                    \pm\langle\boldsymbol\omega_{\ell+1,j}(0),u_i\rangle)$.
Then
\begin{align}
 \langle\mathbf w_{\ell,j}(t),v_i\rangle
 &=c_{ji}^{+}e^{\sigma'(0)s_i t/\sqrt{p_\ell}}
   +c_{ji}^{-}e^{-\sigma'(0)s_i t/\sqrt{p_\ell}}
   +\epsilon_{ji}^{w}(t),\nonumber\\
 \langle\boldsymbol\omega_{\ell+1,j}(t),u_i\rangle
 &=c_{ji}^{+}e^{\sigma'(0)s_i t/\sqrt{p_\ell}}
   -c_{ji}^{-}e^{-\sigma'(0)s_i t/\sqrt{p_\ell}}
   +\epsilon_{ji}^{\omega}(t),
 \label{eq:scaled-separate-modal-amplification}
\end{align}
with both errors satisfying the approximation-error bounds in
\eqref{eq:scaled-theorem-bound}.
\end{formaltheorem}
\end{samepage}

\paragraph{Remark (the decaying component).}
The decaying component comes from the difference between the initial row
and column projections. For the linearized trajectories, let
$a(t)=\langle\mathbf w_{\ell,j}^{\rm lin}(t),v_i\rangle$ and
$b(t)=\langle\boldsymbol\omega_{\ell+1,j}^{\rm lin}(t),u_i\rangle$.
Adding and subtracting their evolution equations gives
\[
 \frac{d}{dt}(a+b)=\frac{\sigma'(0)s_i}{\sqrt{p_\ell}}(a+b),
 \qquad
 \frac{d}{dt}(a-b)=-\frac{\sigma'(0)s_i}{\sqrt{p_\ell}}(a-b).
\]
Thus, for $s_i>0$, their sum grows while their difference decays.
The coefficients $c_{ji}^{+}$ and $c_{ji}^{-}$ are half the initial sum
and difference, respectively. The decaying term is needed to recover
both initial projections; it vanishes exactly when they are equal.

The two weights share the growing amplitude $c_{ji}^{+}$, as stated in
Theorem~\ref{thm:controlled-feedback}. Since the outgoing column is
initialized at scale $\alpha_{\ell+1}=o(\alpha_\ell)$,
$c_{ji}^{\pm}=\tfrac12\alpha_\ell
\langle\overline{\mathbf w}_{\ell,j},v_i\rangle+o(\alpha_\ell)$.
At $t=0$, their leading contributions add in the row projection and
cancel in the column projection, leaving its smaller initial value.
The column starts at scale $\alpha_{\ell+1}=o(\alpha_\ell)$ but can grow
to scale $\alpha_\ell$ during this interval. On the shorter interval
$\lambda T\le M\alpha_{\ell+1}/\alpha_\ell$, it remains at its own
initialization scale:
$\|\boldsymbol\omega_{\ell+1,j}(t)\|\le C_M\alpha_{\ell+1}$.
Both approximation errors are then bounded by
$C_M\alpha_\ell\alpha_{\ell+1}$, as shown in the proof below.

\paragraph{Remark (gradient descent).}
We use gradient flow for convenience. For full-batch gradient descent
with step size $\Delta t$, the linear mode factors after $m$ steps are
$(1\pm\Delta t\,\sigma'(0)s_i/\sqrt{p_\ell})^m$.
The same remainder estimates and discrete Gr\"onwall give the
corresponding bounds relative to these linear iterates. If
$\Delta t\,\lambda<1$, replacing the factors by their exponentials
introduces relative error $O(\Delta t\,\lambda^2T)$.
This error vanishes for fixed $\Delta t$, since
$\lambda^2T\le M\lambda\to0$.

\subsection{Proof}

\paragraph{Proof sketch.}
Lemma~\ref{lem:pair-upper-stability} bounds the prediction and the change
in upper sensitivity as the trained pair moves. Using these bounds,
Proposition~\ref{prop:pair-linearization} controls the error in replacing
the gradient updates by the fixed linear system of
Theorem~\ref{thm:controlled-feedback-formal}.
Lemma~\ref{lem:pair-linear-modes} solves this system and bounds how much
it amplifies an error introduced at an earlier time. In the final proof,
we integrate the remainder with this amplification factor and show that
the weights remain small throughout $[0,T]$.

All constants $C$ below are uniform in the initialization scales and
time; they may depend on the fixed data, architecture, normalized seeds,
derivative bounds and the constant in \eqref{eq:pair-correlation-scale}.
Constants denoted $C_M$ may additionally depend on $M$.
Write $\rho=(\|\mathbf w_{\ell,j}\|^2+
\|\boldsymbol\omega_{\ell+1,j}\|^2)^{1/2}$.
We establish the local estimates for $\rho\le r_*$, with a fixed small
radius $r_*\le1$, before showing that the trajectory stays in this ball.

\begin{lemma}[Stability of the upper sensitivity]
\label{lem:pair-upper-stability}
Under the assumptions of Theorem~\ref{thm:controlled-feedback-formal},
for sufficiently small scales and $\rho\le r_*$, uniformly over the
training inputs,
\begin{equation}
 \|\mathbf g^{\ell+1}-\mathbf g_0\|\le C\lambda\rho^2,\qquad
 \|\mathbf g^{\ell+1}\|\le C\lambda,\qquad |f|\le C\lambda.
 \label{eq:scaled-upper-control}
\end{equation}
The upper preactivation changes by at most $C\rho^2$, and
$\lambda=o(\alpha_\ell)$.
\end{lemma}
\begin{proof}
By $\sigma(0)=0$ and the bounded inputs,
\begin{equation}
 \|\mathbf h_{\ell+1}(x)-\mathbf h_{\ell+1,0}(x)\|
 \le C\|\boldsymbol\omega_{\ell+1,j}\|\,
       \|\mathbf w_{\ell,j}\|
 \le C\rho^2.
 \label{eq:scaled-pair-product-bound}
\end{equation}
Thus the upper preactivations stay in a fixed bounded neighborhood
when $\rho\le r_*$.
Since $\sigma(0)=0$ and $\sigma'$ is bounded on the relevant intervals,
$\|\mathbf z_k\|\le C\alpha_k\|\mathbf z_{k-1}\|$ for each frozen layer
$k=\ell+2,\ldots,L$. Together with
$|f|\le C\alpha_{L+1}\|\mathbf z_L\|$ and the boundedness of
$\mathbf z_{\ell+1}$, this gives
$|f|\le C\prod_{k=\ell+2}^{L+1}\alpha_k$.
Every chain-rule term for the first and second derivatives with
respect to $\mathbf h_{\ell+1}$ contains each subsequent weight matrix and
the readout at least once. The factors $\sigma'$ and $\sigma''$ are
bounded, and any additional weight factors are bounded as the scales
vanish. Therefore
\[
 |f|+\left\|\frac{\partial f}{\partial\mathbf h_{\ell+1}}\right\|
 +\left\|\frac{\partial^2f}{\partial\mathbf h_{\ell+1}^2}\right\|_2
 \le C\prod_{k=\ell+2}^{L+1}\alpha_k
 \le C\lambda,
\]
where the last inequality uses \eqref{eq:pair-correlation-scale}.
The mean-value theorem and \eqref{eq:scaled-pair-product-bound} give
the sensitivity difference in \eqref{eq:scaled-upper-control}.
The same derivative bound at $\mathbf h_{\ell+1,0}(x)$, together with the
definition of $\mathbf A_\ell$, gives
\[
 \lambda\le C\prod_{k=\ell+2}^{L+1}\alpha_k=o(\alpha_\ell),
\]
by the initialization hierarchy.
\end{proof}

This estimate controls the errors from freezing the upper sensitivity
and replacing the residual by the label in the two gradient updates.

\begin{proposition}[Linear coupling of a row and its matching column]
\label{prop:pair-linearization}
Under the assumptions of Lemma~\ref{lem:pair-upper-stability}, for
sufficiently small scales and $\rho\le r_*$,
\begin{equation}
 \left\|\left(
 \dot{\mathbf w}_{\ell,j}
  -\frac{\sigma'(0)}{\sqrt{p_\ell}}\mathbf A_\ell^\top\boldsymbol\omega_{\ell+1,j},\,
 \dot{\boldsymbol\omega}_{\ell+1,j}
  -\frac{\sigma'(0)}{\sqrt{p_\ell}}\mathbf A_\ell\mathbf w_{\ell,j}
 \right)\right\|
 \le C\lambda(\rho^2+\alpha_\ell\rho).
 \label{eq:scaled-field-error}
\end{equation}
\end{proposition}
\begin{proof}
The exact gradient updates in \eqref{eq:main-pair-gradient} are
\begin{equation}
 \dot{\mathbf w}_{\ell,j}=\frac{1}{\sqrt{p_\ell}}
       \Etrain[(y-f)\sigma'(h_{\ell,j})
       (\boldsymbol\omega_{\ell+1,j}^\top\mathbf g^{\ell+1})x],
 \qquad
 \dot{\boldsymbol\omega}_{\ell+1,j}=\frac{1}{\sqrt{p_\ell}}
       \Etrain[(y-f)\mathbf g^{\ell+1}\sigma(h_{\ell,j})].
 \label{eq:scaled-exact-pair-flow}
\end{equation}
Since $|h_{\ell,j}|\le C\rho$, Taylor's theorem gives
\[
 \sigma'(h_{\ell,j})=\sigma'(0)+O(\rho),\qquad
 \sigma(h_{\ell,j})=\sigma'(0)h_{\ell,j}+O(\rho^2).
\]
Replacing the activation and its derivative by these leading terms,
the sensitivity by $\mathbf g_0$, and the residual by $y$ gives the
linear system in Theorem~\ref{thm:controlled-feedback-formal}.
By Lemma~\ref{lem:pair-upper-stability}, the Taylor remainders contribute
at most $C\lambda\rho^2$, freezing the sensitivity contributes at most
$C\lambda\rho^3$, and replacing the residual contributes at most
$C\lambda^2\rho$. Using $\rho\le1$ and $\lambda=o(\alpha_\ell)$ gives
\eqref{eq:scaled-field-error}.
\end{proof}

The next lemma gives the modes of this linear approximation and the
growth bound used to integrate the remainder.

\begin{lemma}[Amplification of the linear singular modes]
\label{lem:pair-linear-modes}
Every solution of the linear system in
Theorem~\ref{thm:controlled-feedback-formal} satisfies, for
$0\le s\le t\le T$ with $\lambda T\le M$,
\[
 \|(\mathbf w_{\ell,j}^{\rm lin}(t),
       \boldsymbol\omega_{\ell+1,j}^{\rm lin}(t))\|
 \le C_M
 \|(\mathbf w_{\ell,j}^{\rm lin}(s),
       \boldsymbol\omega_{\ell+1,j}^{\rm lin}(s))\|.
\]
The projection identities \eqref{eq:scaled-separate-modal-amplification} hold
for these linear trajectories with both errors equal to zero.
\end{lemma}
\begin{proof}
The linear system is generated by the symmetric matrix
\[
 \frac{\sigma'(0)}{\sqrt{p_\ell}}
 \begin{pmatrix}0&\mathbf A_\ell^\top\\\mathbf A_\ell&0\end{pmatrix},
\]
whose eigenvalues lie in $[-\lambda,\lambda]$. Since
$\lambda(t-s)\le M$, the scalar evolution along each eigendirection
is bounded by a constant depending only on $M$, proving the norm bound.
Projecting the row and column equations onto a singular pair and
then taking their sum and difference gives
\begin{equation}
 \langle\mathbf w_{\ell,j}^{\rm lin}(t),v_i\rangle
 \pm\langle\boldsymbol\omega_{\ell+1,j}^{\rm lin}(t),u_i\rangle
 =2c_{ji}^{\pm}e^{\pm\sigma'(0)s_i t/\sqrt{p_\ell}}.
 \label{eq:scaled-exact-modal-amplification}
\end{equation}
Adding and subtracting these identities gives
\eqref{eq:scaled-separate-modal-amplification}. Both weights have the
same growing amplitude, and their decaying components have opposite
signs. For $s_i=0$, both projections are constant.
\end{proof}

\begin{proof}[Proof of Theorem~\ref{thm:controlled-feedback-formal}]
We combine the gradient approximation from
Proposition~\ref{prop:pair-linearization} with the growth bound in
Lemma~\ref{lem:pair-linear-modes} to control the error over $[0,T]$.
The initial norm is at most $K\alpha_\ell$ for a fixed $K>0$, so
Lemma~\ref{lem:pair-linear-modes} bounds the norm of the linear solution
by $K_M\alpha_\ell$ throughout $[0,T]$.
Let $E(t)$ be the norm of the difference between the trained pair and
its linear approximation. Integrating the remainder bound and using
the bounded linear evolution gives, while $\rho\le r_*$,
\begin{equation}
 E(t)\le C_M\lambda\int_0^t
                    [\rho(s)^2+\alpha_\ell\rho(s)]\,ds.
 \label{eq:scaled-duhamel}
\end{equation}
Suppose $\rho(s)\le2K_M\alpha_\ell$ for $0\le s\le t$.
Since $\lambda t\le M$, this gives
\[
 E(t)\le C_M\lambda t\,\alpha_\ell^2
       \le C_M\alpha_\ell^2.
\]
For sufficiently small scales,
$\rho(t)\le K_M\alpha_\ell+E(t)<\tfrac32K_M\alpha_\ell$
and $2K_M\alpha_\ell<r_*$.
The bound assumed on $\rho$ holds initially and cannot reach equality
before $T$, by this strict improvement. Thus the flow stays in the
bounded neighborhood and continues to $T$, proving
\eqref{eq:scaled-theorem-bound}.
To obtain the bounds on the shorter interval, integrate the linear
column equation and use the bound on the linear row:
\[
 \|\boldsymbol\omega_{\ell+1,j}^{\rm lin}(t)\|
 \le C\alpha_{\ell+1}+C_M\lambda t\,\alpha_\ell.
\]
If $\lambda T\le M\alpha_{\ell+1}/\alpha_\ell$, this is at most
$C_M\alpha_{\ell+1}$, while the preceding error estimate gives
$E(t)\le C_M\alpha_\ell\alpha_{\ell+1}$.
Adding the error to the linear column bound proves the stated bound
for the trained column.
Finally, Lemma~\ref{lem:pair-linear-modes} gives the projections of the
linear trajectory. Projection onto a unit singular vector cannot
increase the error, so the same bound proves
\eqref{eq:scaled-separate-modal-amplification}.
\end{proof}

\section{Proof of Theorem~\ref{thm:three-layer-closed-flow}}
\label{app:three-forward-flow}

We start by defining the continuous-time
variant of Algorithm~\ref{alg:combined} studied in
Theorem~\ref{thm:three-layer-closed-flow}. At each iteration, the row
correction, column correction and readout gradient update are each
multiplied by the same step size $dt$. We then take $dt\to0$.
A common positive learning-rate multiplier can be absorbed into time, whereas
unequal multipliers change the relative dynamics.
Algorithm~\ref{alg:three-continuous} specifies these updates for a
three-layer network.

\begin{algorithm}[H]
\caption{Three-layer BackLoFi: updates for the continuous-time limit}
\label{alg:three-continuous}
\begin{algorithmic}[1]
\setlength{\itemsep}{3pt}
\Require Forward Neural LoFi features $\mathbf h_1^\star,\mathbf h_2^\star$
 and readout $w_3$; step size $dt$; duration $T$.
\State Set $C=0$ and keep both Gaussian lifts fixed.
\For{each time step of size $dt$, up to time $T$}
 \State \textbf{Select the coupled features.}
 \Statex\hspace{\algorithmicindent}
  Compute $g^2=w_3\sigma'(h_2)$ and the leading singular pair
  $(u,\mathbf v)$ of $\Etrain[y\,g^2\otimes\mathbf z_0]$.
 \Statex\hspace{\algorithmicindent}
  Set $\psi=\langle u,g^2\rangle$, $\chi=\langle\mathbf v,\mathbf z_0\rangle$;
  hold these scores fixed for this step.
 \State \textbf{Correct the first layer.}
  $h_{1,a}^\star\gets h_{1,a}^\star+dt\,\chi$
  for every $a$.
 \State \textbf{Refresh forward LoFi.}
  Recompute $z_1,K_1$ and set $\mathbf h_2^\star\gets\Fwd_2(K_1,y)$.
 \State \textbf{Add the column spike.} Using the updated $z_1$, set
 \Statex\hspace{\algorithmicindent}
  $q\gets\Etrain[y\psi z_1]$,\qquad
  $C\gets C+dt\,u\otimes q$.
 \State \textbf{Train the readout.} Recompute $h_2,z_2,h_3$, then set
 \Statex\hspace{\algorithmicindent}
  $w_3\gets w_3+dt\,\Etrain[(y-h_3)z_2]$.
\EndFor
\State \Return $h_3=\langle w_3,z_2\rangle$.
\end{algorithmic}
\end{algorithm}

\subsection{Full description of the Theorem and the underlying order parameters}

We now give the full description of Theorem~\ref{thm:three-layer-closed-flow}
and the order parameters that allow the equations to close. We study the
continuous-time limit of Algorithm~\ref{alg:three-continuous}, followed
by the width limits $p_1\to\infty$ and $p_2\to\infty$.

\paragraph{Filtered features and accumulated coupling.}
The input map $\mathbf z_0$ is fixed, with kernel
$K_0(x,x')=\langle\mathbf z_0(x),\mathbf z_0(x')\rangle$.
The filtered features $\mathbf h_{1,t}^\star(x)\in\mathbb R^{k_1}$ and
$\mathbf h_{2,t}^\star(x)\in\mathbb R^{k_2}$ describe the two hidden layers
before their Gaussian lifts. The first starts from
$\mathbf h_{1,0}^\star=\Fwd_1(K_0,y)$ and follows the lower correction;
the second is refreshed as $\mathbf h_{2,t}^\star=\Fwd_2(K_{1,t},y)$.

Let $\mathbf r\sim\nu=\mathcal N(0,I_{k_1})$ and
$\eta\sim\mu=\mathcal N(0,I_{k_2})$ be the fixed lift rows.
The scalar function $C_t(\eta,\mathbf r)$ records the total column
coupling between these rows, starting from $C_0=0$. Its action is
\begin{equation}
 (C_tf)(\eta)=\mathbb E_{\mathbf r}
                    [C_t(\eta,\mathbf r)f(\mathbf r)].
 \label{eq:three-gaussian-column-kernel}
\end{equation}
The hidden preactivations and activations are therefore
\begin{align}
 h_{1,t}(\mathbf r,x)&=\mathbf r^\top\mathbf h_{1,t}^\star(x),
 &z_{1,t}&=\sigma(h_{1,t}),\nonumber\\
 h_{2,t}(\eta,x)&=\eta^\top\mathbf h_{2,t}^\star(x)
       +(C_tz_{1,t}(\cdot,x))(\eta),
 &z_{2,t}&=\sigma(h_{2,t}).
 \label{eq:section-three-upper-response}
\end{align}
Thus earlier columns act on the current lower features through $C_tz_{1,t}$.
The arguments of $C_t$ have dimensions $k_2$ and $k_1$, independent of
network width. Neuron inner products are Gaussian expectations, or empirical
averages at finite width; input coefficient inner products are Euclidean.
This normalization absorbs the factors $1/\sqrt{p_j}$
used for the feature vectors in the main text.

\paragraph{Output and sensitivity.}
The scalar output is $h_{3,t}=\langle w_{3,t},z_{2,t}\rangle$.
To track how readout learning changes the sensitivity, retain the response
$S_t(\eta,s)=w_{3,t}(\eta)\sigma'(s)$. Evaluating it at the current
preactivation gives
\begin{equation}
 g_t^2(\eta,x)=S_t(\eta,h_{2,t}(\eta,x)).
 \label{eq:section-three-response}
\end{equation}
The output $h_{3,t}$ alone does not determine this sensitivity when the
hidden features change; $S_t$ retains the readout's response to that change.
Together, $\mathbf h_{1,t}^\star,\mathbf h_{2,t}^\star,C_t,S_t,h_{3,t}$
determine the feature and sensitivity kernels:
\begin{align}
 K_{1,t}(x,x')&=\mathbb E_{\mathbf r}
       [z_{1,t}(\mathbf r,x)z_{1,t}(\mathbf r,x')],\nonumber\\
 K_{2,t}(x,x')&=\mathbb E_\eta[z_{2,t}(\eta,x)z_{2,t}(\eta,x')],\qquad
 G_t(x,x')=\mathbb E_\eta[g_t^2(\eta,x)g_t^2(\eta,x')].
 \label{eq:section-three-kernels}
\end{align}

\paragraph{Selection and continuous updates.}
At each time, $(\psi_t,\chi_t)$ maximizes
\eqref{eq:variational-feature-pair} for the current kernels $G_t,K_0$.
The singular-vector equations become
\begin{equation}
 \rho_t\psi_t(x)=\Etrain[G_t(x,X)y\chi_t(X)],\qquad
 \rho_t\chi_t(x)=\Etrain[K_0(x,X)y\psi_t(X)],
 \label{eq:main-three-selection}
\end{equation}
where both features have unit RKHS norm and $\rho_t>0$ is the achieved
correlation. Algorithm~\ref{alg:three-continuous} gives
\begin{align}
 \dot h_{1,t,a}^{\star}(x)
 &=\chi_t(x),\qquad a=1,\ldots,k_1,
 \label{eq:section-three-feature-flow}\\
 \partial_t C_t(\eta,\mathbf r)
 &=\frac{1}{\rho_t}
       \Etrain[y\chi_t(X)g_t^2(\eta,X)]
       \Etrain[y\psi_t(X)z_{1,t}(\mathbf r,X)],
 \label{eq:section-three-column-flow}\\
 \partial_t S_t(\eta,s)
 &=\sigma'(s)\Etrain[(y-h_{3,t}(X))z_{2,t}(\eta,X)].
 \label{eq:section-three-response-flow}
\end{align}
The first equation corrects the lower filter. The second accumulates its
paired column spikes, using
$u_t=\rho_t^{-1}\Etrain[y\chi_tg_t^2]$ and
$q_t=\Etrain[y\psi_tz_{1,t}]$.
The third is the final-layer gradient-flow update, with $s$ held fixed.
The output evolves by \eqref{eq:main-three-informal-flow}.

\paragraph{Forward filter.}
The forward map $\Fwd_\ell(K,y)$ is Algorithm~\ref{alg:neural-lofi}
expressed in feature space. Let $\mathcal H_K$ be the RKHS of the current
kernel $K$. Following \citet{dandi_neural_lofi}, its features are selected by
\begin{align}
 f_0&\in\underset{\|f\|_{\mathcal H_K}=1}{\arg\max}
                  \;\Etrain[yf],\nonumber\\
 f_a&\in\underset{\substack{\|f\|_{\mathcal H_K}=1\\
              \langle f,f_b\rangle_{\mathcal H_K}=0,\;1\le b<a}}
              {\arg\max}\left|\Etrain[yf^2]\right|,
              \qquad a=1,\ldots,k_\ell-1.
 \label{eq:three-refresh-filter}
\end{align}
The first coordinate maximizes first-order correlation; the remaining
coordinates successively maximize second-order correlation. Orthogonality
is measured in $\mathcal H_K$ among the second-order features.
The map returns $\Fwd_\ell(K,y)=(f_0,f_1,\ldots,f_{k_\ell-1})$.
At initialization, choose the sign of each second-order forward feature
and of $\chi$ to have positive empirical inner product with a fixed
reference function. At later steps, use its previous value as the
reference, and change the sign of $\psi$ together with that of $\chi$.

\begin{assumption}[Data, lifts and initialization]
\label{ass:three-setup}
The training set is finite, and the numbers of filtered coordinates
$k_1,k_2$ are fixed. The lift laws are
$\nu=\mathcal N(0,I_{k_1})$ and $\mu=\mathcal N(0,I_{k_2})$, or their
finite empirical versions. All lift rows are held fixed in time.
Initially $C_0=0$. Choose $h_{3,0}$ in the span of the initial features
$z_{2,0}$ on the training inputs, and set $S_0$ by
\eqref{eq:three-response-initialization}.
\end{assumption}

\begin{assumption}[Activation]
\label{ass:three-activation}
The activation is $C^2$, and $\sigma,\sigma',\sigma''$ are bounded.
\end{assumption}

\begin{assumption}[Spectral gaps]
\label{ass:three-regularity}
On $[0,T]$, the label-weighted first moments used by the forward
filters have norms bounded away from zero. The retained eigenvalues
are uniformly separated in absolute value from zero, from one another,
and from the discarded spectrum. The leading coupled singular value
stays uniformly separated from zero and from the remaining singular values.
\end{assumption}

\begin{assumption}[Sequential width limit]
\label{ass:three-width}
The two lifts are sampled from independent nested Gaussian sequences.
The data, ranks and initial prediction vector are the same across widths.
The same reference functions initialize the signs at every width,
with nonzero inner product with the limiting features.
The initial upper Gaussian-population kernel matrix is
positive definite. The population trajectory satisfies
Assumption~\ref{ass:three-regularity} on $[0,T]$.
\end{assumption}

\begin{formaltheorem}[Precise version of Theorem~\ref{thm:three-layer-closed-flow}]
\label{thm:three-layer-closed-flow-formal}
Under Assumptions~\ref{ass:three-setup}--\ref{ass:three-regularity}, the
current-state equations have a unique local solution. On $[0,T]$,
the piecewise-linear interpolants of
Algorithm~\ref{alg:three-continuous} converge to it uniformly as
$dt\to0$, with error $O_T(dt)$.
Under Assumption~\ref{ass:three-width}, the continuous flows also converge
almost surely, first as $p_1\to\infty$ at fixed admissible $p_2$, then as
$p_2\to\infty$, to the solution of
\eqref{eq:main-coupled-filter-flow}--\eqref{eq:main-three-informal-flow},
with expectations taken over both Gaussian lifts.
Convergence is uniform on $[0,T]$ for $C_t,S_t$ in the fixed norms
specified in \eqref{eq:three-proof-state-space}, and on the data and
$[0,T]$ for the filtered features,
output, selected scores and current kernels $K_{1,t},K_{2,t},G_t$.
\end{formaltheorem}
Here an admissible fixed upper sample means that the mixed
Gaussian-lower/empirical-upper flow realizes the prescribed initial
predictions, has a stable initial kernel inverse, and satisfies
Assumption~\ref{ass:three-regularity} on $[0,T]$. The width statement concerns
sufficiently large upper widths, for which these conditions hold.

\subsection{Proof of Theorem \ref{thm:three-layer-closed-flow-formal}}

\paragraph{Proof sketch.}
To obtain a closed evolution and compare networks of different widths,
we represent the accumulated columns by a kernel $C_t(\eta,\mathbf r)$
on the fixed Gaussian label spaces. Each selected feature pair contributes
a rank-one increment, and $C_tz_{1,t}$ records the accumulated action
on the current lower features.
Differentiating $C_tz_{1,t}$ therefore retains both newly added columns
and the action of earlier columns on moving features. Together with
the readout update, this gives the output equation and the evolving
sensitivity. The spectral gaps ensure that the selected features vary
continuously with the state, yielding a locally Lipschitz flow.
Each iteration of Algorithm~\ref{alg:three-continuous} approximates
\eqref{eq:section-three-feature-flow}--\eqref{eq:section-three-response-flow}
over an interval of length $dt$ with error $O(dt^2)$, giving an
$O_T(dt)$ error on $[0,T]$.
Finally, we compare empirical and Gaussian lifts on the same probability
spaces. Stability of the flow lets us average the first lift and then
the second, while retaining the dependence of learned features on the
sampled neurons.

We start by establishing the continuity of the maximizing feature pair
with respect to the evolving order parameters.

\begin{lemma}[Continuity of the maximizing feature pair]
\label{lem:three-mode-continuity}
Fix the training data, input feature map $\mathbf z_0$ and lift laws
$\mu,\nu$. Under Assumptions~\ref{ass:three-activation}
and~\ref{ass:three-regularity}, the selected singular pair $(u,\mathbf v)$
and scores $(\psi,\chi)$ depend locally Lipschitz-continuously on the current
order parameters $(\mathbf h_1^\star,C,S)$, with
$\mathbf h_2^\star=\Fwd_2(K_1,y)$ and signs chosen as above.
Use the empirical norm $\|f\|_n^2=\Etrain[\|f(X)\|^2]$ for the
filtered features and scores, the Hilbert--Schmidt norm for $C$, and
the $L^2(\mu;\mathcal D)$ norm for $S$, where
$\mathcal D=\operatorname{span}\{\sigma'\}$ and
$\|b\sigma'\|_{\mathcal D}=|b|$.
The statement holds on bounded sets with
$\|S(\eta,\cdot)\|_{\mathcal D}\le M$ almost everywhere, for fixed $M<\infty$.
\end{lemma}
\begin{proof}
The first-layer features
$z_1(\mathbf r,x)=\sigma(\mathbf r^\top\mathbf h_1^\star(x))$
depend Lipschitz-continuously on $\mathbf h_1^\star$: the derivative of
$\sigma$ is bounded and $\mathbf r$ has finite second moment.
Their label-weighted first and second moments inherit this local regularity.
Assumption~\ref{ass:three-regularity} keeps the first moment nonzero
and separates the retained eigenvalues of the second-moment operator.
Normalizing the first
moment and applying eigenvector perturbation bounds therefore give local
Lipschitz continuity of the forward features
$\mathbf h_2^\star=\Fwd_2(K_1,y)$.

For two sets of order parameters $(\mathbf h_1^\star,C,S)$ and
$(\bar{\mathbf h}_1^\star,\bar C,\bar S)$, let $h_2,\bar h_2$ be the
corresponding upper preactivations. Equation~\eqref{eq:section-three-upper-response}
and the bound $\|Cf\|_2\le\|C\|_{\rm HS}\|f\|_2$ give
\[
 \|h_2-\bar h_2\|_2
 \le L\bigl(\|\mathbf h_1^\star-\bar{\mathbf h}_1^\star\|_n
                 +\|C-\bar C\|_{\rm HS}\bigr).
\]
The $L^2$ norms on preactivations and sensitivities average over the
upper label and training inputs. Since $g^2(\eta,x)=S(\eta,h_2(\eta,x))$,
the corresponding sensitivities satisfy
\[
 \|g^2-\bar g^2\|_2
 \le\|\sigma'\|_\infty\|S-\bar S\|_{L^2(\mu;\mathcal D)}
       +M\|\sigma''\|_\infty\|h_2-\bar h_2\|_2.
\]
These estimates give local Lipschitz continuity of $g^2$ in
$(\mathbf h_1^\star,C,S)$. Since
$A=\Etrain[y\,g^2\otimes\mathbf z_0]$ is bounded and linear in $g^2$,
$A^*A$ also depends locally Lipschitz-continuously on these order parameters.
Its leading eigenvalue $\rho^2$ is positive and simple. Fixing the sign
of its unit eigenvector $\mathbf v$ by positive overlap with a reference,
the equations
\[
 A^*A\mathbf v=\rho^2\mathbf v,\qquad \|\mathbf v\|_2=1
\]
have an invertible linearization. On $\mathbf v^\perp$, the inverse of
$\rho^2I-A^*A$ has norm at most $1/(\rho^2-\rho_2^2)$, where $\rho_2$
is the second singular value of $A$ (zero when $A$ has rank one).
Assumption~\ref{ass:three-regularity} keeps this eigenvalue difference
bounded away from zero.
The implicit function theorem gives smooth dependence of $\rho$ and
$\mathbf v$ on $A$, and hence local Lipschitz dependence on the order
parameters. The identities
$u=A\mathbf v/\rho$, $\psi=\langle u,g^2\rangle$ and
$\chi=\langle\mathbf v,\mathbf z_0\rangle$ give the same regularity for
the left singular vector and both maximizing functions.
\end{proof}

We next verify that the column updates accumulate in $C_t$ as defined in
\eqref{eq:three-gaussian-column-kernel}, and derive the evolution of their
action $C_tz_{1,t}$ on the changing lower features.

\begin{proposition}[Accumulated column coupling]
\label{prop:three-accumulated-coupling}
Let $\mu,\nu$ be the Gaussian lift laws or their empirical versions.
Suppose $u_t\in L^2(\mu)$ and $q_t\in L^2(\nu)$ are continuous in $t$.
The column update $\dot C_t=u_t\otimes q_t$, $C_0=0$,
has a unique continuously differentiable Hilbert--Schmidt solution,
whose scalar kernel is given in $L^2(\mu\otimes\nu)$ by
\begin{equation}
 C_t(\eta,\mathbf r)=\int_0^t
                u_\tau(\eta)q_\tau(\mathbf r)\,d\tau.
 \label{eq:three-proof-column-history}
\end{equation}
It maps functions of the lower label $\mathbf r$ to functions of the
upper label $\eta$ through \eqref{eq:three-gaussian-column-kernel}.
For any continuously differentiable lower-feature curve $z_{1,t}$,
\[
 \partial_t(C_tz_{1,t})=\dot C_tz_{1,t}+C_t\dot z_{1,t}.
\]
\end{proposition}
\begin{proof}
A rank-one update has kernel $u_t(\eta)q_t(\mathbf r)$ and
Hilbert--Schmidt norm $\|u_t\|_{L^2(\mu)}\|q_t\|_{L^2(\nu)}$.
Integrating in this Hilbert space gives the unique solution and the bound
\[
 \|C_t\|_{\rm HS}
 \le\int_0^t
       \|u_\tau\|_{L^2(\mu)}\|q_\tau\|_{L^2(\nu)}\,d\tau.
\]
The kernel therefore belongs to $L^2(\mu\otimes\nu)$.
Since applying a Hilbert--Schmidt operator to an $L^2(\nu)$ function is
continuous and bilinear, the product rule gives the displayed derivative.
\end{proof}

Combining this representation with the readout update gives the closed
response and output equations.

\begin{proposition}[Closed evolution of the response and output]
\label{prop:three-response-closure}
Under Assumptions~\ref{ass:three-setup}--\ref{ass:three-regularity}, let
$(\mathbf h_{1,t}^\star,C_t,w_{3,t})$ solve the lower-feature and
column equations \eqref{eq:section-three-feature-flow}--\eqref{eq:section-three-column-flow}
and final-layer gradient flow under squared loss, with the prescribed
initialization. Then $S_t(\eta,s)=w_{3,t}(\eta)\sigma'(s)$ satisfies
\eqref{eq:section-three-response-flow}, and
$h_{3,t}=\langle w_{3,t},z_{2,t}\rangle$ satisfies
\eqref{eq:main-three-informal-flow}.
These derivatives, and those of the feature maps and their kernels,
are determined by the current values of
$(\mathbf h_{1,t}^\star,\mathbf h_{2,t}^\star,C_t,S_t,h_{3,t})$.
\end{proposition}
\begin{proof}
The update rates are continuous in the order parameters by
Lemma~\ref{lem:three-mode-continuity}. Thus any solution of their integral
equations is continuously differentiable in time. We derive the response
and output identities for such a solution; its existence and uniqueness
are established in the proof of
Theorem~\ref{thm:three-layer-closed-flow-formal} below.

Write the readout as $w_{3,t}\in L^2(\mu)$.
The cross-layer operator is $A_t=\Etrain[y\,g_t^2\otimes\mathbf z_0]$.
For a nonzero singular pair, $A_tv_t=\rho_tu_t$ gives
$u_t=\rho_t^{-1}\Etrain[y\chi_tg_t^2]$; applying the upper feature map
gives $\rho_t\psi_t(x)=\Etrain[G_t(x,X)y\chi_t(X)]$.
The adjoint equation similarly gives
$\rho_t\chi_t(x)=\Etrain[K_0(x,X)y\psi_t(X)]$.
Each singular vector lies in the span of its feature map and therefore
has the minimum coefficient norm among representations of the selected
feature, which equals the RKHS norm of that feature. In particular,
the unit vectors $u_t,v_t$ give
$\|\psi_t\|_{\mathcal H_{G_t}}=\|\chi_t\|_{\mathcal H_{K_0}}=1$.

Every lower filtered coordinate follows
\eqref{eq:section-three-feature-flow}. The remaining continuous updates are
\begin{align}
 \dot C_t&=u_t\otimes q_t,\qquad
 \dot w_{3,t}=\Etrain[(y-h_{3,t}(X))z_{2,t}(\cdot,X)],\nonumber\\
 q_t&=\Etrain[y\psi_t(X)z_{1,t}(\cdot,X)],\qquad
 u_t=\rho_t^{-1}\Etrain[y\chi_t(X)g_t^2(\cdot,X)].
 \label{eq:three-proof-parameters}
\end{align}
Here $(u\otimes q)b=u\langle q,b\rangle$, and
$h_{2,t}(\eta,x)=\eta^\top\mathbf h_{2,t}^\star(x)
+(C_tz_{1,t}(x))(\eta)$. Substituting the two representers into
$\dot C_t$ gives \eqref{eq:section-three-column-flow}.
The loading $q_t$ fits the modified label on current lower activations,
and $u_t$ selects the upper response.

To initialize the response from $h_{3,0}$, choose coefficients $b_0$
on the training inputs satisfying
\[
 \Etrain[K_{2,0}(x,X)b_0(X)]=h_{3,0}(x)
 \qquad\text{for every training input }x.
\]
A solution exists by Assumption~\ref{ass:three-setup}. The readout
$w_{3,0}(\eta)=\Etrain[b_0(X)z_{2,0}(\eta,X)]$ realizes $h_{3,0}$
and lies in the span of the initial activation profiles.
Any other readout with the same predictions differs by a function
orthogonal to that span. Hence $w_{3,0}$ has minimum norm and is
independent of the chosen coefficient solution.

The corresponding initial response is
\begin{equation}
 S_0(\eta,s)=\sigma'(s)\Etrain[b_0(X)z_{2,0}(\eta,X)].
 \label{eq:three-response-initialization}
\end{equation}
Holding $s$ fixed and applying the readout gradient update gives
\eqref{eq:section-three-response-flow}; evaluating at $h_{2,t}$ gives
the current sensitivity in \eqref{eq:section-three-response}.
Its total derivative therefore contains both readout learning and
the changing activation derivative:
\begin{align}
 \partial_t g_t^2(\eta,x)
 &=\sigma'(h_{2,t}(\eta,x))
       \Etrain[(y-h_{3,t}(X))z_{2,t}(\eta,X)]\nonumber\\
 &\quad+\partial_sS_t(\eta,h_{2,t}(\eta,x))\,
                          \dot h_{2,t}(\eta,x).
 \label{eq:three-proof-sensitivity-flow}
\end{align}
The first term changes the response even with the hidden features
held fixed. The second changes it as the upper preactivation moves, with
$\partial_sS_t(\eta,s)=w_{3,t}(\eta)\sigma''(s)$.
Both update $G_t$, hence the next selected scores $(\psi_t,\chi_t)$,
which in turn update the lower filter and the column coupling.
Since $\partial_tS_t$ is $\sigma'(s)$ times a function of $\eta$, the
evolution preserves the product form $S_t(\eta,s)=w_{3,t}(\eta)\sigma'(s)$
of the initialization.

Proposition~\ref{prop:three-accumulated-coupling} represents all installed
columns by the current kernel $C_t$.
The accumulated columns act on the current lower features through
$C_tz_{1,t}$. Differentiating this product gives
\begin{align}
 \dot h_{2,t}(\eta,x)
 &=\eta^\top\dot{\mathbf h}_{2,t}^\star(x)
       +(\dot C_tz_{1,t}(x))(\eta)+(C_t\dot z_{1,t}(x))(\eta),\nonumber\\
 \dot z_{1,t}(\mathbf r,x)
 &=\sigma'(\mathbf r^\top\mathbf h_{1,t}^\star(x))
       \chi_t(x)\sum_{a=1}^{k_1}r_a.
 \label{eq:three-proof-feature-velocity}
\end{align}
The three terms in the first line refresh the forward filter, install
the current column, and propagate lower-feature motion through
previously installed columns.
The second line follows because every filtered coordinate has the
same velocity: $\mathbf r^\top\dot{\mathbf h}_{1,t}^\star
=\chi_t\sum_a r_a$.
The change in the first-layer kernel then refreshes the second-layer
forward filter:
\begin{align*}
 \dot K_{1,t}(x,x')
 &=\mathbb E_{\mathbf r}\big[
   \dot z_{1,t}(\mathbf r,x)z_{1,t}(\mathbf r,x')
   +z_{1,t}(\mathbf r,x)\dot z_{1,t}(\mathbf r,x')\big],\\
 \dot{\mathbf h}_{2,t}^{\star}
 &=D_K\Fwd_2(K_{1,t},y)[\dot K_{1,t}],
\end{align*}
where $D_K\Fwd_2$ is the derivative of the forward filter with
respect to its kernel. Thus all three derivatives are determined by
the current filtered features and their prescribed velocity.

Differentiate $h_{3,t}=\langle w_{3,t},z_{2,t}\rangle$ to obtain
\begin{align}
 \dot h_{3,t}(x)&=\Etrain[K_{2,t}(x,X)(y-h_{3,t}(X))]
       +\mathbb E_\eta[g_t^2(\eta,x)\dot h_{2,t}(\eta,x)],\nonumber\\
 \mathbb E_\eta[g_t^2(\eta,x)(\dot C_tz_{1,t}(x))(\eta)]
       &=\psi_t(x)\Etrain[K_{1,t}(x,X)y\psi_t(X)].
 \label{eq:three-proof-chain-rule}
\end{align}
The second identity uses $\langle g_t^2,u_t\rangle=\psi_t$ and
$\langle q_t,z_{1,t}(x)\rangle=\Etrain[K_{1,t}(x,X)y\psi_t(X)]$.
Substituting \eqref{eq:three-proof-feature-velocity} gives the four
terms of \eqref{eq:main-three-informal-flow}, including
$\mathbb E_\eta[g_t^2(C_t\dot z_{1,t})]$ for earlier columns.
These derivatives also hold for Gaussian lifts. The readout equation
keeps $w_{3,t}$ uniformly bounded over neurons on finite intervals;
the formulas for $u_t,q_t$ and $C_t$
then give $|\dot h_{1,t}(\mathbf r,x)|\le C_T\|\mathbf r\|$ and
$|\dot h_{2,t}(\eta,x)|\le C_T(1+\|\eta\|)$.
Bounded activation derivatives and finite Gaussian second moments
therefore justify differentiating the neuron expectations by dominated
convergence.
Conversely, the output equation and the response update give identical
derivatives for $h_{3,t}$ and $\langle w_{3,t},z_{2,t}\rangle$.
Their equal initialization therefore preserves equality.
Since $g_t^2=S_t(\eta,h_{2,t})$ and
$\mathbf h_{2,t}^\star=\Fwd_2(K_{1,t},y)$, the variational pair and the
representers in \eqref{eq:three-proof-parameters} are determined by the
current state, as are the feature and kernel derivatives.
\end{proof}

\begin{proof}[Proof of Theorem~\ref{thm:three-layer-closed-flow-formal}]
We now combine Lemma~\ref{lem:three-mode-continuity} and
Propositions~\ref{prop:three-accumulated-coupling}
and~\ref{prop:three-response-closure} to prove
Theorem~\ref{thm:three-layer-closed-flow-formal}.
Proposition~\ref{prop:three-accumulated-coupling} gives the evolution of
the accumulated columns and their action on the changing lower features;
Proposition~\ref{prop:three-response-closure} gives the closed response
and output equations. Lemma~\ref{lem:three-mode-continuity} supplies the
local Lipschitz bounds used to establish existence and uniqueness,
control the time-discretization error, and pass to the two successive
width limits.

The filtered features give the Gaussian projections
$\mathbf r^\top\mathbf h_1^\star(x)$ and
$\eta^\top\mathbf h_2^\star(x)$. The coupling $C$ maps the lower
activation $z_1(\cdot,x)\in L^2(\nu)$ to the upper-preactivation
correction $Cz_1(\cdot,x)\in L^2(\mu)$, while $S(\eta,\cdot)$ gives
the sensitivity evaluated at the corrected upper preactivation.
Their convergence is therefore naturally formulated on the Hilbert
spaces of functions of these Gaussian variables: we measure $C$ in
Hilbert--Schmidt norm and $S$ in $L^2(\mu;\mathcal D)$, with the response
norm defined in Lemma~\ref{lem:three-mode-continuity}.
Together with the filtered features on the training inputs, this gives
the state space
\begin{equation}
 \mathbb R^{nk_1}
 \times\operatorname{HS}(L^2(\nu),L^2(\mu))
 \times L^2(\mu;\mathcal D)
 \label{eq:three-proof-state-space}
\end{equation}
for $(\mathbf h_1^\star,C,S)$. Here $\operatorname{HS}$ denotes
Hilbert--Schmidt operators: the squared norm of $C$ is the integral
of its squared kernel over both neuron labels.
Since $S(\eta,s)=w_3(\eta)\sigma'(s)$, the chosen response norm satisfies
\[
 \|S\|_{L^2(\mu;\mathcal D)}^2
 =\mathbb E_\eta[|w_3(\eta)|^2].
\]
The same identity holds for differences of two responses.
We can therefore prove existence, uniqueness and convergence for
the coefficient function $w_3$, then recover $S$ by multiplying by
the fixed function $\sigma'$.
Since $\|\chi_t\|_{\mathcal H_{K_0}}=1$, the reproducing property gives
$|\chi_t(x)|\le\sqrt{K_0(x,x)}$. Hence
\[
 \|\dot{\mathbf h}_{1,t}^\star\|_n
 =\sqrt{k_1}\,\|\chi_t\|_n
 \le\sqrt{k_1\Etrain[K_0(X,X)]}.
\]
The fixed input kernel and filter rank make this bound uniform in width.

The canonical initial readout is a finite linear combination of bounded
activations, and its update satisfies
\[
 \|\dot w_{3,t}\|_\infty
 \le\|\sigma\|_\infty
       (\|y\|_n+\|\sigma\|_\infty\|w_{3,t}\|_2).
\]
Lemma~\ref{lem:three-mode-continuity} gives local Lipschitz continuity
of the right-hand sides of
\eqref{eq:section-three-feature-flow}--\eqref{eq:section-three-response-flow}
in $(\mathbf h_1^\star,C,S)$, while the response coefficients are bounded.
Integrating these right-hand sides from the prescribed initial state
gives a contraction on bounded trajectories near the initial state over a sufficiently short
interval: the displayed readout bound keeps the coefficients bounded,
and the Lipschitz constant is multiplied by the interval length.
This proves local existence and uniqueness.
Gr\"onwall bounds $\|w_{3,t}\|_2$ and hence $\|w_{3,t}\|_\infty$ on
finite intervals, so the solution extends while
Assumption~\ref{ass:three-regularity} holds.
Proposition~\ref{prop:three-response-closure} then identifies its output
with \eqref{eq:main-three-informal-flow}.

We now show that Algorithm~\ref{alg:three-continuous} converges to these
equations as $dt\to0$, with the lifts fixed. Denote the changes over one
iteration by $\Delta\mathbf h_1^\star,\Delta C,\Delta S$.
The lower-feature update is exactly $dt$ times
\eqref{eq:section-three-feature-flow}. The column and readout updates use
the features obtained after the earlier updates in that iteration.
Those features differ from their values at the start of the iteration
by $O(dt)$. The Lipschitz bounds therefore give
\begin{align*}
 \Delta\mathbf h_1^\star&=dt\,\dot{\mathbf h}_{1,t}^\star,\\
 \Delta C&=dt\,\partial_tC_t+O_T(dt^2),\\
 \Delta S&=dt\,\partial_tS_t+O_T(dt^2),
\end{align*}
where the derivatives are the right-hand sides of
\eqref{eq:section-three-feature-flow}--\eqref{eq:section-three-response-flow},
evaluated at the start of the iteration. The errors are measured in
\eqref{eq:three-proof-state-space}. Integrating the continuous equations
over the same interval gives these first-order increments with
$O_T(dt^2)$ remainders.

Let $E_k$ be the error between the algorithm and the continuous solution
at time $k\,dt$, measured in \eqref{eq:three-proof-state-space}.
While the iterates remain in a neighborhood of the continuous trajectory
on which Assumption~\ref{ass:three-regularity} holds uniformly, the error satisfies
\[
 E_{k+1}\le(1+L_Tdt)E_k+C_Tdt^2,\qquad E_0=0.
\]
Summing this recurrence gives $\max_{k\,dt\le T}E_k=O_T(dt)$.
For sufficiently small $dt$, this bound keeps every iterate in that
neighborhood up to time $T$. Linear interpolation gives the same uniform
error between steps.
The readout bound above holds for the discrete updates as well.
Continuity of the forward filter and Lemma~\ref{lem:three-mode-continuity}
transfer this convergence to the features, output, kernels and selected scores.

We next take the width limits.
\label{app:three-width-averaging}
The learned features depend on the sampled lifts, so we compare the
finite-width and Gaussian equations along their evolving trajectories.
For the lower lift, write the sampled rows as
$\mathbf r_1,\ldots,\mathbf r_{p_1}$. Partition the Gaussian label space
into measurable cells $V_1,\ldots,V_{p_1}$, each of probability $1/p_1$,
chosen to minimize
\[
 \sum_{i=1}^{p_1}\int_{V_i}\|\mathbf r-\mathbf r_i\|^2\,d\nu(\mathbf r).
\]
Set $T_{1,p_1}(\mathbf r)=\mathbf r_i$ for $\mathbf r\in V_i$.
This is the quadratic optimal transport map from the Gaussian law to
the empirical distribution of the sampled rows. The Gaussian density
and finite second moments ensure that such a minimizing partition exists
\citep[Theorem~2.12]{villani2003topics}.
Construct $T_{2,p_2}$ in the same way for the upper lift.

To see why the transport cost tends to zero, divide a bounded box $B$
into finitely many cells of diameter at most $\varepsilon$.
In each cell, the fraction of sampled rows converges almost surely
to its Gaussian probability by the strong law. Assigning as much Gaussian
mass as possible to sampled rows within each cell leaves a vanishing
unassigned mass inside $B$, with squared distance at most
$\varepsilon^2$ on the assigned mass. For the remaining assignments, use
$\|\mathbf r-\mathbf r_i\|^2\le2\|\mathbf r\|^2+2\|\mathbf r_i\|^2$.
Inside $B$, squared norms are bounded and the unassigned mass vanishes.
Outside $B$, the sampled second moment converges
to $\int_{B^c}\|\mathbf r\|^2\,d\nu$ by the strong law. Since the chosen
partition minimizes the cost,
\[
 \limsup_{p_1\to\infty}
 \|T_{1,p_1}(\mathbf r)-\mathbf r\|_{L^2(\nu)}^2
 \le\varepsilon^2+4\int_{B^c}\|\mathbf r\|^2\,d\nu(\mathbf r)
 \quad\text{almost surely}.
\]
Letting the cell diameter tend to zero and the box grow proves
convergence. The same argument applies to the upper lift, giving
\[
 \delta_{1,p_1}:=\|T_{1,p_1}(\mathbf r)-\mathbf r\|_{L^2(\nu)}\to0,
 \qquad
 \delta_{2,p_2}:=\|T_{2,p_2}(\eta)-\eta\|_{L^2(\mu)}\to0
\]
almost surely.

Extend each sampled neuron profile to be constant on its cell.
Because every cell has probability $1/p_j$, Gaussian integrals of these
extensions equal the original empirical averages. Extend $C_t$ on products
of lower and upper cells, and $S_t$ on upper cells. These extensions
preserve the norms and all update equations.
The finite-width and limiting solutions can therefore be compared in
the same Hilbert space \eqref{eq:three-proof-state-space}.

At fixed order parameters, replacing $\mathbf r$ by
$T_{1,p_1}(\mathbf r)$ changes $z_1$ by at most
$C_T\delta_{1,p_1}$ in $L^2$.
The first-layer kernel and refreshed forward filter satisfy the same
bound. Replacing the upper lift changes the upper preactivation by at
most $C_T(\delta_{1,p_1}+\delta_{2,p_2})$.
Lemma~\ref{lem:three-mode-continuity} gives this bound for the resulting
change in the selected pair and hence in the update equations.
The constants are uniform in width when the response coefficients
and lift moments are bounded and Assumption~\ref{ass:three-regularity}
holds uniformly.
The initial responses also converge: their coefficients solve the kernel
linear system above \eqref{eq:three-response-initialization}, whose
limiting matrix is invertible by assumption.

Let $E(t)$ denote the distance between the two solutions after this
identification, in the norm of \eqref{eq:three-proof-state-space}.
Integrating the difference of their update equations and applying
Gr\"onwall gives
\begin{equation}
 \sup_{t\le T}E(t)
 \le C_T\bigl(E(0)+\delta_{1,p_1}+\delta_{2,p_2}\bigr).
 \label{eq:three-proof-width-stability}
\end{equation}
Choose a neighborhood of the limiting trajectory on which
Assumption~\ref{ass:three-regularity} holds uniformly.
The estimate applies until the finite-width solution
leaves this neighborhood. Its right-hand side tends to zero, so for
sufficiently large widths the solution stays in the neighborhood
throughout $[0,T]$.

To obtain the sequential limit, fix $p_2$ and an admissible set of upper
lift rows. The two flows compared in \eqref{eq:three-proof-width-stability}
use these same rows, so the upper-lift error is zero. As $p_1\to\infty$,
the remaining error tends to zero almost surely. The limit therefore
uses Gaussian expectations in the first layer and empirical averages
over the fixed rows in the second. We then let $p_2\to\infty$ and apply
the same estimate to this intermediate evolution, obtaining the equations
with Gaussian expectations in both layers. The positive-definite initial
limiting kernel and stability under the spectral-gap assumption ensure
that all sufficiently large upper samples are admissible almost surely.

Continuity of the forward filter and the maximizing pair gives convergence
of the filtered features and selected scores. Cauchy--Schwarz then gives
convergence of the output and kernels from their expressions as inner
products. Together with \eqref{eq:three-proof-width-stability}, these
bounds establish almost-sure convergence uniformly on $[0,T]$ in the
stated norms. This completes the proof of
Theorem~\ref{thm:three-layer-closed-flow-formal}.
\end{proof}

\section{The exact linear reduction}
\label{sec:linear-dynamics}
We first derive \eqref{eq:main-linear-modes} for the
one-hidden-layer linear network $f(x)=\mathbf W_2\mathbf W_1x$. Take unnormalized linear activations, unit-rate
squared-loss gradient flow, whitened inputs $\mathbb E[xx^\top]=I$,
and possibly vector-valued targets. Write
$S=\mathbb E[yx^\top]$. Direct differentiation gives
\begin{equation}
 A_\Delta=S-\mathbf W_2\mathbf W_1,\qquad
 \dot{\mathbf W}_1=\mathbf W_2^\top A_\Delta,\qquad
 \dot{\mathbf W}_2=A_\Delta\mathbf W_1^\top.
 \label{eq:linear-exact-flow}
\end{equation}
Thus the incoming row and its matching outgoing column obey the
two actions in \eqref{eq:main-coupled-residual} exactly, with the
residual input--output correlation $A_\Delta$ as coupling matrix. The prediction term produces
saturation as the network fits the target.

Let $S=\sum_i s_i u_i v_i^\top$ and choose orthonormal hidden
directions $q_i$. The aligned initialization
\begin{equation}
 \mathbf W_1=\sum_i\alpha_iq_iv_i^\top,\qquad
 \mathbf W_2=\sum_i\beta_iu_iq_i^\top
 \label{eq:linear-aligned}
\end{equation}
is preserved: substituting it into \eqref{eq:linear-exact-flow}
and using orthogonality gives
\begin{equation}
 \dot\alpha_i=\beta_i(s_i-\alpha_i\beta_i),\qquad
 \dot\beta_i=\alpha_i(s_i-\alpha_i\beta_i),\qquad
 \frac{d}{dt}(\alpha_i^2-\beta_i^2)=0.
 \label{eq:linear-mode-pair}
\end{equation}
For a balanced positive seed,
$m_i=\alpha_i\beta_i$ satisfies $\dot m_i=2m_i(s_i-m_i)$, hence
\[
 m_i(t)=\frac{s_i}{1+(s_i/m_i(0)-1)e^{-2s_it}}.
\]
When $m_i(0)\ll s_i$, each factor initially grows as $e^{s_it}$, and
their product approaches $s_i$. Alignment is needed for independent scalar modes;
an arbitrary initialization also has cross-mode interactions.

Without alignment, linearizing \eqref{eq:linear-exact-flow} at
zero still gives, for each row--column pair,
\[
 \frac{d}{dt}\begin{pmatrix}\mathbf w_{1,j}\\
                   \boldsymbol\omega_{2,j}\end{pmatrix}
 =
 \begin{pmatrix}0&S^\top\\S&0\end{pmatrix}
 \begin{pmatrix}\mathbf w_{1,j}\\
                   \boldsymbol\omega_{2,j}\end{pmatrix}.
\]
The discarded terms are cubic. The block matrix has eigenvectors
$(v_i,\pm u_i)/\sqrt2$ with eigenvalues $\pm s_i$, which proves the
shared exponential amplification near zero for arbitrary initial
directions. This connects the linear dynamics to Theorem~\ref{thm:controlled-feedback};
the nonlinear surrogate additionally freezes gates and replaces the
residual by a label-based approximation.

\paragraph{More than two weight matrices.}
For $f(x)=\mathbf W_D\cdots\mathbf W_1x$, aligned singular directions
reduce each mode to layer amplitudes $a_{1i},\ldots,a_{Di}$.
Writing $m_i=\prod_{k=1}^D a_{ki}$, its loss is
$\tfrac12(s_i-m_i)^2$. Differentiating gives the mode equations of
\citet[Section~2]{saxe2014exact}:
\begin{equation}
 \dot a_{ki}=(s_i-m_i)\prod_{r\ne k}a_{ri},\qquad
 \dot m_i=(s_i-m_i)\sum_{k=1}^D
                         \left(\prod_{r\ne k}a_{ri}\right)^2.
 \label{eq:linear-deep-product}
\end{equation}
Equal positive initialization across layers is preserved by these
equations. Then every $a_{ki}=m_i^{1/D}$,which gives
\eqref{eq:main-deep-linear-mode}.
For $D=2$, this is the balanced
two-factor equation $\dot m_i=2m_i(s_i-m_i)$. For $D>2$, the early
growth law is $\dot m_i\simeq Ds_i m_i^{2-2/D}$.
The additional powers arise from multiplying the contributions of
all layers. The constant-rate exponential approximation for two
weight matrices therefore does not describe the simultaneous growth
of all layers in a deeper network.

\section{The Deep Staircase: A Theoretical Setting}
\label{app:toy-setting}

We analyze here the deep-staircase model of
Section~\ref{sec:hierarchical-discovery}. We first define the complete
teacher, then specialize the forward and backward spectral operations of
Algorithms~\ref{alg:neural-lofi}~and~\ref{alg:combined} to this setting.
We derive the population
mechanism separating the forward and label-corrected statistics, and the
resulting sample-complexity scalings. For simplicity, we do not provide
full finite-sample proofs. These derivations can be turned into rigorous
recovery guarantees by combining the population calculations below with
standard concentration and spectral-perturbation arguments, as in \cite{tabanelli2026deep,wortsmanzurich2026scalinglaws}

\subsection{Teacher and high-dimensional regime}
\label{app:toy-teacher}

\paragraph{Target.}
Let $\mathbf x\sim\mathcal N(0,I_d)$ and let
$\Phi_2(\mathbf x)\in\mathbb R^{D_2(d)}$,
$D_2(d)=d(d+1)/2$, denote the complete normalized quadratic Hermite
representation of $\mathbf x$.
We draw
\[
    A_\star=
    \begin{bmatrix}
        A_U\\ A_V\\ A_W
    \end{bmatrix}
    \in\mathbb R^{d_1\times D_2(d)},
    \qquad
    d_1=p_U+p_V+p_W,
    \qquad
    (A_\star)_{ij}
    \stackrel{\mathrm{iid}}{\sim}
    \mathcal N\!\left(0,D_2(d)^{-1}\right),
\]
and define the first-level features
\begin{equation}
    \mathbf h_U=A_U\Phi_2(\mathbf x),\qquad
    \mathbf h_V=A_V\Phi_2(\mathbf x),\qquad
    \mathbf h_W=A_W\Phi_2(\mathbf x).
    \label{eq:toy-blocks}
\end{equation}
This prior makes the planted directions dense and delocalized in the
quadratic feature space. In the high-dimensional regime below,
$A_\star A_\star^\top\to I_{d_1}$ in operator norm. At the second level, let $B_Q, B_Z\in\mathbb R^{p_U\times p_U}$ be symmetric and define
\begin{equation}
    Q(\mathbf h_U)
    =
    \frac{
        \mathbf h_U^\top B_Q\mathbf h_U-\operatorname{tr}(B_Q)
    }{\sqrt2},
    \qquad
    Z(\mathbf h_U)
    =
    \frac{
        \mathbf h_U^\top B_Z\mathbf h_U-\operatorname{tr}(B_Z)
    }{\sqrt2},
    \qquad
    S=
    \operatorname{sign}
    \bigl(\mathbf h_V^\top C\mathbf h_W\bigr).
    \label{eq:toy-coordinates}
\end{equation}
We assume that $B_Q$ and $B_Z$ are symmetric, have unit Frobenius norm
and are Frobenius orthogonal, and that there exists an orthogonal
transformation $T$ such that
$T^\top B_QT=B_Q$ and $T^\top B_ZT=-B_Z$.
Thus, in the Gaussian latent reference, $Q$ and $Z$ are centered,
unit-variance and uncorrelated, with
$(Q,Z)\overset{d}{=}(Q,-Z)$.
We take $p_V=p_W=p_h$, let
$C\in\mathbb R^{p_h\times p_h}$ be an independent hidden coupling matrix. The raw target is
\begin{equation}
    y_{\rm raw}
    =
    \tanh\!\left(
        (Q+ZS)/\sqrt2+b_0
    \right).
    \label{eq:toy-raw-target}
\end{equation}
We finally standardize $y_{\rm raw}$ of mean $\mu_y$ and variance $\sigma^2_y$ to $y = (y_{\rm raw}-\mu_y)/\sigma_y$ with zero mean and unit variance.

\paragraph{High-dimensional regime.} For the high-dimensional regime we take
\begin{equation}
    p_U,p_V,p_W\asymp d^\epsilon,
    \qquad
    d_1\asymp d^\epsilon,
    \qquad
    0<\epsilon<1,
    \qquad
    n\asymp d^\alpha.
    \label{eq:toy-regime}
\end{equation}
For the hidden-subspace sample complexity result we set $p_V=p_W=p_h$ and assume that
$C$ is full rank with balanced singular values
\begin{equation}
    c_-p_h^{-1/2}
    \leq
    \sigma_j(C)
    \leq
    c_+p_h^{-1/2},
    \qquad
    0<c_-\leq c_+<\infty.
    \label{eq:toy-C-conditioning}
\end{equation}
This assumption fixes the scale $\|C\|_F=\Theta(1)$ while excluding
hidden directions whose interaction strength vanishes asymptotically
relative to the others.

\subsection{BackLoFi in the deep-staircase setting}
\label{app:toy-estimator}

\paragraph{Student architecture and Hermite specialization.}
We first relate the estimator used below to the general Neural LoFi
architecture. In the numerical model, the raw input is first mapped through
a frozen nonlinear random feature map
\[
    \mathbf z_0(\mathbf x)=\sigma(W_0\mathbf x).
\]
Each subsequent Neural LoFi layer has the reduce--lift form
\[
    \mathbf h_\ell^\star
    =
    F_\ell^\top\mathbf z_{\ell-1},
    \qquad
    \mathbf h_\ell
    =
    R_\ell\mathbf h_\ell^\star,
    \qquad
    \mathbf z_\ell
    =
    \sigma(\mathbf h_\ell),
\]
where $F_\ell$ is estimated from the data and $R_\ell$ is a fixed random
lift. Thus $\mathbf h_1^\star$ plays the role of the first teacher level
$(\mathbf h_U,\mathbf h_V,\mathbf h_W)$, while the next filtered
coordinates $\mathbf h_2^\star$ represent second-level features such as
$Q$ and $Z$. For the theoretical analysis, we replace all the random feature map by
its complete degree-two Hermite counterpart, for instance, for the first level we consider
$\mathbf z_0(\mathbf x)=\Phi_2(\mathbf x)$. The numerical estimator keeps
the random feature maps.

\paragraph{Forward pass.}
For a scalar sample weight $\xi$ and a representation $\mathbf z$, define
\begin{equation}
    \widehat m_{\mathbf z}[\xi]
    =
    \widehat{\mathbb E}_n
    \left[
        (\xi-\bar\xi)\mathbf z
    \right],
    \qquad
    \widehat C_{\mathbf z}[\xi]
    =
    \widehat{\mathbb E}_n
    \left[
        (\xi-\bar\xi)\mathbf z\mathbf z^\top
    \right].
    \label{eq:toy-empirical-moments}
\end{equation}
We denote by $m_{\mathbf z}[\xi]$ and $C_{\mathbf z}[\xi]$ the
corresponding population quantities obtained by replacing the empirical
average and empirical mean by their expectations.

The first Forward step acts on
$\mathbf z_0=\Phi_2(\mathbf x)$ and, in the present setting, uses
$\widehat C_{\mathbf z_0}[y]$ to recover the visible
$\mathbf h_U$ subspace. Once this representation has been learned, the
upper Forward step acts on $\mathbf z_1=\Phi_2(\mathbf h_U)$ the quadratic Hermite representation of the
recovered visible coordinates and uses $m_{\mathbf z_1}[y]$ and $C_{\mathbf z_1}[y]$ respectively its first and second Hermite channels to estimate the two relevant second-level directions $Q$ and
$Z$ respectively. Their precise population components are computed in the next
subsection.
Finally, the recovered second-level coordinates are passed through the
last frozen random-feature map, $\mathbf z_2 = \sigma\!\left(R_2\mathbf h^{(2)}\right)$,
and a scalar ridge readout is fitted,
\[
    \widehat{\mathbf w}_{\rm out}
    =
    \arg\min_{\mathbf w}
    \widehat{\mathbb E}_n
    \left[
        \left(
            y-\mathbf w^\top\mathbf z_2
        \right)^2
    \right]
    +
    \beta\|\mathbf w\|^2,
    \qquad
    \widehat f
    =
    \widehat{\mathbf w}_{\rm out}^\top\mathbf z_2 .
\]

\paragraph{Backward score.}
In the general BackLoFi notation, the cross-layer operator at the first
boundary is
\[
    \mathcal A
    =
    \widehat{\mathbb E}_n
    \left[
        y\,\mathbf g^{2}(\mathbf x)\mathbf z_0(\mathbf x)^\top
    \right] =
    U_A\Sigma V_A^\top
\]
where $\mathbf g^{2}$ is the sensitivity obtained by differentiating the
fitted upper predictor with respect to its preactivation coordinates.
In the Hermite specialization, after fitting the upper Hermite predictor
$\widehat f_H$.
The cross-layer SVD defined above retains left singular vectors that
define the sample-dependent score
\begin{equation}
    s_i(\mathbf x)
    =
    u_i^\top\mathbf g(\mathbf x).
    \label{eq:toy-learned-scores}
\end{equation}
In the present setting we retain only the row, or effective-label,
channel of BackLoFi. The right singular vectors are not directly
installed and no paired column correction is performed. The information
returned to the lower spectral estimator is therefore entirely contained
in the scores $s_i$.

\paragraph{Second Forward pass.}
For each retained score, consider the effective label $y\,s_i$, we apply the same lower covariance estimator as in the first Forward $\widehat C_{\mathbf r_0}[y\,s_i]$
pass. After projecting out the already recovered visible subspace, the leading
directions of this operator provide candidates for the missing
$\mathbf h_V,\mathbf h_W$ subspace. If several score channels are
retained, their projected covariance operators are combined as specified
in Algorithm~\ref{alg:toy-covariance-row}. The selected directions are
then appended to the visible $\mathbf h_U$ subspace to form the corrected
first-level representation $\widehat{\mathbf h}^{(1)}$. At the second
level, we set
$\mathbf z_1^{\rm B}=\Phi_2(\widehat{\mathbf h}^{(1)})$ and apply again
the Forward LoFi moments
$m_{\mathbf z_1^{\rm B}}[y]$ and
$C_{\mathbf z_1^{\rm B}}[y]$.
The resulting second-level representation is finally used to refit the
predictor against the original labels.

\begin{algorithm}[t]
\caption{Covariance-row BackLoFi for the deep staircase}
\label{alg:toy-covariance-row}
\begin{algorithmic}[1]
\setlength{\itemsep}{1pt}
\Require Data $\{(\mathbf x_\mu,y_\mu)\}_{\mu=1}^n$,
first-level dimensions $(p_U,p_V,p_W)$, retained upper rank $k_2$.

\State Form the prescribed lower representation
$\mathbf z_0=\sigma(W_0\mathbf x)$ and compute
$\widehat C_{\mathbf z_0}[y]$.
Retain its leading $p_U$ directions by absolute eigenvalue and denote
their projector by $\widehat P_U$.

\State Let $\widehat{\mathbf h}_U$ be the corresponding visible
coordinates and form
\[
    \mathbf z_1=\Phi_2(\widehat{\mathbf h}_U).
\]
Apply upper Forward LoFi: use the direction of
$\widehat m_{\mathbf z_1}[y]$ to estimate $Q$, then project this direction
out of the second-order statistic and retain its leading eigenvector to
estimate $Z$.

\State Fit the upper predictor $\widehat f_H$ and compute its fitted
sensitivity $\mathbf g$.

\State Compute the cross-layer SVD
\[
    \widehat{\mathcal A}
    =
    \widehat{\mathbb E}_n
    \left[
        y\,\mathbf g\,\mathbf z_0^\top
    \right]
    =
    U_A\Sigma V_A^\top .
\]

\For{$\mathcal I\in\{\{1\},\{2\},\{1,2\}\}$}
    \For{$i\in\mathcal I$}
        \State Set
        $s_i=u_i^\top\mathbf g$ and form
        $\widehat C_{\mathbf z_0}[y s_i]$.
    \EndFor

    \State Set $\widehat P_U^\perp=I-\widehat P_U$ and form
    \[
        T_{\mathcal I}
        =
        \sum_{i\in\mathcal I}
        \left(
            \widehat P_U^\perp
            \widehat C_{\mathbf z_0}[y s_i]
            \widehat P_U^\perp
        \right)^2 .
    \]

    \State Retain the leading $p_V+p_W$ eigenvectors of
    $T_{\mathcal I}$ and append them to the protected visible subspace,
    giving the corrected first-level representation
    $\widehat{\mathbf h}^{(1)}_B$.

    \State Form the corrected upper quadratic representation
    \[
        \mathbf z_1^B
        =
        \Phi_2(\widehat{\mathbf h}^{(1)}_B).
    \]

    \State Rerun Forward LoFi with the original labels $y$:
    retain the normalized direction selected by
    $\widehat m_{\mathbf z_1^B}[y]$, then retain $k_2-1$ additional
    directions from the corresponding deflated second-Hermite/Wick
    operator. Denote the resulting upper coordinates by
    $\widehat{\mathbf h}^{(2)}_B$.

    \State Refit the scalar predictor against the original labels $y$
    by ridge regression on
    \[
        \left[
            \widehat{\mathbf h}^{(2)}_B;
            \Phi_2(\widehat{\mathbf h}^{(2)}_B)
        \right].
    \]
\EndFor

\State Select $\mathcal I$ by validation error, and return the corresponding predictor
and corrected first-level representation.
\end{algorithmic}
\end{algorithm}

\subsection{The staircase mechanism in population}
\label{app:toy-population}

\paragraph{Target decomposition and parity.} For the population proof only, set
$t_\pm=\frac{Q\pm Z}{\sqrt2}+b_0$
and define
\begin{equation}
    a(Q,Z)
    =
    \frac{1}{2\sigma_y}\big(
        \tanh(t_+)+\tanh(t_-)-2\mu_y
    \big),
    \qquad
    b(Q,Z)
    =
    \frac{1}{2\sigma_y}\big( \tanh(t_+)-\tanh(t_-) \big).
    \label{eq:toy-ab}
\end{equation}
Then the standardized target satisfies exactly
\begin{equation}
    y=a(Q,Z)+b(Q,Z)S,
    \label{eq:toy-decomposition}
\end{equation}
where $a$ is even in Z while $b$ is odd in $Z$. Their expansions around
$(Q,Z)=(0,0)$ make the first relevant components explicit:
\begin{equation}
\begin{aligned}
    a(Q,Z)
    &=
    a_0
    +
    \frac{\tanh'(b_0)}{\sqrt2\,\sigma_y}\,Q
    +
    \frac{\tanh''(b_0)}{4\sigma_y}
    \left(Q^2+Z^2\right)
    +
    O\!\left(\|(Q,Z)\|^3\right),\\
    b(Q,Z)
    &=
    \frac{\tanh'(b_0)}{\sqrt2\,\sigma_y}\,Z
    +
    \frac{\tanh''(b_0)}{2\sigma_y}\,QZ
    +
    O\!\left(\|(Q,Z)\|^3\right),
\end{aligned}
\label{eq:toy-ab-expansion}
\end{equation}
with $a_0=(\tanh(b_0)-\mu_y)/\sigma_y$. The corresponding expansion in $\mathbf h_U$ follows immediately, since
both $Q$ and $Z$ are quadratic features of $\mathbf h_U$.

\paragraph{Lower-level Forward operator.}
We first analyze the population counterpart of the lower Forward
estimator, since
$\mathbf h^{(1)}
=(\mathbf h_U,\mathbf h_V,\mathbf h_W)
=A_\star\Phi_2(\mathbf x)$,
its restriction to the planted first-level coordinates is
\[
    A_\star C_{\mathbf z_0}[y]A_\star^\top
    =
    A_\star \mathbb E
    \left[
        y\,\Phi_2(\mathbf x)\Phi_2(\mathbf x)^\top
    \right]A_\star^\top
    =
    \mathbb E
    \left[\big(a(Q,Z)+b(Q,Z)S\big)
        \,\mathbf h^{(1)}
        \mathbf h^{(1)\top}
    \right].
\]
Hence the $U,U$, $V,V$, $W,W$ and $V,W$ blocks of this matrix are
exactly the corresponding population blocks of the operator diagonalized
by Forward LoFi.
In the Gaussian latent limit, the three first-level blocks are independent
standard Gaussian vectors. In the decomposition above, $a(Q,Z)$ is even
in $Z$, while $b(Q,Z)$ is odd. Since the output is centered and the hidden
sign $S$ is balanced, $\mathbb E[a(Q,Z)]=0$, while the symmetry
$Z\mapsto-Z$ gives $\mathbb E[b(Q,Z)]=0$. Hence the complete hidden block
of the first Forward operator vanishes:
\begin{equation}
\mathbb E\!\left[
y
\begin{pmatrix}
\mathbf h_V\mathbf h_V^\top &
\mathbf h_V\mathbf h_W^\top\\
\mathbf h_W\mathbf h_V^\top &
\mathbf h_W\mathbf h_W^\top
\end{pmatrix}
\right]
=
\underbrace{\mathbb E[a(Q,Z)]}_{=0}
I_{p_V+p_W}
+
\underbrace{\mathbb E[b(Q,Z)]}_{=0\ {\rm by\ parity}}
\,
\mathbb E\!\left[
S
\begin{pmatrix}
\mathbf h_V\mathbf h_V^\top &
\mathbf h_V\mathbf h_W^\top\\
\mathbf h_W\mathbf h_V^\top &
\mathbf h_W\mathbf h_W^\top
\end{pmatrix}
\right]
=0.
\label{eq:toy-forward-hidden-cancellation}
\end{equation}
The mixed visible-hidden blocks vanish by the same sign symmetries. By contrast, the visible block is nonzero:
\begin{equation}
A_U C_{\mathbf z_0}[y]A_U^\top
=
\mathbb E\!\left[
a(Q,Z)\mathbf h_U\mathbf h_U^\top
\right]
\neq0,
\end{equation}
accordingly to the expansion of $a(Q,Z)$ in \ref{eq:toy-ab-expansion}. Thus the first Forward operator
contains signal on $\mathbf h_U$ while its entire
$\mathbf h_V\oplus\mathbf h_W$ block vanishes.

\paragraph{Upper-level Forward operator.}
Once the visible subspace has been recovered, the next Forward step acts
on $\mathbf z_1=\Phi_2(\mathbf h_U)$. Averaging over the hidden blocks
removes the $b(Q,Z)S$ contribution and gives
\[
    \mathbb E_{\mathbf h_V,\mathbf h_W}
    [y\mid\mathbf h_U]
    =
    a(Q,Z).
\]
Since $Q$ and $Z$ are quadratic coordinates of $\mathbf h_U$, they are
linear directions of $\mathbf z_1$. In population, the expansion~\eqref{eq:toy-ab-expansion} enters directly
into the two Forward statistics, at the leading order it yields:
\begin{align}
    m_{\mathbf z_1}[y]
    &=
    \mathbb E\!\left[a(Q,Z)\mathbf z_1\right]
    =
    \frac{\tanh'(b_0)}{\sqrt2\,\sigma_y}
    \underbrace{\mathbb E[Q\mathbf z_1]}_{\propto\, B_Q} ,
    \label{eq:toy-upper-first-moment}
    \\
    C_{\mathbf z_1}[y]
    &=
    \mathbb E\!\left[
        a(Q,Z)\mathbf z_1\mathbf z_1^\top
    \right]
    =
    \frac{\tanh''(b_0)}{4\sigma_y}
    \underbrace{
        \mathbb E\!\left[
            (Q^2+Z^2))\mathbf z_1\mathbf z_1^\top
        \right]
    }_{\text{spikes along } B_Q \text{ and } B_Z}.
    \label{eq:toy-upper-second-moment}
\end{align}
The first moment identifies the weight $B_Q$ associated with $Q$. We
then restrict the pure second-Hermite operator to the orthogonal
complement of $B_Q$, so that its leading informative direction is
$B_Z$. Projecting these two recovered weights back onto $\mathbf z_1$
gives the scalar coordinates $Q$ and $Z$. We therefore retain the
two-dimensional upper representation
\[
    \mathbf h^{(2)}=(Q,Z)\in\mathbb R^2.
\]
These are intermediate features of the same scalar hierarchy, which are
subsequently combined by the upper readout and provide the two channels
used by the backward pass.

\paragraph{Upper readout.}
After recovering $\mathbf h^{(2)}=(Q,Z)$, the Forward student applies
the final random-feature map
$\mathbf z_2=\sigma(R_2\mathbf h^{(2)})$ and fits a linear ridge
readout. To see the population components selected by this fit, let
$\mathbf r_i^\top=(r_{i,Q},r_{i,Z})$ denote the $i$-th row of $R_2$.
The corresponding random-feature coordinate is
\[
    z_{2,i}
    =
    \sigma\!\left(r_{i,Q}Q+r_{i,Z}Z\right),
\]
whose low-order Hermite expansion is
\[
    z_{2,i}
    =
    \alpha_0
    +
    \alpha_1(r_{i,Q}Q+r_{i,Z}Z)
    +
    \alpha_2\left[
        r_{i,Q}^2H_2(Q)
        +\sqrt2\,r_{i,Q}r_{i,Z}QZ
        +r_{i,Z}^2H_2(Z)
    \right]
    +\cdots .
\]
Since the visible regression function $a(Q,Z)$ is even in $Z$,
\[
\begin{aligned}
    \mathbb E[a(Q,Z)z_{2,i}]
    &=
    \alpha_1r_{i,Q}\,\mathbb E[aQ]
    +
    \alpha_2r_{i,Q}^2\,\mathbb E[aH_2(Q)] \\
    &\quad+
    \alpha_2r_{i,Z}^2\,\mathbb E[aH_2(Z)]
    +\cdots ,
\end{aligned}
\]
whereas the linear $Z$ and $QZ$ components vanish by parity.
In the Gaussian RF limit, the population feature covariance is
block-diagonal across Hermite sectors, so the ridge fit only rescales
these components. Consequently, the first relevant terms of the fitted
predictor have the form
\[
    \widehat f(Q,Z)
    =
    \gamma_Q Q
    +
    \gamma_{QQ}H_2(Q)
    +
    \gamma_ZH_2(Z)
    +
    \cdots .
\]

\paragraph{Backward score.}
Consequently, the fitted upper sensitivity with respect to the
second-level coordinates satisfies
\begin{equation}
    \mathbf g(Q,Z)
    =
    \nabla_{(Q,Z)}\widehat f(Q,Z)
    =
    \begin{pmatrix}
        \gamma_Q+\sqrt2\,\gamma_{QQ}Q+\cdots\\
        \sqrt2\,\gamma_Z Z+\cdots
    \end{pmatrix}.
    \label{eq:toy-upper-sensitivity}
\end{equation}
The relevant distinction is their parity in $Z$: the first response is
even in $Z$, whereas the second is odd and starts linearly in $Z$.
Since the hidden coefficient $b(Q,Z)$ is itself odd in $Z$, only the
latter can produce a non-zero hidden coupling after backward modulation.
Indeed, let $\mathcal A_Z$ denote the cross-layer row associated with
the second component of~\eqref{eq:toy-upper-sensitivity}. We identify
quadratic Hermite coefficient vectors with their symmetric-matrix
representation and keep the same notation for both. We approximate the
latent moments by replacing $(\mathbf h_U,\mathbf h_V,\mathbf h_W)$
with independent standard Gaussian vectors $(\mathbf u,\mathbf v,\mathbf w)$
of the same dimensions. The actual latent blocks are quadratic functions
of the same input $\mathbf x$; their joint law is not exactly Gaussian.
We sketch the reason for this approximation below.
In $\mathbb E_{\mathbf u}$, $Q,Z$ are evaluated at $\mathbf h_U=\mathbf u$;
in $\mathbb E_{\mathbf v,\mathbf w}$,
$S=\operatorname{sign}(\mathbf v^\top C\mathbf w)$, with $C$ fixed.
The second-order $V,W$ contraction gives
\begin{align}
    \mathcal A_Z
    &=
    \mathbb E\!\left[
        \bigl(a(Q,Z)+b(Q,Z)S\bigr)g_Z(Q,Z)
        \frac{\mathbf x\mathbf x^\top-I_d}{\sqrt2}
    \right]
    \nonumber\\
    &=
    \sqrt2\,\mathbb E_{\mathbf u}[b(Q,Z)g_Z(Q,Z)]
    \sum_{a=1}^{p_V}\sum_{b=1}^{p_W}
    \mathbb E_{\mathbf v,\mathbf w}[S\,v_aw_b]
    \left(
        A_{V,a}A_{W,b}+A_{W,b}A_{V,a}
    \right)
    +\mathcal R_Z,
    \label{eq:toy-cross-layer-contraction}
\end{align}
where $\mathcal R_Z$ collects the other contractions and the error from
replacing the dense quadratic latents by their Gaussian reference.
The product of two normalized quadratic features satisfies the exact identity
\begin{equation}
    h_{V,a}h_{W,b}
    =
    \operatorname{tr}(A_{V,a}A_{W,b})
    +
    \underbrace{
    \left\langle
        A_{V,a}A_{W,b}+A_{W,b}A_{V,a},
        \mathbf x\mathbf x^\top-I_d
    \right\rangle_F
    }_{\text{second Hermite order}}
    +
    \underbrace{
    \operatorname{Proj}_{\mathcal H_4}
    \!\left(h_{V,a}h_{W,b}\right)
    }_{\text{fourth Hermite order}} .
    \label{eq:toy-vw-hermite-product}
\end{equation}
The constant term has zero projection onto the quadratic Hermite
features. For the leading response $g_Z=\sqrt2\gamma_Z Z$ with
$\gamma_Z\neq0$, the displayed contraction is nonzero: strict
monotonicity of $\tanh$ gives $b(Q,Z)Z>0$ for $Z\neq0$, so
$\mathbb E_{\mathbf u}[bg_Z]=\sqrt2\gamma_Z\mathbb E_{\mathbf u}[bZ]\neq0$.
Its size is calculated in Section~\ref{app:toy-finite-sample}.

If $\mathcal A=U_A\Sigma V_A^\top$ and a retained left singular vector
$u_i=(u_{i,Q},u_{i,Z})$ has $u_{i,Z}\neq0$, its score has
the form
\begin{equation}
\label{eq:toy-backward-score}
    s_i(Q,Z)
    =
    u_i^\top\mathbf g(Q,Z)
    =
    s_{i,\mathrm{even}}(Q,Z)
    +
    \eta_i Z
    +
    \cdots ,
    \qquad
    \eta_i\neq0
\end{equation}

\paragraph{Sensitivity through the random lift.}
The sensitivity used below is written with respect to the filtered
second-level coordinates $\mathbf h^{(2)}=(Q,Z)$, whereas the general
BackLoFi construction uses the sensitivity with respect to the neuron
preactivations. Here these preactivations are
$R_2\mathbf h^{(2)}$. If
\[
    \mathbf g_{\rm neu}
    =
    \nabla_{R_2\mathbf h^{(2)}}\widehat f,
\]
then the chain rule gives
\begin{equation}
    \mathbf g
    :=
    \nabla_{\mathbf h^{(2)}}\widehat f
    =
    R_2^\top\mathbf g_{\rm neu},
    \qquad
    \mathcal A
    =
    \mathbb E
    \!\left[
        y\,\mathbf g\,\mathbf z_0^\top
    \right]
    =
    R_2^\top
    \mathbb E
    \!\left[
        y\,\mathbf g_{\rm neu}\mathbf z_0^\top
    \right].
    \label{eq:toy-random-lift-sensitivity}
\end{equation}
Thus the two-dimensional cross-layer operator considered below is the
neuron-level operator restricted to the subspace generated by the random
lift. At finite width, orthonormalizing this subspace amounts to replacing
$\mathbf g$ by
$(R_2^\top R_2)^{-1/2}\mathbf g$. For an iid Gaussian lift of fixed
rank two,
\[
    \frac{1}{p}R_2^\top R_2
    \longrightarrow I_2,
\]
so this only multiplies the score $u_i^\top\mathbf g$ by the common factor
$p^{-1/2}$ asymptotically, which does not change the recovered lower
eigenspace. We therefore use the simpler reduced-coordinate sensitivity
$\mathbf g$ below.

\paragraph{Second Forward pass.}
For each retained score $s_i$ in~\eqref{eq:toy-backward-score}, the
backward correction simply reruns the lower Forward covariance with the
effective label $y s_i$. In planted coordinates,
\begin{equation}
    A_\star C_{\mathbf z_0}[y s_i]A_\star^\top
    =
    \mathbb E\!\left[
        y s_i\,
        \mathbf h^{(1)}\mathbf h^{(1)\top}
    \right].
    \label{eq:toy-backward-operator}
\end{equation}
Its $V,W$ block is, at leading order in the odd component of the score,
\begin{align}
    \mathbb E\!\left[
        y s_i\,\mathbf h_V\mathbf h_W^\top
    \right]
    =
    \underbrace{\mathbb E\!\left[
        b(Q,Z)s_i(Q,Z)
    \right]}_{\eta_i\,
    \mathbb E\!\left[b(Q,Z)Z\right]+\cdots}
    \mathbb E\!\left[
        S\,\mathbf h_V\mathbf h_W^\top
    \right] \neq 0
    \label{eq:toy-backward-hidden-block}
\end{align}
The even part of $s_i$ does not contribute because $b(Q,Z)$ is odd in
$Z$. Thus every retained mode with $\eta_i\neq0$ produces a covariance
spike on the hidden $V,W$ subspace, which was absent from the ordinary
Forward operator. When several scores are retained, their projected covariance operators
are combined using the same squared-covariance aggregation as in the
estimator, before extracting the hidden eigenspace.

Since $C$ has full rank, the hidden restriction of
$C_{\mathbf z_0}[ys_i]$ has range
$\mathbf h_V\oplus\mathbf h_W$. Its leading $p_V+p_W$ eigendirections
therefore recover the joint hidden first-level subspace. We do not need
to identify the $V$ and $W$ bases separately: the subspace is recovered
up to an orthogonal change of coordinates. Projecting the data onto the
visible and recovered hidden directions yields the corrected
first-level representation
\[
    \widehat{\mathbf h}_B^{(1)}
    =
    \bigl(
        \widehat{\mathbf h}_U,
        \widehat{\mathbf h}_{VW}
    \bigr).
\]

\paragraph{Corrected upper Forward pass.}
We then resume the ordinary
Forward hierarchy using the original labels $y$ and form
\[
    \mathbf z_1^{B}
    =
    \Phi_2\!\left(\widehat{\mathbf h}^{(1)}_B\right).
\]
At population level, the relevant quadratic coordinates in this dictionary
are $Q$, $Z$, and
$T=\mathbf h_V^\top C\mathbf h_W$.
The coefficients of $m_{\mathbf z_1^B}[y]$ along the relevant
quadratic directions satisfy
\[
    \mathbb E[yQ]\neq0,
    \qquad
    \mathbb E[yZ]=0,
    \qquad
    \mathbb E[yT]=0.
\]
We retain this direction the $Q$ direction and then apply the covariance estimator $C_{\mathbf z_1^{B}}[y]$ operator on its orthogonal complement. The new feature made available by the backward
correction is the nonzero coupling
\begin{equation}
\begin{aligned}
    \mathbb E[yZT]
    &=
    \mathbb E\!\left[
        \bigl(a(Q,Z)+b(Q,Z)S\bigr)ZT
    \right] =
    \mathbb E[b(Q,Z)Z]\,
    \mathbb E|T|
    >0 ,
\end{aligned}
\label{eq:toy-corrected-upper-coupling}
\end{equation}
where $S=\operatorname{sign}(T)$. Hence the $Q$-deflated upper operator
has a nontrivial component on $\operatorname{span}\{Z,T\}$. Under a spectral-gap assumption, the corrected upper Forward step retains
a rank-$k_2$ representation $\mathbf h_B^{(2)}\in\mathbb R^{k_2}$
whose informative population span contains $Q$ together with
$\operatorname{span}\{Z,T\}$.
The scalar predictor is then refitted against the original labels $y$ using
a ridge readout on $\bigl[ \mathbf h_B^{(2)};\Phi_2(\mathbf h_B^{(2)}) \bigr]$
Since the retained upper span contains both $Z$ and $T$, this quadratic
readout contains the interaction $ZT$, providing a non-zero low-degree
prediction channel for the hidden branch.

\subsection{Finite-sample scaling and the learned-score bottleneck}
\label{app:toy-finite-sample}

We now translate the population gaps of the previous subsection into
finite-sample scales. The purpose is not to prove a sharp concentration
theorem for the complete learned-score pipeline, but to identify the
effective dimension and spectral gap controlling each stage. Throughout,
\[
    D_0=D_2(d)\asymp d^2,
    \qquad
    p_U,p_V,p_W\asymp d^\epsilon,
    \qquad
    0<\epsilon<1.
\]
We suppress logarithmic factors from truncating the Hermite features and
constants independent of $d$ and $n$.

\paragraph{Matrix concentration and subspace recovery.}
We consider the weighted covariances
defined in~\eqref{eq:toy-empirical-moments}:
\[
    M=C_{\mathbf z}[\xi]
    =\mathbb E\!\left[(\xi-\mathbb E\xi)\mathbf z\mathbf z^\top\right],
    \qquad D=\dim(\mathbf z).
\]
The lower Forward pass uses $C_{\mathbf z_0}[y]$, and the backward
correction uses $C_{\mathbf z_0}[ys_i]$. Both act on
$\mathbf z_0=\Phi_2(\mathbf x)$, so $D=D_2(d)=d(d+1)/2$.
The upper Forward covariance is $C_{\mathbf z_1}[y]$, with
$\mathbf z_1=\Phi_2(\mathbf h_U)$ and $D=D_2(p_U)=p_U(p_U+1)/2$.
In each case, let $\widehat M=\widehat C_{\mathbf z}[\xi]$ be the
empirical matrix from~\eqref{eq:toy-empirical-moments}, and let
$P$ and $\widehat P$ project onto a target eigenspace of $M$ and the
corresponding eigenspace of $\widehat M$. For independent samples, when the
truncated centered summands have operator norm and per-sample matrix variance
$\widetilde O(D)$ and the truncation bias is negligible, the matrix
Bernstein inequality \citep[Theorem~5.4.1]{vershynin2018highdimensional} gives
\begin{equation}
    \|\widehat M-M\|_{\mathrm{op}}
    =
    \widetilde O_{\mathbb P}
    \left(
        \sqrt{\frac{D}{n}}+\frac{D}{n}
    \right).
    \label{eq:toy-generic-concentration}
\end{equation}
For the truncation and concentration argument for Hermite features, see
\citet[Appendix~F.2.1, Lemmas~F.4--F.5]{wen2025gaussian_equivalence},
in particular the proof of Lemma~F.5. In the regime $n\gg D$ considered
below, the square-root term dominates.
If the target eigenspace is separated by a population gap $\Delta$,
the Davis--Kahan theorem \citep{yu2015davis_kahan} then gives the
sufficient condition, up to logarithmic factors,
\begin{equation}
    n\gg \frac{D}{\Delta^2}
    \label{eq:toy-generic-eigen-scale}
\end{equation}
for consistent recovery of its projector.

We use the analogous statement for the cross-layer matrix. If
$\mathcal A$ is a rectangular population operator, $\widehat{\mathcal A}$
its empirical estimate, and $u$ and $\widehat u$ corresponding left
singular vectors separated by a singular-value gap $\Delta_A$,
\citet{wedin1972perturbation} gives the same sufficient scaling,
\begin{equation}
    n\gg \frac{D}{\Delta_A^2},
    \label{eq:toy-generic-singular-scale}
\end{equation}
where $D$ is the effective dimension controlling
$\|\widehat{\mathcal A}-\mathcal A\|_{\mathrm{op}}$.
These estimates apply directly to fixed representations and scores.
Reusing the samples that fitted them requires a separate dependence
argument.

\paragraph{Visible Forward recovery.}
The lower Forward covariance acts on
$\Phi_2(\mathbf x)\in\mathbb R^{D_0}$.
From the population calculation above, the smallest visible spike satisfies
$\Delta_U\asymp p_U^{-1/2}$. Applying
\eqref{eq:toy-generic-eigen-scale} therefore gives
\begin{equation}
    n_U
    \gg
    D_0p_U
    \asymp
    d^{2+\epsilon}.
    \label{eq:toy-visible-sample-scale}
\end{equation}
This is the scale at which the Forward estimator recovers the visible
$U$-subspace. Its hidden population block remains zero, so this argument
does not reveal the $V,W$ directions.

\paragraph{Upper Forward recovery.}
Once $U$ has been recovered, the upper problem lives in
\[
    D_U=D_2(p_U)\asymp p_U^2
\]
dimensions. Assuming that the first-moment coefficient selecting $Q$ and
the $Q$-deflated Wick gap selecting $Z$ remain bounded away from zero, the
same effective-dimension argument gives
\begin{equation}
    n_{Q,Z}
    \gg
    p_U^2
    \asymp
    d^{2\epsilon}.
    \label{eq:toy-upper-sample-scale}
\end{equation}
Since $2\epsilon<2+\epsilon$, this stage is not a bottleneck.

\paragraph{Hidden recovery given an informative score.}
Let $p_h=p_V=p_W$. From the population $V,W$ block derived above,
the nonzero hidden eigenvalues have the same order as the singular values
of $C$. Under the balanced-singular-value condition
\eqref{eq:toy-C-conditioning}, the smallest hidden population gap scales as
$\Delta_H\asymp p_h^{-1/2}$.
The operator again acts in the $D_0$-dimensional quadratic input space,
so \eqref{eq:toy-generic-eigen-scale} yields
\begin{equation}
    n_{V,W}
    \gg
    D_0p_h
    \asymp
    d^{2+\epsilon}.
    \label{eq:toy-hidden-sample-scale}
\end{equation}

\paragraph{Learned-score bottleneck.}
The practical method must learn the informative score from the cross-layer
matrix
\[
    \widehat{\mathcal A}
    =
    \widehat{\mathbb E}
    \left[
        y\,\mathbf g(\mathbf x)\Phi_2(\mathbf x)^\top
    \right].
\]
We now compute the size of the $V,W$ contraction in
\eqref{eq:toy-vw-hermite-product}. For this calculation, we replace the
actual latent blocks $(\mathbf h_U,\mathbf h_V,\mathbf h_W)$ by
independent standard Gaussian vectors $(\mathbf u,\mathbf v,\mathbf w)$
of the same dimensions. We first explain why this is a natural
approximation and what it establishes.

Fix a teacher row and write $A$ for its symmetric-matrix representation.
Diagonalizing $A$ gives
\[
    h_A(\mathbf x)
    =\frac{\mathbf x^\top A\mathbf x-\operatorname{tr}A}{\sqrt2}
    =\frac1{\sqrt2}\sum_{r=1}^d\lambda_r(A)(\xi_r^2-1),
    \qquad \xi_r\stackrel{\mathrm{iid}}{\sim}\mathcal N(0,1).
\]
For the dense Gaussian teacher draws,
$\|A\|_F^2\to1$ and $\|A\|_{\mathrm{op}}=O_{\mathbb P}(d^{-1/2})$.
Thus the total variance tends to one, while the largest summand has
variance $O_{\mathbb P}(d^{-1})$. The Lindeberg central limit theorem
then gives a standard Gaussian limit. The same argument applies to
every fixed linear combination of teacher rows. Moreover,
$\mathbb E_{\mathbf x}[h_{A_i}h_{A_j}]
=\operatorname{tr}(A_iA_j)\to\delta_{ij}$.
Consequently, conditional on typical teacher draws, any fixed collection
of latent coordinates converges jointly, over the randomness of
$\mathbf x$, to independent standard Gaussians. This is also a
degree-two instance of the Gaussian approximation criterion in
\citet[Theorem~4.1]{wen2025gaussian_equivalence}.

This argument explains the independent Gaussian variables used in the
coefficient calculation. It does not yet bound the replacement error
when all $p_U+p_V+p_W\asymp d^\epsilon$ coordinates enter $Q,Z,S$.
In particular, convergence in distribution alone does not show that
the resulting error in the cross-layer row is smaller than
$d^{-1/2}$. That error remains part of $\mathcal R_Z$ in
\eqref{eq:toy-cross-layer-contraction}.

We therefore compute the following coefficients in the Gaussian
approximation: evaluate $Q,Z$ at $\mathbf u$, set
$S=\operatorname{sign}(\mathbf v^\top C\mathbf w)$, and keep $C$ fixed.
Then
\[
    H_{ab}=\mathbb E_{\mathbf v,\mathbf w}[S\,v_aw_b],
    \qquad
    \beta_Z=\mathbb E_{\mathbf u}[b(Q,Z)g_Z(Q,Z)],
    \qquad
    K_d=\sum_{a,b}H_{ab}
    (A_{V,a}A_{W,b}+A_{W,b}A_{V,a}).
\]
Under the norm-preserving identification of quadratic Hermite coefficient
vectors with symmetric matrices, the displayed contribution to the
$Z$ row in~\eqref{eq:toy-cross-layer-contraction} is
$\mathcal A_Z^{VW}=\sqrt2\,\beta_ZK_d$.
Conditional on $C$, $H$ is fixed when averaging over the independent
$V,W$ teacher rows.

Conditioning on $\mathbf v$ gives
\[
    \mathbb E_{\mathbf w}[S\mathbf w\mid\mathbf v]
    =\sqrt{\frac2\pi}\frac{C^\top\mathbf v}{\|C^\top\mathbf v\|},
    \qquad
    \langle H,C\rangle_F
    =\sqrt{\frac2\pi}\mathbb E_{\mathbf v}\|C^\top\mathbf v\|.
\]
The singular-value bounds in~\eqref{eq:toy-C-conditioning} and
Cauchy--Schwarz therefore imply
\[
    \sqrt{\frac2\pi}\frac{c_-}{c_+}
    \frac{\mathbb E_{\mathbf v}\|\mathbf v\|}{\sqrt{p_h}}
    \leq\|H\|_F\leq1.
\]
For the upper bound, the products $v_aw_b$ are orthonormal in Gaussian
$L^2$, so Bessel's inequality gives
$\sum_{a,b}(\mathbb E[Sv_aw_b])^2\leq\mathbb E[S^2]=1$.
Since $\mathbb E\|\mathbf v\|/\sqrt{p_h}\to1$, these bounds show
$\|H\|_F=\Theta(1)$ directly for the prescribed matrix $C$.

Each teacher row, viewed as a symmetric matrix $A$, has Gaussian
coefficients of variance $D_0^{-1}$ in an orthonormal Frobenius basis.
Writing $\mathbb E_A$ for expectation over these draws,
\[
    \mathbb E_A[A_{ij}A_{kl}]
    =\frac{\delta_{ik}\delta_{jl}+\delta_{il}\delta_{jk}}{2D_0},
    \qquad
    \mathbb E_A[A^2]=\frac{I_d}{d}.
\]
For two independent rows $A,B$, Gaussian fourth moments give
\begin{equation}
\begin{aligned}
    \mathbb E_A\|AB+BA\|_F^2
    &=2\mathbb E_A\operatorname{tr}(A^2B^2)
      +2\mathbb E_A\operatorname{tr}(ABAB)\\
    &=\frac2d+\frac1{D_0}+\frac d{D_0^2}
      =:\nu_d=\frac2d+O(d^{-2}).
\end{aligned}
\label{eq:toy-pair-contraction-moment}
\end{equation}
In $\mathbb E_A\|K_d\|_F^2$, all cross terms between distinct row
pairs vanish by independence and centering. Hence
\begin{equation}
    \mathbb E_A\|K_d\|_F^2=\nu_d\|H\|_F^2,
    \qquad
    \operatorname{Var}_A(\|K_d\|_F^2)
    \leq\frac{96\nu_d}{D_0}\|H\|_F^4.
    \label{eq:toy-contraction-concentration}
\end{equation}
For the variance bound, rotate the independent Gaussian row families
using an SVD of $H$, so that
$K_d=\sum_j h_j(A_jB_j+B_jA_j)$, where $h_j$ are the singular values
of $H$. With $F=\|K_d\|_F^2$, its gradient is
$\nabla_{A_j}F=2h_j(K_dB_j+B_jK_d)$, and similarly for $B_j$.
The Gaussian Poincar\'e inequality, with coefficient variance $D_0^{-1}$,
gives
\[
    \operatorname{Var}_A F
    \leq\frac{16}{D_0}\mathbb E_A\!\left[
        F\sum_jh_j^2(\|A_j\|_F^2+\|B_j\|_F^2)
    \right]
    \leq\frac{96}{D_0}\mathbb E_A[F]\|H\|_F^2.
\]
The last inequality follows by conditioning on either row family and
using the Gaussian fourth-moment bound
$\mathbb E[XY]\leq3\mathbb E[X]\mathbb E[Y]$ for nonnegative
quadratic forms $X,Y$ of the same centered Gaussian vector.
Since the variance in~\eqref{eq:toy-contraction-concentration} is
$O(d^{-1})$ times the squared mean, Chebyshev's inequality yields
\begin{equation}
    \|K_d\|_F=(\sqrt2+o_{\mathbb P}(1))\frac{\|H\|_F}{\sqrt d},
    \qquad
    \|\mathcal A_Z^{VW}\|_2
    =(2+o_{\mathbb P}(1))\frac{|\beta_Z|\|H\|_F}{\sqrt d}.
    \label{eq:toy-cross-layer-contraction-scale}
\end{equation}
Thus the displayed contraction has size $\Theta_{\mathbb P}(d^{-1/2})$
when $|\beta_Z|=\Theta(1)$. For the retained linear response
$g_Z=\sqrt2\gamma_ZZ$, this coefficient is
$\beta_Z=\sqrt2\gamma_Z\mathbb E_{\mathbf u}[bZ]$.
In the numerical specialization, $(Q,Z)$ converges to two independent
standard Gaussians and
$\mathbb E_{\mathbf u}[bZ]\to\mathbb E[\tanh'(G+b_0)]/(\sqrt2\sigma_y)>0$,
where $G\sim\mathcal N(0,1)$ and $\sigma_y$ denotes its limiting value.
If $|\gamma_Z|=\Theta(1)$, this gives the stated scale.

The rank-one matrix containing only this $Z$ row has a singular value
and a gap to zero of order $d^{-1/2}$. For the complete cross-layer
matrix, using this gap also requires control of $\mathcal R_Z$ and
separation from the other singular values. The calculation above does
not supply those bounds. If the displayed contribution remains
spectrally isolated at this scale and the fitted-source concentration
bound holds, then the right feature dimension $D_0\asymp d^2$,
\eqref{eq:toy-generic-singular-scale} and
$\Delta_A\asymp d^{-1/2}$ give
\begin{equation}
    n_{\mathrm{cross}}
    \gg
    \frac{D_0}{\Delta_A^2}
    \asymp
    d^3.
    \label{eq:toy-cross-layer-sample-scale}
\end{equation}
Hence the transition near $\alpha=3$ is naturally associated with learning
the backward score rather than with the subsequent effective-label
covariance.

\paragraph{Corrected upper Forward pass.}
After the lower correction, a new upper Forward estimator acts on a
representation of effective quadratic dimension
\[
    D_2(d_1)\asymp d_1^2\asymp d^{2\epsilon}.
\]
For a fixed number of retained upper directions and nonvanishing population
gaps, its natural sample scale is therefore again $d^{2\epsilon}$.
We leave the exact corrected upper rank unspecified until the corresponding
numerical architecture is fixed.

\begin{table}[t]
    \centering
    \begin{tabular}{lc}
        \hline
        Stage & Predicted sample scale \\
        \hline
        Visible $U$ recovery
        & $d^{2+\epsilon}$ \\
        Upper $Q,Z$ recovery
        & $d^{2\epsilon}$ \\
        Learned backward score
        & $d^3$ \\
        Hidden $V,W$ recovery given the score
        & $d^{2+\epsilon}$ \\
        Corrected upper Forward pass
        & $d^{2\epsilon}$ \\
        \hline
    \end{tabular}
    \caption{Stagewise effective-dimension scales for the quadratic
    staircase. The $d^3$ score-selection scale uses the contraction
    size derived in~\eqref{eq:toy-cross-layer-contraction-scale} and
    requires this contribution to remain spectrally isolated in the
    complete cross-layer matrix.}
    \label{tab:toy-sample-scales}
\end{table}

\paragraph{Scaling summary.}
Table~\ref{tab:toy-sample-scales} collects the resulting scales. For $0<\epsilon<1$, the learned-score route is therefore predicted to be
controlled by
\begin{equation}
    n_{\mathrm{BackLoFi}}
    \asymp
    \max\left\{
        d^{2+\epsilon},d^3
    \right\}
    =
    d^3.
    \label{eq:toy-backlofi-sample-scale}
\end{equation}
For the numerical choice $\epsilon=1/2$, the corresponding transitions are
\[
    \alpha_U=\frac52,
    \qquad
    \alpha_{\mathrm{Back}}=3.
\]

The $d^{2+\epsilon}$ scales follow from the population gaps established
above combined with the effective-dimension concentration rule
\eqref{eq:toy-generic-concentration}. By contrast, the $d^3$ prediction is
based on the contraction size derived
in~\eqref{eq:toy-cross-layer-contraction-scale}. Its application to
the complete fitted operator still requires control of the remaining
contractions and the resulting singular-value gap.

\section{The implemented algorithm}
\label{app:algorithm}
This appendix gives Algorithm~\ref{alg:combined} exactly as it is run for the
experiments, with the choices that the main text leaves open. The values of
every hyperparameter are listed in Appendix~\ref{app:hyperparameters}.

\paragraph{Parameterization.}
A layer is a filter $\mathbf F_\ell\in\mathbb R^{p_{\ell-1}\times K_\ell}$
followed by a Gaussian lift $\mathbf R_\ell\in\mathbb R^{p_\ell\times K_\ell}$
with i.i.d.\ $\mathcal N(0,1)$ entries,
\begin{equation}
 \mathbf h_\ell=\frac{1}{c_\ell}\mathbf R_\ell\mathbf F_\ell^{\top}\mathbf z_{\ell-1},
 \qquad
 \mathbf z_\ell=\frac{\sigma(\mathbf h_\ell)}{\sqrt{p_\ell}},
 \qquad
 \mathbf D_\ell=\frac{\operatorname{diag}(\sigma'(\mathbf h_\ell))}{\sqrt{p_\ell}},
 \qquad
 c_\ell^2=\Etrain\|\mathbf F_\ell^{\top}\mathbf z_{\ell-1}\|_2^2 ,
 \label{eq:spec-block}
\end{equation}
so that $\mathbf W_\ell=\mathbf R_\ell\mathbf F_\ell^{\top}/c_\ell$. The scale
$c_\ell$ is fitted with the forward network and never refitted. A cutoff $k_\ell$
stores $K_\ell=k_\ell+1$ columns, the normalized first moment followed by $k_\ell$
eigenvectors of $\widehat{\mathbf C}_\ell$ deflated against it. Each column carries
a tag, \textsc{fwd}, \textsc{read} or \textsc{pair}, according to whether it was
installed by forward Neural LoFi, by a row step or by a column step. The rank of
every boundary is $r_\ell=\max\{1,\operatorname{round}(fk)\}$ for a rank fraction
$f$. The widths never change: every append replaces $r_\ell$ existing columns.

\begin{algorithm}[t]
\caption{One pass of Algorithm~\ref{alg:combined}, as implemented.}
\label{alg:ffn-spec}
\begin{algorithmic}[1]
\Require Network $(\mathbf F_\ell,\mathbf R_\ell,c_\ell)_{\ell\le L}$ and readout
 $(\mathbf a_L,b)$; rank fraction $f$, scale $\beta$; signal $\in\{$first-order,
 exact$\}$.
\State Compute $\mathbf z_\ell,\mathbf D_\ell$ by \eqref{eq:spec-block};\;
 $\mathbf g^L\gets\mathbf D_L\mathbf a_L$;\; keep a copy $(\mathbf F^0_\ell,\mathbf R^0_\ell)$ of the network
\For{$\ell=L-1,\dots,1$}
 \State $(\mathbf U,\mathbf V)\gets$ top $r_\ell$ singular pairs of
  $\mathbf A_\ell=\Etrain[y\,\mathbf g^{\ell+1}\mathbf z_{\ell-1}^{\top}]$;\;
  $\mathbf s\gets\mathbf U^{\top}\mathbf g^{\ell+1}$ \Comment{frozen scores}
 \State $I\gets$ the $r_\ell$ columns of $\mathbf F_\ell$ of lowest importance
  \eqref{eq:spec-importance};\; $\tau_\ell\gets$ mean coordinate RMS of the kept
  \textsc{fwd} and \textsc{read} columns
 \State $\mathbf F_\ell[:,I]\gets\mathbf v_i\,\beta\tau_\ell/\RMS\langle\mathbf v_i,\mathbf z_{\ell-1}\rangle$;\;
  $\mathbf R_\ell[:,I]\gets$ fresh $\mathcal N(0,1)$ \Comment{row step}
 \State Recompute $\widetilde{\mathbf z}_\ell,\mathbf D_\ell$;\;
  $\mathbf Q\gets\Etrain[\widetilde{\mathbf z}_\ell(y\mathbf s)^{\top}]$
 \State $J\gets$ the $r_\ell$ columns of $\mathbf F_{\ell+1}$ of lowest importance;\;
  $\tau_{\ell+1}$ as above
 \State $\mathbf F_{\ell+1}[:,J]\gets\mathbf q_i\,\beta\sqrt{p_{\ell+1}}\,\tau_{\ell+1}/\RMS\langle\mathbf q_i,\widetilde{\mathbf z}_\ell\rangle$;\;
  $\mathbf R_{\ell+1}[:,J]\gets\mathbf U$ \Comment{column step}
 \If{$\ell>1$}
  \State first-order: $\mathbf g^\ell\gets\mathbf D_\ell\bigl(\Etrain\bigl[\tfrac{1}{c_{\ell+1}}\mathbf F^0_{\ell+1}(\mathbf R^0_{\ell+1})^{\top}\mathbf g^{\ell+1}\bigr]+\tfrac{1}{c_{\ell+1}}\mathbf Q^0\mathbf s\bigr)$,\;
   $\mathbf q^0_i\propto\tfrac{1}{c_\ell}\mathbf R^0_\ell(\mathbf F^0_\ell)^{\top}\mathbf v_i$
  \State exact: $\mathbf g^\ell\gets\mathbf D_\ell\mathbf W_{\ell+1}^{\top}\mathbf g^{\ell+1}$ through the corrected $\mathbf W_{\ell+1}$, per sample
 \EndIf
\EndFor
\State Recompute $\mathbf z_L$;\; refit $(\mathbf a_L,b)$ by ridge regression
\end{algorithmic}
\end{algorithm}

\paragraph{Scales.}
The only free number is $\beta$, the size of an appended
coordinate in units of a forward one. It sets the amplitude of every column
appended by the row step and by the column step. Write
$\RMS\,\xi=(\Etrain\xi^2)^{1/2}$ for the root mean square over the training
set, and let $\tau_\ell$ be the mean of
$\RMS\langle\mathbf f_a,\mathbf z_{\ell-1}\rangle$ over the surviving columns
$\mathbf f_a$ of $\mathbf F_\ell$ that carry a random lift (\textsc{fwd} and
\textsc{read}). The row step appends $\gamma^{\rm row}_i\mathbf v_i$ to
$\mathbf F_\ell$ and the column step appends $\gamma^{\rm col}_i\mathbf q_i$ to
$\mathbf F_{\ell+1}$, with
\begin{equation}
 \gamma^{\rm row}_i=\frac{\beta\,\tau_\ell}{\RMS\langle\mathbf v_i,\mathbf z_{\ell-1}\rangle},
 \qquad
 \gamma^{\rm col}_i=\frac{\beta\sqrt{p_{\ell+1}}\,\tau_{\ell+1}}{\RMS\langle\mathbf q_i,\widetilde{\mathbf z}_\ell\rangle}.
 \label{eq:spec-gammas}
\end{equation}
In the notation of \eqref{eq:main-coupled-steps}, where the factor $1/c$ of
\eqref{eq:spec-block} is absorbed, this reads
$\boldsymbol\Gamma^{\rm row}_\ell=\operatorname{diag}(\gamma^{\rm row}_i)/c_\ell$
and
$\boldsymbol\Gamma^{\rm col}_\ell=\operatorname{diag}(\gamma^{\rm col}_i)/c_{\ell+1}$.
An appended read coordinate thus has root mean square $\beta\tau_\ell$. A
paired column carries the extra factor $\sqrt{p_{\ell+1}}$ because its lift
column $\mathbf u_i$ has unit norm, while a forward column is lifted by
$\mathcal N(0,1)$ entries; with it, the paired column contributes to a neuron
of layer $\ell+1$ as much as $\beta$ forward coordinates. Each $\gamma_i$ is
set per column, which equalizes the amplitudes of the retained modes. The
columns $\mathbf q^0_i$ of the first-order signal are rescaled to the same
target as $\gamma^{\rm col}_i\mathbf q_i$. We use $\beta=0.25$ for the fully
connected network and $\beta=4$ for the convolutional one, each chosen by a
local scan at $n=10^4$.

\paragraph{Which columns are replaced.}
The widths are fixed, so each append replaces $r_\ell$
existing columns: the row step replaces the set $I$ of columns of
$\mathbf F_\ell$, and the column step the set $J$ of columns of
$\mathbf F_{\ell+1}$. In both cases we remove the $r_\ell$ columns of lowest
importance $\Omega$, and the rule treats forward, read and paired columns
alike at every pass. Importance is defined as in forward Neural LoFi, where a
direction matters when gradient flow amplifies its coefficients in the rows,
but it is computed on the current network, whose layers above are no longer
random. For a column $a$ of layer $m$ it is built in three steps.
\begin{enumerate}
 \item \emph{Coefficients.} Let $\widehat{\mathbf v}_a$ be the
  unit direction of the column and
  $h_a=\langle\widehat{\mathbf v}_a,\mathbf z_{m-1}\rangle$ its coordinate.
  Neuron $j$ reads it with the coefficient
  $c_{j,a}=(\mathbf R_m)_{ja}\|\mathbf v_a\|/c_m$, and sees the effective
  label $y\,\langle\boldsymbol\omega_j,\mathbf g^{m+1}\rangle$, where
  $\boldsymbol\omega_j$ is the column of $\mathbf W_{m+1}$ that reads it.
 \item \emph{Rate.} Project the row update
  \eqref{eq:main-pair-gradient} on $\widehat{\mathbf v}_a$, with the label in
  place of the residual and the gate expanded to first degree,
  $\sigma'(h)\simeq d_m+e_m h$ with
  $(d_m,e_m)=(\Etrain\sigma'(h),\Etrain h\sigma'(h))$. This gives
  $\dot c_{j,a}\propto e_m\,\lambda_j(a)\,c_{j,a}+d_m\,\ell_j(a)$, with the
  drivers $\lambda_j,\ell_j$ given in \eqref{eq:spec-importance}.
 \item \emph{Importance.} Take the root mean square of this
  rate over the neurons of the layer:
\end{enumerate}
\begin{equation}
\begin{gathered}
 \lambda_j(a)=\bigl\langle\boldsymbol\omega_j,\Etrain[y\,\widehat{\mathbf g}\,h_a^2]\bigr\rangle,
 \qquad
 \ell_j(a)=\bigl\langle\boldsymbol\omega_j,\Etrain[y\,\widehat{\mathbf g}\,h_a]\bigr\rangle,\\
 \Omega(a)=\RMS_j\bigl(e_m\,\lambda_j(a)\,c_{j,a}+d_m\,\ell_j(a)\bigr).
\end{gathered}
 \label{eq:spec-importance}
\end{equation}
A column whose coefficients no neuron would move can thus be
removed at the least cost. In the rate, the term in $\lambda_j$ is the covariance channel,
multiplicative in the coefficient, and the term in $\ell_j$ the spike
channel, additive; only the diagonal of the driver is
kept, since the drive of $\widehat{\mathbf v}_a$ by the other directions a
neuron reads averages out over neurons. The backward signal enters through
$\widehat{\mathbf g}=\bar{\mathbf g}+\mathbf U\mathbf U^{\top}(\mathbf g^{m+1}-\bar{\mathbf g})$,
its mean plus its sample dependence along the retained modes of the boundary
above, the only part that reaches the neurons with a coherent sign. When the backward signal carries no sample dependence,
$\widehat{\mathbf g}=\bar{\mathbf g}$, $\lambda_j(a)$ is
$\langle\boldsymbol\omega_j,\bar{\mathbf g}\rangle\,\Etrain[yh_a^2]$, and
$\Omega$ orders the covariance directions by the signed Rayleigh quotient that
forward Neural LoFi uses, weighted by how strongly the neurons read them. For
the top layer $\boldsymbol\omega_j=(\mathbf a_L)_j$ and $\widehat{\mathbf g}\equiv1$.

\paragraph{Convolutional network.}
The network is conv1 ($3\times3$, no filter), conv2 ($3\times3$, filter
$k_c$ over channels) with pooling, and a fully connected layer (filter $k_3$).
At the boundary conv2\,|\,FC, $\mathbf A$ pairs the FC sensitivity with the
spatial mean of conv1's output; the row step replaces channel directions of
conv2, each with a fresh random $3\times3$ kernel, and the column step
replaces FC filter columns by $\mathbf Q$ computed on the pooled map, with all
$H'W'$ lift rows of a paired channel equal to $\mathbf u_i$ and the target
divided by $\sqrt{H'W'}$ to match a forward channel. Conv1 has no filter, so the
boundary conv1\,|\,conv2 performs the column step only, with locations treated
as samples and $\mathbf u_i$ placed on the centre tap. The backward signal is
carried through the pooling by its adjoint, and the importance is computed per
filter and tap. Every statistic is streamed over the training set.

\paragraph{Numerics.}
The top $r_\ell$ singular pairs are computed exactly when $r_\ell+32$ exceeds
half the smaller dimension of $\mathbf A_\ell$ and by a randomized SVD otherwise
(oversampling $32$, four power iterations, fixed seed). Statistics are chunked
averages over the training set, accumulated in double precision.

\subsection{Variants tested and not retained}
\label{app:algorithm-variants}
We think it is useful to report approaches that did not
work, so that others need not try them again. The following variants were
implemented and run on the fully connected network, and none was retained.
\begin{description}
 \setlength{\itemsep}{1pt}
 \item[Gate-only signal] ($\mathbf g^\ell=\mathbf D_\ell\mathbf 1$): no gain over
  forward Neural LoFi, as expected, since its read-directions stay in the span
  of the forward filter (Section~\ref{sec:combined}).
 \item[Mean-field and exact signals:] the same test error as the first-order
  signal within seed noise; the first-order signal is kept because it needs
  no propagation through the installed blocks.
 \item[Common column scale] (the mean or the largest paired column set to
  $\beta\sqrt{p_{\ell+1}}\tau_{\ell+1}$, relative amplitudes kept): no
  improvement over forward Neural LoFi; only the per-column scale did improve.
 \item[Scale $\tau$ over all columns:] paired columns carry
  $\sqrt{p}$-larger coordinates, so including them inflates $\tau$ by a factor
  of about $6$ per pass, and the later appends with it.
 \item[Weakest swap] (the lowest Rayleigh quotients $|\Etrain[yh_a^2]|$):
  backward columns score high, so passes accumulate them and consume the
  forward columns.
 \item[Replace-own swap] (each pass rewrites the slots of the previous one):
  no gain beyond the first pass.
 \item[Update of the whole block $\ell+1$] (forward columns refitted on
  $\widetilde{\mathbf z}_\ell$ before the swap): the same test error, within
  $0.1$ points at every $n$.
 \item[Dense interpolation]
  ($\mathbf W\gets(1-\alpha)\mathbf W+\alpha\,\text{estimate}$, $\alpha=0.1$, no
  swap): the same test error, within $0.1$ points at every $n$.
\end{description}

\section{Additional numerical simulations}
\label{app:additional-simulations}

\subsection{Number of passes at fixed width}
\label{app:passes}
Figure~\ref{fig:passes-ffn} shows the test error of the fully connected
network against the number of passes of Algorithm~\ref{alg:combined} at fixed
width. With $r_\ell=5\%$, $7\%$, $10\%$ and $12\%$ of $k_\ell$ the error
decreases with the passes to the same minimum, reached when about half of the
filter columns have been replaced (pass $4$ at $12\%$, pass $5$ at $10\%$, pass
$7$--$8$ at $7\%$, pass $9$--$10$ at $5\%$); further passes exhaust the forward
columns and the error rises quickly. With $r_\ell=3\%$ the error is still decreasing slowly after ten
passes, and with $r_\ell=2\%$ it stays at the forward level.

\begin{figure}[H]
 \centering
 \includegraphics[width=\textwidth]{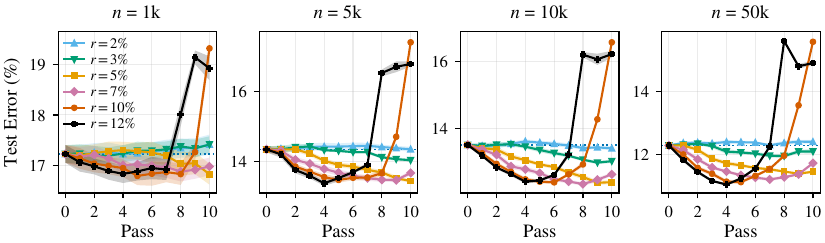}
 \caption{Test error of the fully connected network against the number of
 passes of Algorithm~\ref{alg:combined} at fixed width, for six rank
 fractions $r_\ell/k_\ell$ and four training-set sizes $n$. Pass $0$ is
 forward Neural LoFi (dotted line). Each point is the best filter width
 $k_\ell$ for its $n$, rank fraction and pass, selected by the mean over $10$
 seeds; bands are one standard error. Settings in
 Appendix~\ref{app:hyperparameters}.}
 \label{fig:passes-ffn}
\end{figure}

\subsection{What successive passes select}
\label{app:pass-overlap}
Figure~\ref{fig:pass-overlap} shows, for each pass, the fraction of its new
read and write directions that lies in the span of those installed by the
earlier passes. For the read directions it stays above $0.97$, against at most
$0.02$ for a random subspace: each pass selects again the same read subspace
and installs it with fresh lifts. Successive passes therefore resample
random features within the same subspace. The error in
Figure~\ref{fig:passes-ffn} consequently depends on the fraction of columns
replaced and is insensitive to the per-pass step $r_\ell$.

\begin{figure}[H]
 \centering
 \includegraphics[width=\textwidth]{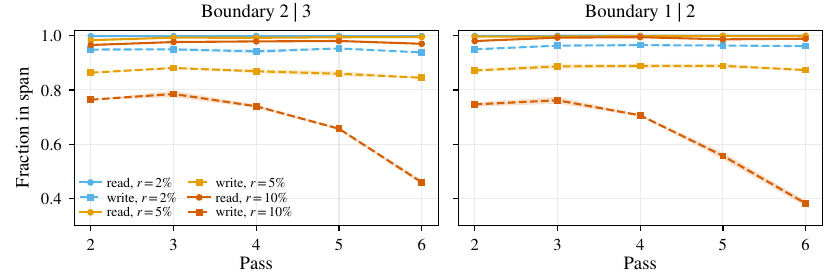}
 \caption{Fraction of the read (full) and write (dashed) directions selected
 at a pass inside the span of those installed by the earlier passes,
 $\|V_t^{\top}Q_{<t}\|_F^2/r_\ell$, at the two boundaries of the fully connected
 network ($k_\ell=512$, $n=10{,}000$, mean and standard error over $3$ seeds).
 A random subspace gives at most $0.02$.}
 \label{fig:pass-overlap}
\end{figure}

\subsection{Damped updates}
\label{app:damped}
The fixed-width version of Algorithm~\ref{alg:combined} used in
Figure~\ref{fig:test-error-vs-n} replaces at each pass the $r_\ell$ forward
columns of lowest importance by the corrected ones, so that the passes consume
the forward columns (Appendix~\ref{app:passes}). Figure~\ref{fig:damped-ffn}
tests a rule that consumes nothing: the layers are kept dense and the two
updates of~\eqref{eq:main-coupled-steps} are applied as damped averages,
\begin{equation}
 \mathbf W_\ell\leftarrow(1-\alpha)\,\mathbf W_\ell
   +\alpha\,\mathbf R^{\leftarrow}_\ell\boldsymbol\Gamma^{\rm row}_\ell\mathbf V_\ell^{\top},
 \qquad
 \mathbf W_{\ell+1}\leftarrow(1-\alpha)\,\mathbf W_{\ell+1}
   +\alpha\,\mathbf U_\ell\boldsymbol\Gamma^{\rm col}_\ell\mathbf Q_\ell^{\top},
 \label{eq:damped-steps}
\end{equation}
with $\alpha=0.1$, the same estimates, no column replaced and the readout
refitted after each pass. With $r_\ell=25\%$ of $k_\ell$ the error decreases
monotonically over about $25$ passes to $18.7\%$, $16.9\%$, $13.6\%$, $12.6\%$
and $11.0\%$ at $n=500$, $1{,}000$, $5{,}000$, $10{,}000$ and $50{,}000$,
within $0.2$ points of the minimum of the swap ($18.7\%$, $16.8\%$, $13.4\%$,
$12.4\%$ and $11.1\%$) and never significantly below it, then rises slowly, by
$0.15$--$0.25$ points at pass $50$; with $r_\ell=10\%$ the descent is slower
(minimum at passes $30$--$48$ for $n\ge5{,}000$) and stops $0.2$ points higher.
The damped rule reaches essentially the same minimum error as the swap
over more passes, followed by a gradual rise. This is consistent with the
identity $(1-\alpha)\mathbf W+\alpha\mathbf E=(1-\alpha)\bigl(\mathbf
W+\tfrac{\alpha}{1-\alpha}\mathbf E\bigr)$: the network is positively
homogeneous and the ridge readout absorbs the factor $1-\alpha$, so the damped
rule is the update~\eqref{eq:main-coupled-steps} with the scale $\beta$
replaced by $\beta\alpha/(1-\alpha)$, applied at every pass. Both methods
use the same direction-selection rule.

\begin{figure}[H]
 \centering
 \includegraphics[width=\textwidth]{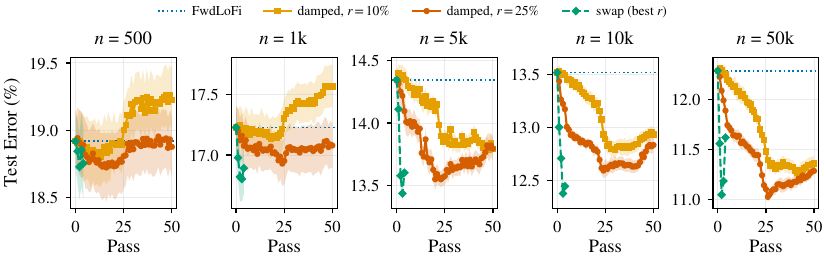}
 \caption{Test error of the fully connected network against the number of
 passes of the damped update~\eqref{eq:damped-steps} with $\alpha=0.1$, at
 two rank fractions $r_\ell/k_\ell$, and of the fixed-width swap of
 Figure~\ref{fig:test-error-vs-n} at its best rank fraction (dashed), for five
 training-set sizes $n$. Pass $0$ is forward Neural LoFi (dotted line). Each
 point is the best filter width $k_\ell$ for its $n$, rank fraction and pass,
 selected by the mean over $10$ seeds; bands are one standard error. Settings
 in Appendix~\ref{app:hyperparameters}.}
 \label{fig:damped-ffn}
\end{figure}

\subsection{Gradient-descent training curves}
\label{app:gd-curves}
Figures~\ref{fig:gd-curves} and~\ref{fig:gd-curves-cnn} show the test error of
full-batch gradient descent along training for the learning rates that do not
diverge, on the fully connected and on the convolutional network. On the fully
connected network the rate $0.1$ is the best at every $n$. On the convolutional
network the rate $0.003$ is the best up to $n=5{,}000$ and the largest stable
rate, $0.005$, from $n=10{,}000$ upwards; the rate $0.001$ is behind at every
$n$. At large $n$ the error is still decreasing at the last step on both
networks, so the gradient-descent curves of Figure~\ref{fig:test-error-vs-n}
are limited by the step budget there ($3{,}000$ steps on the fully connected
network, $5{,}000$ on the convolutional one).

\begin{figure}[H]
 \centering
 \includegraphics[width=\textwidth]{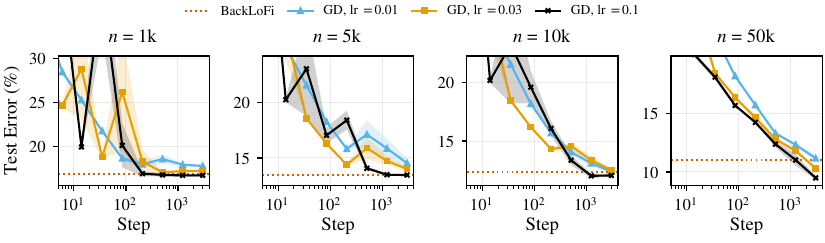}
 \caption{Test error of full-batch gradient descent on the fully connected
 network against the step, for three learning rates; mean and standard error
 over $5$ seeds. Dotted: BackLoFi of Figure~\ref{fig:test-error-vs-n}.}
 \label{fig:gd-curves}
\end{figure}

\begin{figure}[H]
 \centering
 \includegraphics[width=\textwidth]{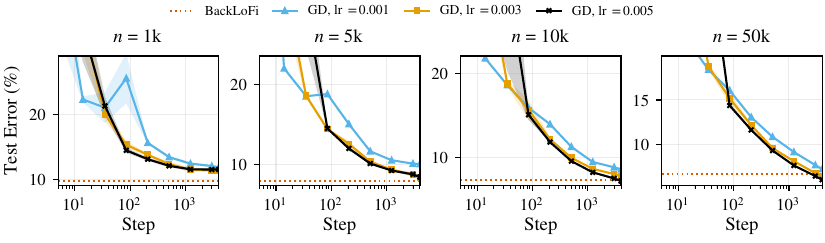}
 \caption{Test error of full-batch gradient descent on the convolutional
 network against the step, for three learning rates at every $n$; mean and
 standard error over $3$ seeds for the rates $0.003$ and $0.005$, one seed
 for $0.001$ at $n\ge5{,}000$. Dotted: BackLoFi of
 Figure~\ref{fig:test-error-vs-n}.}
 \label{fig:gd-curves-cnn}
\end{figure}

\section{Hyperparameters of the figures}
\label{app:hyperparameters}
The image experiments share one protocol: binary CIFAR-10 (vehicles $+1$,
animals $-1$), a training sub-sample of size $n$, the $10^4$ test images, a
ridge readout with leave-one-out penalty among $50$ log-spaced values in
$[10^{-6},10^{6}]$, test error as the fraction of wrong signs, means over seeds
with every selection made on the seed mean. Algorithmic terms are defined in
Appendix~\ref{app:algorithm}.
\begin{description}
 \setlength{\itemsep}{2pt}
\item[Figure~\ref{fig:toy-main}.]
We use $d=100$ and the finite-dimensional realization
$(p_U,p_V,p_W)=(6,5,5)$, corresponding to the $\epsilon=1/2$ scaling. The sixteen first-level quadratic directions are
drawn from a Gaussian ensemble and jointly orthonormalized in
$\mathbb R^{D_2(100)}$. For the second level,
$B_Q=I_6/\sqrt6$, while $B_Z$ is a randomly rotated balanced contrast
with three eigenvalues $+1/\sqrt6$ and three eigenvalues $-1/\sqrt6$.
The hidden coupling $C$ is drawn as a Haar orthogonal matrix scaled by
$1/\sqrt5$. We use $b_0=1/2$ and no additive label noise. Labels are
centered and standardized using the training statistics, which are then
kept fixed for validation and test.

The spectral estimators follow Algorithm~4. The frozen lower random-feature
map has width $65{,}536$, unit-norm Gaussian weights, and activation
$\sigma_{\geq2}(t) =\operatorname{ReLU}(t)-\frac{t}{2}-\frac{1}{\sqrt{2\pi}}$
The lower filter retains $16$ directions, of which the leading $6$ define
the protected visible anchor. The upper Forward step uses the complete
quadratic Hermite lift of these $16$ coordinates, of dimension $136$, and
retains $16$ directions. Within BackLoFi algorithm, after correction, the upper
Forward step is rerun with the original labels. The final affine ridge
readout is fitted on the linear and quadratic upper features, with relative
regularization selected from
$\{10^{-8},10^{-6},10^{-4},10^{-2},1\}$ by validation MSE.

For the gradient-descent baselines, we use the same frozen lower
random-feature map followed by two trainable hidden layers of widths
$(512,512)$ with ReLU activations and an affine scalar readout. We use AdamW with
learning rate in $\{3\times10^{-4},10^{-3}\}$, batch size
$\min(n,1024)$, weight decay $10^{-4}$, gradient-norm clipping at $5$,
and $4000$ updates. End-to-end GD trains both hidden layers jointly.
For layer-wise GD, the budget is split into $2000$ updates for the first
layer and a temporary readout, followed by $2000$ updates for the second
layer and final readout with the first layer frozen. The best validation
checkpoint and learning rate are selected separately for each sample size
and seed. We use
$\alpha\in\{1.5,1.6,\ldots,3.5\}$ with
$n=\operatorname{round}(100^\alpha)$, two seeds $\{0,1\}$,
$16{,}384$ validation samples and $32{,}768$ test samples. Curves report
the mean and one standard error over the two seeds.

 \item[Figure~\ref{fig:test-error-vs-n}.]
  $n\in\{500,1\mathrm{k},2\mathrm{k},5\mathrm{k},10\mathrm{k},20\mathrm{k},50\mathrm{k}\}$.
  \emph{Left:} $L=3$ fully connected ReLU layers, $p_\ell=5{,}000$, $10$ seeds.
  FwdLoFi: $k_\ell\in\{32,\dots,2{,}048\}$ (powers of two), best per $n$.
  BackLoFi ($1^{\rm st}$ ord): fixed width, $r_\ell/k_\ell\in\{5,10,15,20,25\}\%$,
  $P\le4$ passes, $\beta=0.25$, best $(k_\ell,r_\ell,P)$ per $n$; exact:
  $r_\ell/k_\ell\in\{10,20,25\}\%$, $k_\ell\ge64$. RF: identity filters. Ridge: on
  the $3{,}072$ pixels. GD: same widths, full batch, MSE on $\pm1$ labels, no
  momentum or weight decay, $3{,}000$ steps, rates $\{0.01,0.03,0.1,0.2\}$; curve
  at rate $0.1$ (best on average over $n$) and its best recorded step, $5$ seeds.
  \emph{Right:} conv $3\times3$ with $4{,}096$ channels (no filter), conv
  $3\times3$ with $4{,}096$ channels (filter $k_c$), $4\times4$ max pooling,
  fully connected $5{,}000$ (filter $k_3$), ReLU, $5$ seeds.
  FwdLoFi: $k_c\in\{64,\dots,1{,}024\}$, $k_3\in\{32,\dots,1{,}024\}$.
  BackLoFi ($1^{\rm st}$ ord): $k_3\le512$, ranks
  $\mathrm{round}(f\min(k_c,k_3))$ and $\mathrm{round}(fk_c)$ at the two
  boundaries, $f\in\{5,10,15,20,25\}\%$, $P\le4$, $\beta=4$; exact:
  $k_c,k_3\in\{128,\dots,1{,}024\}$, $f\in\{10,20,25\}\%$, $P\le3$. RF: identity
  filters.
  GD: same widths, full batch, MSE on $\pm1$ labels, no momentum or weight
  decay, $5{,}000$ steps, bfloat16 autocast (a replica without it at $n=500$
  gives the same curve within the seed noise), rates $\{0.001,0.003,0.005\}$
  ($0.01$ diverges at the third step for every $n$); $3$ seeds for $0.003$ and
  $0.005$ at every $n$ and for $0.001$ up to $n=2\mathrm{k}$, one seed above;
  the curve uses, at each $n$, the best of the rates run with three seeds,
  selected on the seed mean, at its best recorded step ($0.003$ up to
  $n=5\mathrm{k}$, $0.005$ above).
  Shaded bands are one standard error over seeds.
  \emph{Backward signal:} compared on the grid of the exact arm, the
  first-order signal of \eqref{eq:main-first-order} and the exact signal
  backpropagated through the corrected network give the same error on the
  fully connected network within the standard error at every $n$; on the
  convolutional network the first-order signal is lower by about $0.3$ points
  for $2{,}000\le n\le10{,}000$ ($0.4$ at $n=500$) and equal at $n=50{,}000$.
  \emph{Reference values at $n=50{,}000$:} fully connected, RF $13.9\%$,
  FwdLoFi $12.3\%$, BackLoFi $11.1\%$, GD $9.5\%$; convolutional, RF $8.4\%$,
  FwdLoFi $7.8\%$, BackLoFi $6.7\%$, GD $5.7\%$.
 \item[Figure~\ref{fig:passes-ffn}.] Fully connected network and BackLoFi
  ($1^{\rm st}$ ord) of Figure~\ref{fig:test-error-vs-n}, with
  $r_\ell/k_\ell\in\{2,\allowbreak 3,\allowbreak 5,\allowbreak 7,\allowbreak 10,\allowbreak 12\}\%$, $P=1,\dots,10$, $\beta=0.25$,
  $k_\ell\in\{32,\dots,2{,}048\}$ (best per point),
  $n\in\{1\mathrm{k},5\mathrm{k},10\mathrm{k},50\mathrm{k}\}$, $10$ seeds.
 \item[Figure~\ref{fig:pass-overlap}.] Fully connected network, first-order
  signal, $k_\ell=512$, $n=10\mathrm{k}$, $r_\ell/k_\ell\in\{2,5,10\}\%$, $6$ passes,
  $\beta=0.25$; $3$ seeds.
 \item[Figure~\ref{fig:damped-ffn}.] Fully connected network of
  Figure~\ref{fig:test-error-vs-n} with dense layers and the damped
  updates~\eqref{eq:damped-steps}, first-order signal, $\alpha=0.1$,
  $r_\ell/k_\ell\in\{5,10,25\}\%$ ($5\%$ not shown: it stays within $0.1$ points
  of FwdLoFi for $n\le1\mathrm{k}$ and descends more slowly than $10\%$ above),
  $P=1,\dots,50$, $\beta=0.25$, $k_\ell\in\{32,\dots,2{,}048\}$ (best per
  point), $n\in\{500,1\mathrm{k},5\mathrm{k},10\mathrm{k},50\mathrm{k}\}$, $10$
  seeds. The forward scale $\tau$ entering $\boldsymbol\Gamma$ is measured
  with the forward-phase filters on the current features; both updates are
  applied at each boundary in turn, so a hidden layer between two boundaries
  is damped twice per pass; a layer whose pre-activation scale drifts by more
  than a factor $8$ would be rescaled by a power of two (never triggered).
  Swap reference: the $4$-pass arm of Figure~\ref{fig:test-error-vs-n} at its
  best $r_\ell$ for each $n$.
 \item[Figure~\ref{fig:gd-curves}.] GD runs of
  Figure~\ref{fig:test-error-vs-n}, rates $\{0.01,0.03,0.1\}$,
  $n\in\{1\mathrm{k},5\mathrm{k},10\mathrm{k},50\mathrm{k}\}$, $5$ seeds.
 \item[Figure~\ref{fig:gd-curves-cnn}.] Convolutional network of
  Figure~\ref{fig:test-error-vs-n} with trainable layers of the same widths,
  random initialization, full batch (gradient accumulated over the whole
  training set), MSE on $\pm1$ labels, no momentum or weight decay, $5{,}000$
  steps, bfloat16 autocast; rates $\{0.001,0.003,0.005\}$ ($0.01$ diverges) at
  every $n$; seeds as in the caption.
\end{description}

\end{document}